\documentclass{article}

\usepackage[main, final]{neurips_2026}

\usepackage{microtype}
\usepackage{graphicx}
\usepackage{subcaption}
\usepackage{booktabs} % for professional tables
\usepackage{makecell}
\usepackage{hyperref}

\usepackage{amsmath, amsfonts, amssymb, amsthm, mathtools}
\usepackage{times}
\usepackage{latexsym}

\usepackage[table,xcdraw,dvipsnames]{xcolor} 
\usepackage[most]{tcolorbox}
\usepackage{pifont}
\usepackage{booktabs}
\usepackage{multirow}
\usepackage{tabularx}
\usepackage{colortbl} 

\usepackage[normalem]{ulem}
\usepackage{algorithm}
\usepackage{algpseudocode}
\usepackage{enumitem}

\usepackage{wrapfig}
\usepackage{subcaption}
\usepackage{graphicx}

\usepackage{hyperref}

\usepackage[capitalize,noabbrev]{cleveref}

\newsavebox{\fitbox}
\newcommand{\fitwidth}[1]{%
  \sbox{\fitbox}{#1}%
  \ifdim\wd\fitbox>\linewidth
    \resizebox{\linewidth}{!}{\usebox{\fitbox}}%
  \else
    \usebox{\fitbox}%
  \fi}

\theoremstyle{plain}
\newtheorem{theorem}{Theorem}[section]
\newtheorem{proposition}[theorem]{Proposition}
\newtheorem{lemma}[theorem]{Lemma}

\theoremstyle{definition}
\newtheorem{definition}[theorem]{Definition}
\newtheorem{assumption}[theorem]{Assumption}
\theoremstyle{remark}

\usepackage[textsize=tiny]{todonotes}

\title{LoopRPT: Reinforcement Pre-Training for Looped Language Models}

\let\oldthefootnote\thefootnote
\renewcommand{\thefootnote}{\fnsymbol{footnote}}

\author{
  Guo Tang$^{1}$\footnotemark[1]~~
  Shixin Jiang$^{1}$\footnotemark[1]~~
  Heng Chang$^{3}$\footnotemark[1]~~
  Zihan Zhang$^{4}$~~
  Nuo Chen$^{5}$ \\
  \textbf{Yuhan Li}$^{5}$~~
  \textbf{Huiming Fan}$^{1}$~~
  \textbf{Jia Li}$^{6}$~~ 
  \textbf{Ming Liu}$^{1,2}$~~
  \textbf{Bing Qin}$^{1,2}$\footnotemark[2] \\[2pt]
  \normalfont
  $^{1}$Harbin Institute of Technology, Harbin, China~~$^{2}$Peng Cheng Laboratory, Shenzhen, China \\
  $^{3}$Tsinghua University~~$^{4}$Harvard University 
  $^{5}$HY Foundation Model Team, Tencent~~ \\
  $^{6}$Hong Kong University of Science and Technology (Guangzhou) \\[2pt]
  \texttt{\{gtang,sxjiang,mliu,qinb\}@ir.hit.edu.cn~~changh17@tsinghua.org.cn}
}

\begin{document}

\maketitle

\footnotetext[1]{Equal contribution.}
\footnotetext[2]{Corresponding author.}
\let\thefootnote\oldthefootnote

\begin{abstract}
    Looped language models (LoopLMs) perform iterative latent computation to refine internal representations, offering an alternative to explicit chain-of-thought (CoT) reasoning. However, existing reinforcement learning (RL) paradigms primarily target output tokens, creating a structural mismatch with looped architectures whose reasoning unfolds implicitly. In this work, we propose \textbf{LoopRPT}, a reinforcement pre-training framework tailored for LoopLMs. By reframing next-token prediction as a next-token reasoning task, LoopRPT assigns reinforcement signals directly to latent steps using an EMA teacher reference and noisy latent rollouts. This formulation enables RL to directly shape intermediate representations, compressing effective reasoning into fewer iterations. We instantiate LoopRPT on Ouro across multiple model scales. Results show that LoopRPT improves per-step representation quality and achieves Pareto dominance in accuracy–computation trade-offs. Large gains on hard tokens indicate that LoopRPT enhances early-stage reasoning rather than merely encouraging premature exits. Our findings highlight reinforcement pre-training as a principled paradigm for learning efficient latent reasoning in LoopLMs.

\end{abstract}

\section{Introduction}
\label{sec:introduction}

Modern Large Language Models (LLMs) are trained to ``think" primarily through explicit text generation, such as chain-of-thought (CoT) prompting~\citep{wei2022chain}. While effective, this paradigm defers reasoning to post-training and often underuses the structural information available in pre-training data~\citep{NIPS2017_d5e2c0ad,wen2025reinforcementlearningverifiablerewards,Guo_2025}. In contrast, Looped Language Models (LoopLMs), such as Ouro, use a parameter-shared looped backbone where internal representations are refined recurrently in latent space~\citep{zhu2025scaling}. This enables iterative computation and deeper reasoning without a proportional increase in parameter count, compressing multi-step reasoning into latent transitions rather than explicit tokens (Fig.~\ref{fig:compare_cot_latent}).

However, effectively incentivizing these looped structures using reinforcement learning (RL) remains a challenge. Traditional Reinforcement Learning with Verifiable Rewards (RLVR) paradigms primarily operate on output tokens, creating a structural mismatch with looped architectures whose reasoning unfolds implicitly~\citep{shao2024deepseekmathpushinglimitsmathematical,wen2025reinforcementlearningverifiablerewards}. As identified in Ouro~\citep{zhu2025scaling}, the dynamic early-exit mechanism, which adaptively sets computation depth, introduces instability and credit-assignment issues for RLVR alignment. This leads to a fundamental research question:

\begin{figure*}
    \centering
    \begin{subfigure}[b]{0.64\textwidth}  % b图占48%宽度，确保两个图总宽度小于1
        \centering
        \includegraphics[width=\textwidth]{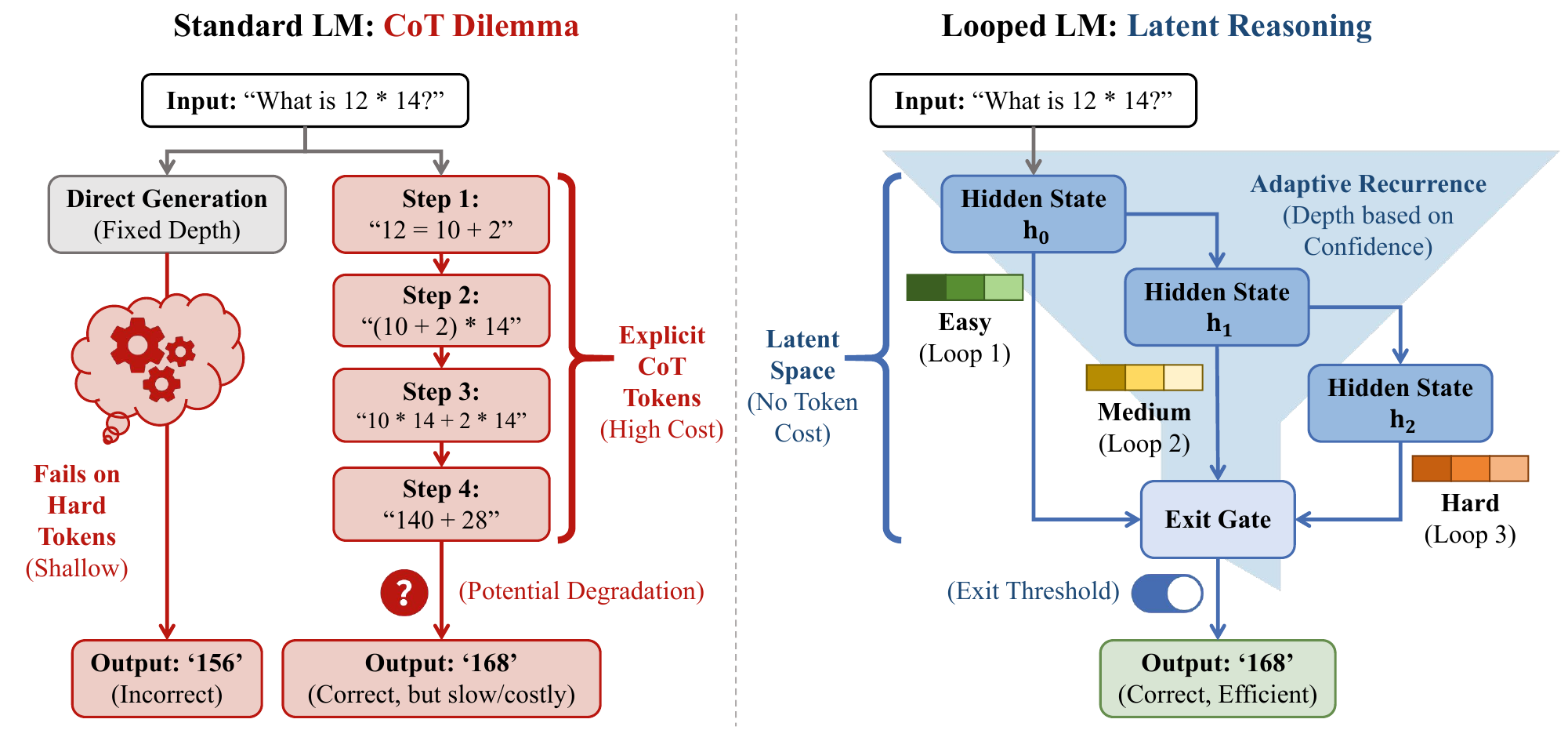}
        \caption{Comparison of explicit CoT reasoning and latent space reasoning.}
        \label{fig:compare_cot_latent}
    \end{subfigure}
    \hfill
    \begin{subfigure}[b]{0.33\textwidth}  % a图占48%宽度
        \centering
        \includegraphics[width=\textwidth]{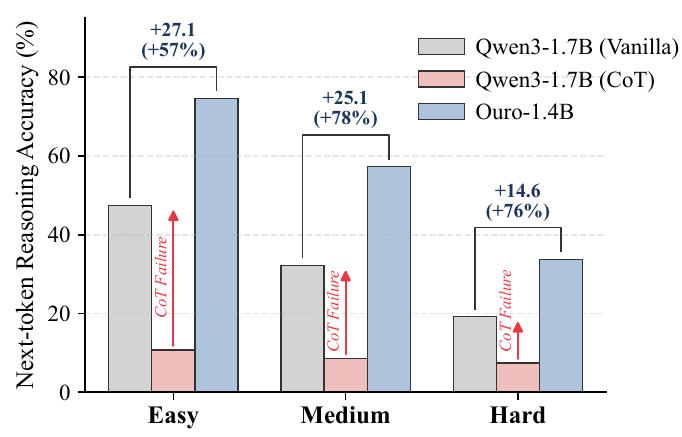}
        \caption{Performance comparison 
 between Ouro-1.4B and Qwen3-1.7B on next-token reasoning tasks across difficulty.}
        \label{fig:compare_perform_ntr}
    \end{subfigure}

    %\vspace{3pt}  % 添加一些垂直间距
    
    \caption{
        \textbf{Motivation for LoopRPT.}
        \textbf{(a)} Standard LLMs reason through explicit chain-of-thought tokens, whereas LoopLMs compress multi-step reasoning into latent space. \textbf{(b)} LoopLMs achieve higher next-token-reasoning accuracy than standard LLMs across difficulty levels.
    }
    \label{fig:ab_distribution}
\end{figure*}

\begin{center}
    \textbf{\textit{How to incentivize LoopLMs using RL with performance gains compared to non-recursive transformer models?}}
\end{center}

Current RLVR tasks rely on sparse rewards derived from final output tokens, which are insufficient for the dense, multi-step latent reasoning required by LoopLMs. Recently, \textbf{Reinforcement Pre-Training (RPT)} has emerged as a new scaling paradigm, reframing next-token prediction as a reasoning task~\citep{hatamizadeh2025rlpreinforcementpretrainingobjective,method:rpt}. By treating internal computation as an exploratory action and rewarding the predictive gain for future tokens, RPT turns vast text data into dense, self-supervised RL signals, a path toward training looped architectures.

In this work, we make the first attempt to unify looped architectures with the RPT paradigm. Our initial analysis (Fig.~\ref{fig:compare_perform_ntr}) shows that LoopLMs are naturally ``reasoning-friendly" for next-token tasks: vanilla LLMs, lacking training for CoT-based next-token reasoning, degrade sharply, while looped models stay accurate across difficulty levels. This suggests that the knowledge-manipulation strength of LoopLMs can be unlocked by an RL algorithm designed for latent recurrence.

To fill this gap, we propose \textbf{LoopRPT}, a \textbf{R}einforcement \textbf{P}re-\textbf{T}raining Framework tailored for \textbf{Loop}ed LMs. Built upon the next-token reasoning task, LoopRPT assigns reinforcement signals directly to latent reasoning steps. The framework incorporates three key innovations: (i) an \textbf{entropy-based selector} that identifies ``hard tokens" where reasoning is most beneficial~\citep{wang20258020rulehighentropyminority} , (ii) \textbf{step-wise rewards} computed against a dynamic EMA teacher to shape intermediate representations~\citep{tarvainen2017mean,hatamizadeh2025rlpreinforcementpretrainingobjective} , and (iii) \textbf{noisy latent rollouts} to jointly optimize the exit policy and backbone representations. By combining accuracy gains with a difficulty-aware time penalty, LoopRPT encourages early effective exits without sacrificing correctness on challenging tokens.

We conduct experiments on Ouro-1.4B and Ouro-2.6B using the \textsc{Omni-Math}~\citep{gao2024omnimathuniversalolympiadlevel} and general reasoning datasets. LoopRPT improves per-step representation quality and achieves a better accuracy-computation trade-off: fewer average steps while maintaining or improving hard-token accuracy. Gains transfer to downstream benchmarks, with improvements in GSM8K~\citep{bench:gsm8k}, MBPP~\citep{bench:mbpp}, and HumanEval~\citep{bench:humaneval}.

Our primary contributions are summarized as follows: 
\begin{itemize}
[itemsep=2pt,topsep=2pt,parsep=0pt,leftmargin=*]
    \item We introduce \textbf{LoopRPT}, the first framework to apply the Reinforcement Pre-Training paradigm to Looped LMs, addressing the mismatch between latent reasoning and sparse token-level rewards.
    \item We design a \textbf{step-wise reinforcement objective} using an EMA teacher and noisy latent rollouts to provide dense supervision over internal recurrent transitions.
    \item Empirical results show that LoopRPT achieves Pareto dominance in accuracy-computation trade-offs across multiple scales, outperforming both vanilla looped models and standard LLMs with explicit CoT.
\end{itemize}

\section{Preliminaries}
\label{sec:preliminary}
\noindent\textbf{Looped Language Models (LoopLMs).}
Unlike standard transformers, LoopLMs refine internal representations through iterative latent computation. Given context $x_{<t}$, the model updates a latent state $\mathbf{h}^{(k)}$ via a parameter-shared backbone $f_\theta$: \begin{equation} \mathbf{h}^{(k+1)} = f_\theta(\mathbf{h}^{(k)}, x_{<t}), \quad k = 0, \dots, K-1 , \end{equation} The next-token distribution is then derived from the final state: $p_\theta(x_t \mid x_{<t}) = g_\theta(\mathbf{h}^{(K)})$, where $g_\theta$ represents an LM head. This enables multi-step reasoning in latent space without emitting intermediate tokens.

\noindent\textbf{Standard Next-Token Prediction (NTP)}
remains the dominant objective, maximizing the log-likelihood: $\mathcal{L}_{\mathrm{NTP}}(\theta) = - \sum \log p_\theta(x_t \mid x_{<t})$. However, NTP treats predictions independently and does not model the latent reasoning required for difficult tokens.

\noindent\textbf{Reinforcement Pre-Training (RPT)}
 reframes prediction as a reasoning task. The model samples an output $o_t \sim \pi_\theta$ and receives a verifiable reward $r(o_t, x_t)$ based on the ground-truth continuation. The objective maximizes the expected reward over the corpus: $\mathcal{J}_{\mathrm{RPT}}(\theta)
= \mathbb{E}_{x_{<t}} \, \mathbb{E}_{o_t \sim \pi_\theta(\cdot \mid x_{<t})}
\big[ r(o_t, x_t) \big] .$ 
Using intrinsic corpus rewards, RPT scales RL to large datasets and incentivizes next-token reasoning over simple likelihood maximization.

\section{Method}
\label{sec:method}

In the RLVR paradigm, rewards are sparse, derived solely from final answer tokens. For looped LMs that reason over multiple latent steps, this sparsity exacerbates credit assignment across intermediate iterations. We introduce dense step-wise reward signals on challenging next-token predictions, enabling looped LMs to improve intermediate reasoning and learn an adaptive exit policy.

LoopRPT is a reinforcement pre-training framework for looped LMs, where computation unfolds over latent steps and an exit mechanism sets the number of recurrent iterations per token. As illustrated in Fig.~\ref{fig:method}, LoopRPT (i) selects \emph{hard} next-token prediction instances via entropy, (ii) defines a \emph{step-wise} reward on latent iterations relative to an EMA teacher, and (iii) jointly optimizes the exit policy and intermediate representations using noisy latent rollouts and step-weighted next-token learning.~\footnote{Full training algorithms are provided in App.~\ref{app:algorithm}.}

\begin{figure*}[t]  
    \centering
    \includegraphics[width=\textwidth]{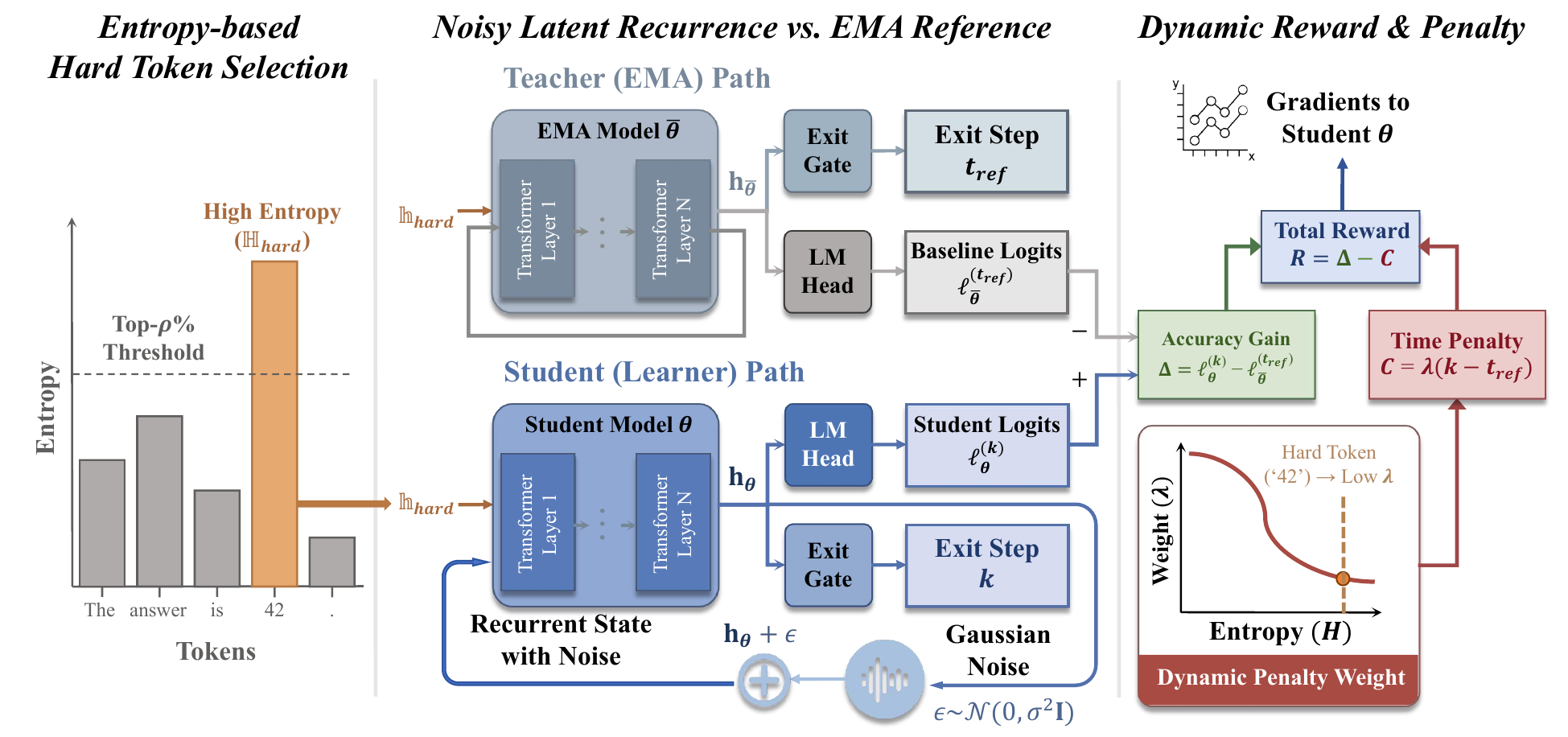}
    \caption{\textbf{LoopRPT} assigns reinforcement signals directly to latent reasoning steps. An entropy-based selector identifies hard tokens, and a student trains against an EMA teacher via step-wise rewards that combine an accuracy gain over the teacher with a dynamic time penalty, shaping intermediate representations and encouraging earlier effective exits.}
    \label{fig:method}
\end{figure*}

\subsection{Looped LM with an exit-step distribution}
\label{sec:method:looplm}

As in Sec.~\ref{sec:preliminary}, each latent step produces next-token logits and a scalar exit-gate logit $a^{(k)}$, with exit probability $\lambda^{(k)} = \sigma(a^{(k)})$.
Following Ouro's early-exit mechanism~\citep{zhu2025scaling}, 
we define a survival probability $s^{(1)}=1$ and update
\begin{align}
\pi_{\theta}(k) &=
\begin{cases}
\lambda^{(k)} s^{(k)}, & k<K,\\
s^{(K)}, & k=K,
\end{cases}
\label{eq:exit_prob} \\
s^{(k+1)} &= s^{(k)}\bigl(1-\lambda^{(k)}\bigr)\ \ (k<K),
\label{eq:s_recursion}
\end{align}

where the last step absorbs the remaining probability mass, ensuring $\sum_{k=1}^K \pi_{\theta}(k)=1$ exactly.
We compute the CDF $\Pi(k)=\sum_{j\le k}\pi_{\theta}(j)$ and determine the threshold exit step
$t(\tau) = \min\{k:\Pi(k)\ge \tau\}$.
For numerical robustness, if no step satisfies $\Pi(k)\ge\tau$ (e.g., due to finite precision), we set $t(\tau)=K$.

\subsection{Entropy-based hard-token selection}
\label{sec:method:selection}

Applying reinforcement learning uniformly to all tokens is inefficient because most tokens are easy and contribute weak learning signals. LoopRPT therefore trains primarily on \emph{hard} next-token instances identified by teacher uncertainty.\footnote{We provide theoretical analysis in App.~\ref{app:theory_entropy_select}.} Concretely, for each position $t$ in a sequence we compute the teacher's next-token distribution and entropy
\begin{equation}
H_t = -\sum_{v} p_{\bar\theta}(v\mid x_{<t}) \log p_{\bar\theta}(v\mid x_{<t}),
\label{eq:ref_entropy}
\end{equation}
and select a top fraction of positions (top-$\rho\%$) within each example as hard tokens.\footnote{Details are deferred to App.~\ref{app:hardtokens}.}
All losses in Sec.~\ref{sec:method:reward}--\ref{sec:method:opt} are applied only on the selected positions.

\subsection{Step-wise rewards with an EMA teacher}
\label{sec:method:reward}

In order to provide a dense credit assignment per-step anchored by a stable moving reference~\citep{hatamizadeh2025rlpreinforcementpretrainingobjective}, LoopRPT assigns reinforcement signals \emph{directly to latent reasoning steps} and maintains an EMA teacher $\bar\theta$. $\bar\theta$ tracks student parameters $\theta$ and serves as a dynamic reference for both exiting and reward baselines.

\noindent\textbf{Teacher reference step.}
For each token, we compute the teacher's exit distribution $\pi_{\bar\theta}$ and define the teacher reference step
$t_{\text{ref}} = \min\Bigl\{k:\sum_{j\le k}\pi_{\bar\theta}(j)\ge \tau\Bigr\}.$
Intuitively, $t_{\text{ref}}$ is the teacher's notion of sufficient computation at the current exit threshold.

\noindent\textbf{Accuracy gain.}
Let $\ell_\theta^{(k)} = \log p_\theta(x_t \mid x_{<t}; \mathbf{h}^{(k)})\in\mathbb{R}^{B\times(S-1)}$ denote the student log-probability of the ground-truth next token when read out from step $k$. We take the teacher log-probability at the reference step as a per-token baseline,
\begin{equation}
\label{eq:ref_logprob}
b_{\text{ref}} = \ell_{\bar\theta}^{(t_{\text{ref}})}\in\mathbb{R}^{B\times(S-1)}.
\end{equation}
The step-wise accuracy gain is
\begin{equation}
\label{eq:acc_delta}
\Delta_{\text{acc}}(k) = \ell_\theta^{(k)} - b_{\text{ref}}.
\end{equation}

\noindent\textbf{Difficulty-aware time penalty.}
To encourage earlier \emph{effective} exits without forcing premature termination on difficult tokens, we penalize excessive computation relative to the teacher reference step:
\begin{equation}
\label{eq:time_penalty}
C(k) = \lambda_t\,(k - t_{\text{ref}}).
\end{equation}
Here $\lambda_t$ is a token-dependent penalty weight derived from teacher uncertainty at the reference step. 
Specifically, we compute the entropy $H^{\text{ref}}_t$ at the teacher's reference step distribution,
 normalize it by $\log|\mathcal{V}|$ and clamp it to obtain a difficulty score
$d_t = \mathrm{Clamp}\bigl(H^{\text{ref}}_{t}/\log|\mathcal{V}|,\,0,\,1\bigr).$
We introduce the base coefficient $\lambda_{\text{base}}$ and the scaling coefficient $\lambda_{\text{scale}}$, 
thus $\lambda_t$ is defined as
\begin{equation}
\label{eq:lambda_t}
\lambda_t = \lambda_{\text{base}}\Bigl(1 + \lambda_{\text{scale}}(1-d_t)\Bigr),
\end{equation}
so that easier tokens (lower entropy) incur larger time penalties, while harder tokens are penalized less.

\noindent\textbf{Total step-wise reward.}
Combining the above, LoopRPT defines a reward for each latent step:
\begin{equation}
R(k) = \Delta_{\text{acc}}(k) - C(k).
\label{eq:step_reward}
\end{equation}
This reward provides dense supervision over latent iterations, shaping intermediate representations and aligning them with early exiting (Fig.~\ref{fig:method}).

Finally, we compute a per-token \emph{step advantage} by normalizing $R(k)$ across steps $k=1,\dots,K$:
\begin{equation}
\label{eq:advantage_nonoise}
\widehat{A}(k) = \frac{R(k)-\mu_R}{\sigma_R+\epsilon},
\quad
\mu_R=\frac{1}{K}\sum_{j=1}^K R(j),
\quad
\sigma_R = \sqrt{\frac{1}{K}\sum_{j=1}^K (R(j) - \mu_R)^2},
\end{equation}
which will be used to emphasize beneficial latent steps in the representation learning objective.

\subsection{Joint optimization via noisy latent rollouts}
\label{sec:method:opt}

LoopRPT jointly trains (i) the \emph{exit policy} induced by the gate distribution $\pi_\theta$ and (ii) the \emph{backbone} representations $\{\mathbf{h}^{(k)}\}$ that support accurate prediction at early steps.

\noindent\textbf{Noisy latent rollouts for the exit policy.}
As LoopLM reasoning is implicit in latent space, we obtain on-policy variability by injecting Gaussian noise into the recurrent hidden states.\footnote{We provide theoretical analysis in App.~\ref{app:theory_noise}.} For each selected token, we draw $G$ rollouts and perturb the latent states during recurrence:
\begin{equation}
\label{eq:Gaussian_noise}
\mathbf{h}^{(k)} \leftarrow \mathbf{h}^{(k)} + \boldsymbol{\epsilon}^{(k)}, \quad
\boldsymbol{\epsilon}^{(k)} \sim \mathcal{N}(0,\sigma^2 \mathbf{I}),
\end{equation}
yielding a rollout-specific exit distribution $\pi_\theta^{(g)}(\cdot)$.
We then sample an exit step $t^{(g)} \sim \pi_\theta^{(g)}$ and define the rollout reward by indexing the step-wise reward table in Eq.~\eqref{eq:step_reward},
$r^{(g)} = R\bigl(t^{(g)}\bigr)$.
To stabilize policy learning, we use group-wise normalization across the $G$ rollouts for the same input token:
\begin{equation}
\label{eq:group_advantage}
A^{(g)} = \frac{r^{(g)} - \mathrm{mean}_g[r^{(g)}]}{\mathrm{std}_g[r^{(g)}]+\epsilon}.
\end{equation}
The policy-gradient loss is
\begin{equation}
\label{eq:pg_loss}
\mathcal{L}_{\text{PG}}
=
-\mathbb{E}_{g}\bigl[A^{(g)} \log \pi_\theta^{(g)}(t^{(g)})\bigr],
\end{equation}
applied only on the selected hard-token positions.

\noindent\textbf{Step-weighted representation learning.}
LoopRPT also strengthens intermediate representations so that early latent steps can predict the correct token.
Using a deterministic (noise-free) forward pass, we compute per-step log-probabilities $\ell_\theta^{(k)}$ and the (noise-free) exit distribution $\pi_\theta(k)$. We then optimize a step-weighted next-token objective:
\begin{align}
\mathcal{L}_{\text{rep}}
&=
-\sum_{k=1}^K w_k \, \ell_\theta^{(k)},
\label{eq:backbone_loss} \\
w_k &= \text{softmax}\bigl[\pi_\theta(k)\bigl(1 + \mathrm{ReLU}(\widehat{A}(k))\bigr)\bigr].
\label{eq:bb_loss_weight}
\end{align}
The $\pi_\theta(k)$ term emphasizes likely exit steps; the advantage-shaping term focuses learning on higher-reward steps under Eq.~\eqref{eq:step_reward}.

\noindent\textbf{Regularization and total objective.}
We add an entropy bonus on the exit distribution to prevent early collapse,
$\mathcal{L}_{\text{ent}} = -\mathbb{E}\bigl[\sum_{k}\pi_\theta(k)\log(\pi_\theta(k))\bigr]$,
and a KL-style trust-region penalty to the EMA teacher based on a token-level surrogate computed from teacher and student step-wise log-probabilities.\footnote{Details are provided in App.~\ref{app:KL}.}
The final training loss is
\begin{equation}
\label{eq:total_loss}
\mathcal{L}
=
\alpha\,\mathcal{L}_{\text{PG}}
+
\beta\,\mathcal{L}_{\text{rep}}
+
\gamma\,\mathcal{L}_{\text{ent}}
+
\delta\,\mathcal{L}_{\text{KL}}.
\end{equation}
After each update to $\theta$, we update the EMA teacher parameters $\bar\theta$.~\footnote{We provide theoretical analysis in App.~\ref{app:theory_ema}.}

\section{Experimental Setup}
\label{sec:setup}

%\subsection{Training Configuration}
%\paragraph{Training setup.}
We use \textsc{Omni-Math}~\citep{gao2024omnimathuniversalolympiadlevel} as the training data. \textsc{Omni-Math} contains 4{,}428 competition-level mathematical problems paired with solutions, among which 200 examples are held out as the validation set. We first conduct experiments on Ouro-1.4B and then scale up to Ouro-2.6B. Both the teacher model and the student model are initialized from the same base model. The teacher parameters are maintained as an exponential moving average of the current student parameters (momentum $0.995$), providing a stable reference for generating outputs, while the student model is updated via backpropagation. For token-level updates, we treat the teacher model as the entropy proxy and compute token-level entropy based on its output distribution. We then construct an entropy mask to suppress gradient updates on low-entropy tokens~\citep{wang20258020rulehighentropyminority}. \footnote{Full hyperparameters and training details are deferred to Appendix~\ref{app:setup_details} (Table~\ref{tab:hyperparameters}).}

%\subsection{Evaluation Setting}
\noindent\textbf{Language modeling.}
Following RPT~\citep{method:rpt}, we evaluate the performance on the held-out validation set of 200 examples from \textsc{Omni-Math}. We bucket tokens by their entropy and report performance on easy/medium/hard splits, defined by thresholds 0.5, 1.0 and 1.5, respectively. For the Ouro models, we adopt the corresponding base model as an entropy proxy and evaluate performance under two inference settings: maximum-loop execution and adaptive early exit. We also evaluate the non-loop Qwen3-1.7B under two settings: the vanilla mode that performs NTP directly, and the CoT mode that first generates a chain-of-thought and then performs NTP.

\noindent\textbf{End tasks.}
We evaluate zero-shot performance on MMLU~\citep{bench:mmlu}, MMLU-Pro~\citep{bench:mmlu_pro}, BBH~\citep{bench:bbh}, ARC-C~\citep{bench:arc_challenge}, HellaSwag~\citep{bench:hellaswag}, Winogrande~\citep{bench:winogrande}, GSM8K~\citep{bench:gsm8k}, MBPP~\citep{bench:mbpp}, MBPP+~\citep{bench:evalplus}, HumanEval~\cite{bench:humaneval}, and HumanEval+~\citep{bench:evalplus}. Following Ouro~\citep{zhu2025scaling}, all evaluations were conducted using lm-eval-harness~\citep{bench:eval-harness} and evalplus~\citep{bench:evalplus}. \footnote{Detailed evaluation settings and metrics are provided in Appendix~\ref{app:setup_details} (Table~\ref{tab:eval_setting}).}

\noindent\textbf{Baselines and additional evaluations.}
To separate the effect of the objective from additional training, we compare LoopRPT against two compute-matched controls that share its checkpoint, data, step budget, and schedule: Continued-NTP (SFT) and GRPO/RLVR. We also evaluate MATH500, the full Hendrycks MATH set, and AIME 2024/2025 for transfer to olympiad-level mathematics, add a no-CoT zero-shot pass over MMLU, GSM8K, and BBH, and repeat the recipe on the broader OpenThoughts-3 corpus. \footnote{The full compute-matched matrix, the additional evaluation protocols, and the broader-corpus study are given in Appendices~\ref{app:attribution} and~\ref{app:scaling}.}

\section{Experimental Results}
\label{sec:result}

\subsection{Main Results}

\begin{wrapfigure}{r}{0.45\columnwidth}
    \centering
    \includegraphics[width=\linewidth]{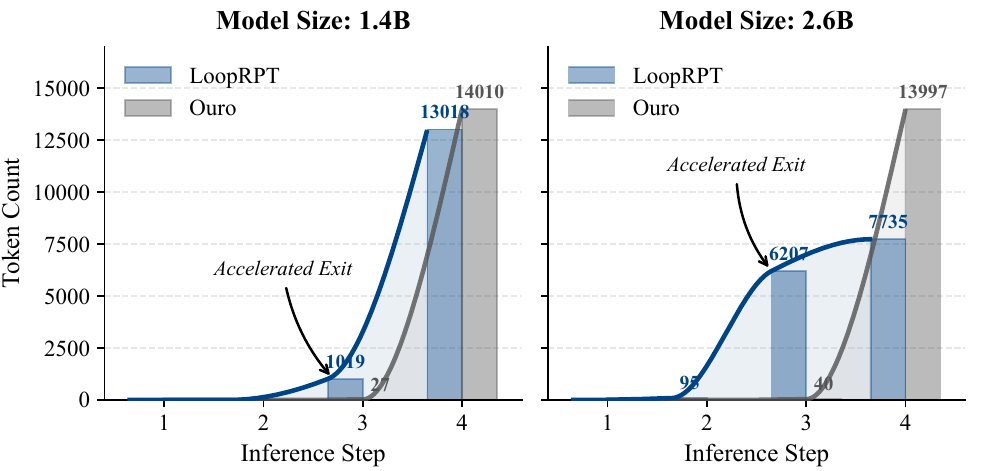}
    \caption{Distribution of exit steps on MMLU. LoopRPT shifts tokens to earlier iterations while keeping final-step exits dominant.}
    \label{fig:step_distribution}
\end{wrapfigure}
\textbf{Next-token Reasoning (NTR).} Table~\ref{tab:ntr} shows that LoopRPT consistently improves accuracy and efficiency across all scales. On Ouro-2.6B, LoopRPT boosts Hard Peak accuracy by +3.58 while reducing average steps from 3.51 to 2.28. While explicit CoT prompting on Qwen3-1.7B severely degrades performance, LoopRPT successfully compresses effective computation into fewer latent iterations without sacrificing correctness.

\textbf{Downstream Benchmarks.} As summarized in Tables~\ref{tab:main_14b} and \ref{tab:main_26b}, LoopRPT delivers consistent gains over Ouro, particularly in coding (MBPP+ +2.91) and mathematical reasoning. At the 2.6B scale, GSM8K accuracy increases from 81.76 to 85.36, alongside improvements in HumanEval+, indicating that reinforcement pre-training on hard instances effectively transfers to complex program synthesis.

\textbf{Exit Dynamics.} LoopRPT strengthens adaptive exiting by increasing early-step exits while maintaining final-step dominance (Fig.~\ref{fig:step_distribution}). At larger scales, the mechanism becomes more aggressive, with a shift towards step-3 exits. This reflects improved confidence calibration and earlier latent convergence, aligning with the accuracy--compute Pareto dominance observed in NTR.

\begin{table*}[t]
\centering
\caption{Performance comparison on next-token reasoning task across three difficulty levels. \textit{Peak} indicates reasoning up to the maximum latent loops ($K=4$), while \textit{Adap.} indicates adaptive early exiting. Subscripts indicate improvement (${\color{OliveGreen}+}$) or degradation (${\color{Maroon}-}$) compared to baseline.}
\label{tab:ntr}
\footnotesize % 保持字体一致性，不使用 resizebox
\renewcommand{\arraystretch}{0.94}
\setlength{\tabcolsep}{2.8pt} % 优化列间距
\begin{tabular}{@{}ll ll ll ll@{}}
\toprule
\multirow{2}{*}{\textbf{Model}} & \multirow{2}{*}{\textbf{Method}} & \multicolumn{2}{c}{\textbf{Easy}} & \multicolumn{2}{c}{\textbf{Medium}} & \multicolumn{2}{c}{\textbf{Hard}} \\ 
\cmidrule(lr){3-4} \cmidrule(lr){5-6} \cmidrule(lr){7-8} 
 & & Acc ($\uparrow$) & Avg step ($\downarrow$) & Acc ($\uparrow$) & Avg step ($\downarrow$) & Acc ($\uparrow$) & Avg step ($\downarrow$) \\ \midrule

% --- Qwen Group ---
\multirow{2}{*}{Qwen3-1.7B} & Vanilla & 47.49 & -- & 32.18 & -- & 19.19 & -- \\
 & +CoT & 10.70{\color{Maroon!80!black}$_{\scriptscriptstyle-36.79}$} & -- & 8.60{\color{Maroon!80!black}$_{\scriptscriptstyle-23.58}$} & -- & 7.44{\color{Maroon!80!black}$_{\scriptscriptstyle-11.75}$} & -- \\ 
\addlinespace[0.2em] \midrule

% --- Ouro 1.4B Group ---
\multirow{4}{*}{Ouro-1.4B} & Peak & 74.62 & 4.00 & 57.28 & 4.00 & 33.79 & 4.00 \\
 & \cellcolor[HTML]{F2F7FF}\textbf{+LoopRPT} & \cellcolor[HTML]{F2F7FF}\textbf{75.38}{\color{OliveGreen!80!black}$_{\scriptscriptstyle+0.76}$} & \cellcolor[HTML]{F2F7FF}4.00 & \cellcolor[HTML]{F2F7FF}\textbf{58.29}{\color{OliveGreen!80!black}$_{\scriptscriptstyle+1.01}$} & \cellcolor[HTML]{F2F7FF}4.00 & \cellcolor[HTML]{F2F7FF}\textbf{34.74}{\color{OliveGreen!80!black}$_{\scriptscriptstyle+0.95}$} & \cellcolor[HTML]{F2F7FF}4.00 \\ 
%\cmidrule(lr){2-8}
 & Adap. & 74.40 & 3.34 & 57.20 & 3.53 & 33.91 & 3.75 \\
 & \cellcolor[HTML]{F2F7FF}\textbf{+LoopRPT} & \cellcolor[HTML]{F2F7FF}\textbf{75.00}{\color{OliveGreen!80!black}$_{\scriptscriptstyle+0.60}$} & \cellcolor[HTML]{F2F7FF}\textbf{2.50}{\color{NavyBlue!80!black}$_{\scriptscriptstyle-0.84}$} & \cellcolor[HTML]{F2F7FF}\textbf{57.72}{\color{OliveGreen!80!black}$_{\scriptscriptstyle+0.52}$} & \cellcolor[HTML]{F2F7FF}\textbf{2.81}{\color{NavyBlue!80!black}$_{\scriptscriptstyle-0.72}$} & \cellcolor[HTML]{F2F7FF}\textbf{34.82}{\color{OliveGreen!80!black}$_{\scriptscriptstyle+0.91}$} & \cellcolor[HTML]{F2F7FF}\textbf{3.07}{\color{NavyBlue!80!black}$_{\scriptscriptstyle-0.68}$} \\ 
\addlinespace[0.2em] \midrule

% --- Ouro 2.6B Group ---
\multirow{4}{*}{Ouro-2.6B} & Peak & 74.33 & 4.00 & 57.19 & 4.00 & 34.52 & 4.00 \\
 & \cellcolor[HTML]{F2F7FF}\textbf{+LoopRPT} & \cellcolor[HTML]{F2F7FF}\textbf{76.89}{\color{OliveGreen!80!black}$_{\scriptscriptstyle+2.56}$} & \cellcolor[HTML]{F2F7FF}4.00 & \cellcolor[HTML]{F2F7FF}\textbf{61.15}{\color{OliveGreen!80!black}$_{\scriptscriptstyle+3.96}$} & \cellcolor[HTML]{F2F7FF}4.00 & \cellcolor[HTML]{F2F7FF}\textbf{38.10}{\color{OliveGreen!80!black}$_{\scriptscriptstyle+3.58}$} & \cellcolor[HTML]{F2F7FF}4.00 \\ 
%\cmidrule(lr){2-8}
 & Adap. & 74.51 & 3.24 & 57.35 & 3.35 & 34.35 & 3.51 \\
 & \cellcolor[HTML]{F2F7FF}\textbf{+LoopRPT} & \cellcolor[HTML]{F2F7FF}\textbf{76.07}{\color{OliveGreen!80!black}$_{\scriptscriptstyle+1.56}$} & \cellcolor[HTML]{F2F7FF}\textbf{2.05}{\color{NavyBlue!80!black}$_{\scriptscriptstyle-1.19}$} & \cellcolor[HTML]{F2F7FF}\textbf{60.21}{\color{OliveGreen!80!black}$_{\scriptscriptstyle+2.86}$} & \cellcolor[HTML]{F2F7FF}\textbf{2.18}{\color{NavyBlue!80!black}$_{\scriptscriptstyle-1.17}$} & \cellcolor[HTML]{F2F7FF}\textbf{37.24}{\color{OliveGreen!80!black}$_{\scriptscriptstyle+2.89}$} & \cellcolor[HTML]{F2F7FF}\textbf{2.28}{\color{NavyBlue!80!black}$_{\scriptscriptstyle-1.23}$} \\ \bottomrule
\end{tabular}
\end{table*}
\begin{table*}[t]
\centering
\caption{Comparison of 1.4B LoopLM model with 1-4B parameter baselines. Results for vanilla LLMs are taken from \cite{zhu2025scaling}, with Ouro results obtained via our reimplementation. The best score is \textbf{bolded}, and the second-best is \underline{underlined}. Subscripts indicate LoopRPT's improvement over Ouro baseline.}
\label{tab:main_14b}
\footnotesize % 使用标准小号字体，确保版面紧凑且不失真
\renewcommand{\arraystretch}{0.94}
\setlength{\tabcolsep}{2.5pt} % 针对多列数据进行的极限间距优化
\begin{tabular}{l ccccccccc >{\columncolor[HTML]{ECF4FF}}c}
\toprule
\textbf{Benchmark} & 
\begin{tabular}[c]{@{}c@{}}Gemma3\\1B\end{tabular} & 
\begin{tabular}[c]{@{}c@{}}Llama3.2\\1.2B\end{tabular} & 
\begin{tabular}[c]{@{}c@{}}Qwen2.5\\1.5B\end{tabular} & 
\begin{tabular}[c]{@{}c@{}}Qwen3\\1.7B\end{tabular} & 
\begin{tabular}[c]{@{}c@{}}Qwen2.5\\3B\end{tabular} & 
\begin{tabular}[c]{@{}c@{}}Llama3.2\\3B\end{tabular} & 
\begin{tabular}[c]{@{}c@{}}Qwen3\\4B\end{tabular} & 
\begin{tabular}[c]{@{}c@{}}Gemma3\\4B\end{tabular} & 
\begin{tabular}[c]{@{}c@{}}Ouro\\1.4B\end{tabular} & 
\begin{tabular}[c]{@{}c@{}}\textbf{LoopRPT}\\\textbf{1.4B}\end{tabular} \\ \midrule
\rowcolor[HTML]{F5F5F5} \multicolumn{11}{c}{\textit{\textbf{General Tasks}}} \\
MMLU       & 39.85 & 45.46 & 60.99 & 62.46 & 65.62 & 59.69 & \textbf{73.19} & 58.37 & 67.35 & \underline{67.62}{\color{OliveGreen!90!black}$_{\scriptscriptstyle+0.27}$} \\
MMLU-Pro   & 11.31 & 11.80 & 29.11 & 37.27 & 37.87 & 33.34 & \textbf{51.40} & 34.61 & 48.64 & \underline{49.21}{\color{OliveGreen!90!black}$_{\scriptscriptstyle+0.57}$} \\
BBH        & 30.26 & 30.72 & 43.66 & 53.51 & 55.37 & 39.45 & \underline{70.95} & 66.32 & 70.97 & \textbf{71.37}{\color{OliveGreen!90!black}$_{\scriptscriptstyle+0.40}$} \\
ARC-C      & 39.25 & 41.98 & 54.44 & 55.72 & 55.46 & 52.47 & \textbf{63.65} & 60.92 & 60.95 & \underline{61.10}{\color{OliveGreen!90!black}$_{\scriptscriptstyle+0.15}$} \\
HellaSwag  & 56.12 & 59.35 & 67.73 & 67.09 & 74.54 & 73.09 & \textbf{75.66} & \underline{75.58} & 74.30 & 74.93{\color{OliveGreen!90!black}$_{\scriptscriptstyle+0.63}$} \\
Winogrande & 58.72 & 62.75 & 66.77 & 66.30 & 70.17 & 69.14 & 71.19 & 71.07 & \underline{71.67} & \textbf{71.98}{\color{OliveGreen!90!black}$_{\scriptscriptstyle+0.31}$} \\ \midrule
\rowcolor[HTML]{F5F5F5} \multicolumn{11}{c}{\textit{\textbf{Math \& Coding Tasks}}} \\
GSM8K      & 2.05  & 7.05  & 60.73 & 70.28 & 74.60 & 67.20 & 72.86 & 68.69 & \underline{78.17} & \textbf{78.92}{\color{OliveGreen!90!black}$_{\scriptscriptstyle+0.75}$} \\
HumanEval  & 6.70  & 19.50 & 52.40 & 66.50 & 68.90 & 29.90 & \textbf{77.40} & 34.80 & 69.47 & \underline{69.51}{\color{OliveGreen!90!black}$_{\scriptscriptstyle+0.04}$} \\
HumanEval+ & 5.50  & 17.40 & 46.30 & 59.80 & 62.20 & 26.20 & \textbf{70.70} & 29.30 & 67.01 & \underline{67.07}{\color{OliveGreen!90!black}$_{\scriptscriptstyle+0.06}$} \\
MBPP       & 12.40 & 35.70 & 60.30 & 68.00 & 63.00 & 50.30 & \textbf{78.80} & 60.60 & 73.02 & \underline{75.40}{\color{OliveGreen!90!black}$_{\scriptscriptstyle+2.38}$} \\
MBPP+      & 10.10 & 29.10 & 50.00 & 58.50 & 54.20 & 39.70 & \textbf{65.90} & 51.10 & 60.85 & \underline{63.76}{\color{OliveGreen!90!black}$_{\scriptscriptstyle+2.91}$} \\ \bottomrule
\end{tabular}
\end{table*}

\subsection{Compute-Matched Controls}
\label{subsec:attribution}
\begin{table*}[b]
\centering
\caption{\textbf{Compute-matched controls.} All arms share the checkpoint, data, step budget, and schedule; only the objective changes. NTR-Hard is hard-split accuracy under adaptive early exit; best per scale in \textbf{bold}.}
\label{tab:attribution}
\scriptsize
\renewcommand{\arraystretch}{0.94}
\setlength{\tabcolsep}{4pt}
\begin{tabular}{ll ccccccc}
\toprule
\textbf{Scale} & \textbf{Method} & \textbf{NTR-Hard} & \textbf{GSM8K} & \textbf{MATH500} & \textbf{ARC-C} & \textbf{HellaSwag} & \textbf{BBH} & \textbf{MBPP+} \\
\midrule
\multirow{4}{*}{Ouro-1.4B}
 & Base                 & 33.91 & 78.17 & 20.60 & 60.95 & 74.30 & 70.97 & 60.85 \\
 & $+$ Continued-NTP    & 34.29 & 77.63 & 29.40 & 60.58 & 74.51 & 71.10 & 62.17 \\
 & $+$ GRPO / RLVR      & 33.00 & 4.17  & 12.60 & 49.23 & 74.36 & 28.84 & 20.58 \\
 & \cellcolor[HTML]{F2F7FF}$+$ \textbf{LoopRPT} & \cellcolor[HTML]{F2F7FF}\textbf{34.82} & \cellcolor[HTML]{F2F7FF}\textbf{78.92} & \cellcolor[HTML]{F2F7FF}\textbf{37.00} & \cellcolor[HTML]{F2F7FF}\textbf{61.10} & \cellcolor[HTML]{F2F7FF}\textbf{74.93} & \cellcolor[HTML]{F2F7FF}\textbf{71.37} & \cellcolor[HTML]{F2F7FF}\textbf{63.76} \\
\midrule
\multirow{4}{*}{Ouro-2.6B}
 & Base                 & 34.35 & 81.76 & 43.00 & 66.12 & 79.28 & 77.98 & 64.81 \\
 & $+$ Continued-NTP    & 35.21 & 82.41 & 43.20 & 66.75 & 79.74 & 78.19 & 63.76 \\
 & $+$ GRPO / RLVR      & 33.00 & 6.63  & 18.00 & 49.89 & 78.07 & 29.36 & 57.58 \\
 & \cellcolor[HTML]{F2F7FF}$+$ \textbf{LoopRPT} & \cellcolor[HTML]{F2F7FF}\textbf{37.24} & \cellcolor[HTML]{F2F7FF}\textbf{85.36} & \cellcolor[HTML]{F2F7FF}\textbf{43.50} & \cellcolor[HTML]{F2F7FF}\textbf{66.89} & \cellcolor[HTML]{F2F7FF}\textbf{80.03} & \cellcolor[HTML]{F2F7FF}\textbf{78.24} & \cellcolor[HTML]{F2F7FF}\textbf{65.08} \\
\bottomrule
\end{tabular}
\end{table*}

Comparing LoopRPT against the original Ouro checkpoint alone cannot separate the reinforcement objective from additional training on \textsc{Omni-Math}. We add two control arms that share its checkpoint, data, step budget, and schedule, changing only the objective: \emph{Continued-NTP} uses plain next-token prediction, and \emph{GRPO/RLVR} uses group-relative policy optimization on output-token rewards. Table~\ref{tab:attribution} reports both at both scales. Continued-NTP stays close to the baseline on every metric except MATH500, so a further pass over the corpus does not reproduce the gains. GRPO fails to converge: GSM8K falls to a few percent, BBH drops by over 40 points, and ARC-C and MBPP+ degrade, matching the credit-assignment difficulty of sparse output rewards over latent iterations (Sec.~\ref{sec:method}). LoopRPT keeps the same budget and raises every metric in the table, and its margin over Continued-NTP isolates the objective itself.

\subsection{Transfer to Hard Mathematics and No-CoT Reasoning}
\label{subsec:transfer}
\begin{figure}[t]
    \centering
    \includegraphics[width=0.9\columnwidth]{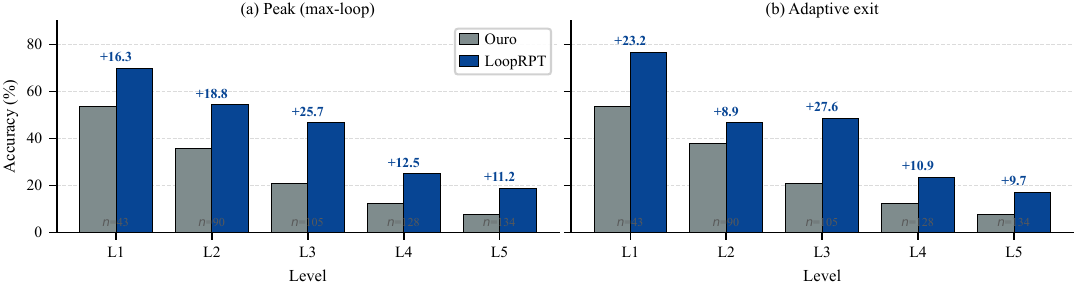}
    \caption{\textbf{MATH500 accuracy by difficulty level} (Ouro-1.4B vs.\ LoopRPT-1.4B). Levels use the official Hendrycks problem-level labels, which are independent of the token-level entropy that drives token selection during training. LoopRPT improves every level under both maximum-loop (a) and adaptive early exit (b); the per-level sample sizes are shown at the bottom of each panel.}
    \label{fig:math500_levels}
\end{figure}

\begin{table*}[t]
\centering
\begin{minipage}[t]{0.43\textwidth}
    \centering
    \captionof{table}{\textbf{Transfer to hard mathematics.} Subscripts give the change over Ouro.}
    \label{tab:transfer_math}
    \footnotesize
    \fitwidth{%
    \begin{tabular}{@{}llcc@{}}
        \toprule
        \textbf{Scale} & \textbf{Benchmark} & \textbf{Ouro} & \textbf{LoopRPT} \\
        \midrule
        \multirow{5}{*}{1.4B}
         & MATH500     & 20.60 & 37.00{\color{OliveGreen!90!black}$_{\scriptscriptstyle+16.40}$} \\
         & MATH (full) & 71.74 & 76.20{\color{OliveGreen!90!black}$_{\scriptscriptstyle+4.46}$} \\
         & AIME 2024   & 6.67  & 6.67 \\
         & AIME 2025   & 6.67  & 10.00{\color{OliveGreen!90!black}$_{\scriptscriptstyle+3.33}$} \\
         & GSM8K       & 78.17 & 78.92{\color{OliveGreen!90!black}$_{\scriptscriptstyle+0.75}$} \\
        \midrule
        \multirow{3}{*}{2.6B}
         & MATH500     & 43.00 & 43.50{\color{OliveGreen!90!black}$_{\scriptscriptstyle+0.50}$} \\
         & MATH (full) & 76.50 & 77.50{\color{OliveGreen!90!black}$_{\scriptscriptstyle+1.00}$} \\
         & GSM8K       & 81.76 & 85.36{\color{OliveGreen!90!black}$_{\scriptscriptstyle+3.60}$} \\
        \bottomrule
    \end{tabular}}
\end{minipage}
\hfill
\begin{minipage}[t]{0.53\textwidth}
    \centering
    \captionof{table}{Ablation study of LoopRPT-1.4B. \textit{Overall} is the average of the three difficulty levels.}
    \label{tab:ablation}
    \fitwidth{%
    \begin{tabular}{@{}lcccccc@{}}
        \toprule
        \multirow{2}{*}{\textbf{Varies}} & \multicolumn{3}{c}{\textbf{Hard}} & \multicolumn{3}{c}{\textbf{Overall}} \\
        \cmidrule(lr){2-4} \cmidrule(lr){5-7}
        & Peak & Adap. & Avg Step & Peak & Adap. & Avg Step \\
        \midrule
        \textbf{LoopRPT} & \textbf{34.74} & \textbf{34.82} & \textbf{3.07} & \textbf{56.14} & \textbf{55.85} & \textbf{2.79} \\
        \midrule
        w/o Gaussian Noise & 34.69 & 34.76 & 3.29{\color{red!70!black}$_{\scriptstyle(\uparrow)}$} & 55.86 & 55.57 & 3.17{\color{red!70!black}$_{\scriptstyle(\uparrow)}$} \\
        w/o $\mathcal{L}_{\text{PG}} + \mathcal{L}_{\text{KL}}$ & 34.59 & 34.17 & 3.28{\color{red!70!black}$_{\scriptstyle(\uparrow)}$} & 55.89 & 55.51 & 3.14{\color{red!70!black}$_{\scriptstyle(\uparrow)}$} \\
        w/o $\mathcal{L}_{\text{rep}}+\mathcal{L}_{\text{ent}}$ & 34.02 & 33.75 & 3.46{\color{red!70!black}$_{\scriptstyle(\uparrow)}$} & 55.52 & 55.19 & 3.24{\color{red!70!black}$_{\scriptstyle(\uparrow)}$} \\
        w/o Token Selection & 34.60 & 34.52 & 2.89{\color{blue!60!black}$_{\scriptstyle(\downarrow)}$} & 55.76 & 55.41 & 2.49{\color{blue!60!black}$_{\scriptstyle(\downarrow)}$} \\
        w/o Time Penalty & 34.57 & 34.80 & 3.30{\color{red!70!black}$_{\scriptstyle(\uparrow)}$} & 55.94 & 55.78 & 3.19{\color{red!70!black}$_{\scriptstyle(\uparrow)}$} \\
        \bottomrule
    \end{tabular}}
\end{minipage}
\end{table*}

\textbf{Hard mathematics.} The original suite contains one mathematics benchmark, GSM8K, which Ouro-2.6B nearly saturates at 81.76. We add MATH500, the full 5{,}000-problem Hendrycks MATH set, and AIME 2024/2025, all scored with the same protocol for Ouro and LoopRPT. Table~\ref{tab:transfer_math} shows that the gain extends beyond the benchmark closest to the training distribution: MATH500 and the broader Hendrycks set both improve at each scale, and AIME 2025 improves as well. AIME 2024 (30 problems) is the only benchmark that stays flat, and it is the furthest from the training distribution.

\textbf{Difficulty stratification.} Fig.~\ref{fig:math500_levels} breaks MATH500 down by the official Hendrycks problem-level labels (Level~1--5), which are assigned per problem and are independent of the training-time token entropy. LoopRPT improves every level by 9 to 28 points across the two inference settings, including the hardest level.

\textbf{No-CoT zero-shot.} GSM8K and BBH use few-shot chain-of-thought in the original suite, so we also run a no-CoT zero-shot pass (few-shot $=0$, GSM8K answers directly, BBH uses the no-CoT template; Table~\ref{tab:nocot} in the appendix). No-CoT accuracy is far below the few-shot numbers on both, so explicit CoT overstates a looped model's internal reasoning. The gain survives the removal: LoopRPT raises no-CoT GSM8K by 13.7 points at 1.4B and 26.0 at 2.6B, its largest downstream gain in any setting. MMLU, unrelated to the training signal, barely changes.

\subsection{Ablation Study}
\label{subsec:ablation}

Table~\ref{tab:ablation} ablates key design choices of LoopRPT on Ouro-1.4B across \textit{Hard} and \textit{Overall} splits. Overall, removing any component degrades accuracy or efficiency, confirming that LoopRPT relies on multiple complementary ingredients.

\textbf{Stabilization and trajectory diversity.} Disabling Gaussian noise reduces accuracy and increases average exit steps, suggesting that stochastic latent rollouts facilitate exploration and earlier latent convergence. Similarly, removing policy-gradient and KL stabilizers leads to large performance drops and increased latency under adaptive inference, highlighting their importance for stable on-policy learning and calibrated early exiting.

\textbf{Representation shaping and efficiency.}
Removing $\mathcal{L}_{\text{rep}}+\mathcal{L}_{\text{ent}}$ yields the largest degradation in both accuracy and efficiency, underscoring that representation regularization and entropy shaping are vital for learning effective latent transitions.
While disabling token selection slightly reduces steps, it notably hurts Hard-split accuracy, proving that hard-token updates preferentially allocate capacity to difficult reasoning tasks.
Finally, removing the time penalty modestly hurts accuracy while consistently increasing average steps, confirming its role in explicitly incentivizing early exit and improving the accuracy--compute trade-off.

\begin{figure*}[t]
    \centering
    \begin{subfigure}[b]{0.28\textwidth}  
        \centering
        \includegraphics[width=\textwidth]{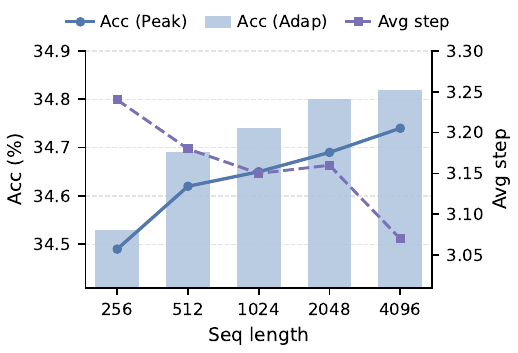}
        \caption{Effect of the maximum sequence length.}
        \label{fig:max_len}
    \end{subfigure}
    \hspace{0.03\textwidth}
    \begin{subfigure}[b]{0.28\textwidth}  
        \centering
        \includegraphics[width=\textwidth]{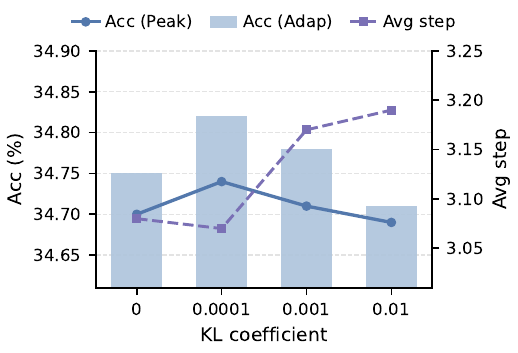}
        \caption{Effect of the KL regularization coefficient.}
        \label{fig:kl_coef}
    \end{subfigure}
    \hspace{0.03\textwidth} 
    \begin{subfigure}[b]{0.28\textwidth}  
        \centering
        \includegraphics[width=\textwidth]{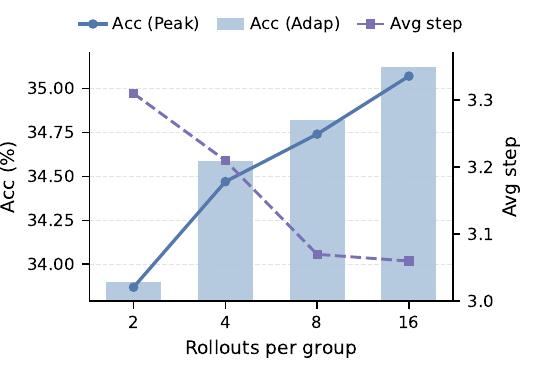}
        \caption{Effect of the group size during noisy latent rollouts.}
        \label{fig:rollout}
    \end{subfigure}
    %\vspace{3pt}  % 添加一些垂直间距
    
    \caption{
        Sensitivity analysis of LoopRPT hyperparameters on the next-token reasoning task.
All results are reported on the \textit{hard} split, showing accuracy and average exit steps under different settings.
    }
    \label{fig:rl_hyper}
\end{figure*}

\section{Analysis and Discussion}
\label{sec:discuss}
\subsection{Hyperparameter Sensitivity}

LoopRPT exhibits low sensitivity to key hyperparameters and follows a predictable accuracy--compute trade-off. Increasing the maximum sequence length (see Fig.~\ref{fig:max_len}) consistently improves both \textit{Peak} and \textit{Adap.} accuracy, while the average exit step slightly decreases or saturates, suggesting that longer context helps latent computation converge earlier. The KL regularization coefficient (Fig.~\ref{fig:kl_coef}) shows a mild sweet spot: moderate KL improves \textit{Adap.} performance, whereas overly strong KL tends to reduce gains and yields more conservative exiting (higher average steps), consistent with over-constraining on-policy updates. Finally, enlarging the GRPO group size (Fig.~\ref{fig:rollout}) improves accuracy and reduces average steps, indicating that better group-wise advantage estimates strengthen both learning signal and early-exit calibration, with diminishing returns at larger group sizes.

% \begin{figure}[bt]
%     \centering
%     \begin{minipage}{0.49\linewidth}
%         \centering
%         \includegraphics[width=\linewidth]{figures/exit_thres.pdf}
%         \captionof{figure}{Accuracy–computation trade-offs across difficulty buckets. LoopRPT enhances next-token reasoning accuracy with reduced loop cycles.}   % 不作为子图
%         \label{fig:exit_thres}
%     \end{minipage}
%     \hfill
%     \begin{minipage}{0.49\linewidth}
%         \centering
%         \includegraphics[width=\linewidth]{figures/step_wise.pdf}
%         \captionof{figure}{Per-step next-token reasoning accuracy across difficulty buckets.}
%         \label{fig:step_wise}
%     \end{minipage}
% \end{figure}

\subsection{Analysis of Latent Reasoning}
\begin{wrapfigure}{r}{0.45\columnwidth}  
    \centering
    \includegraphics[width=\linewidth]{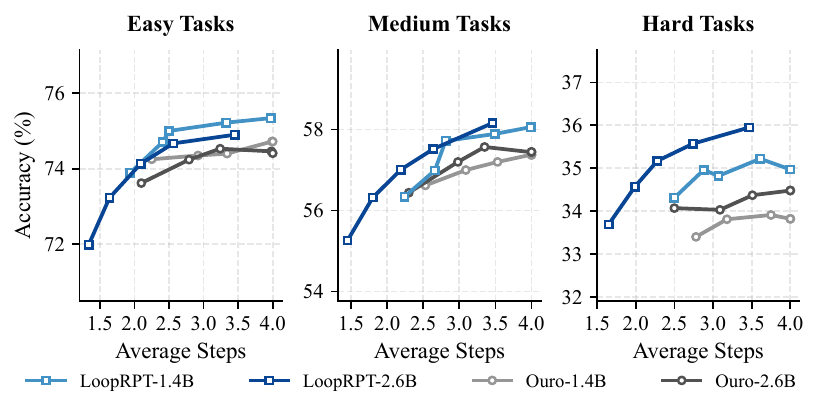}
    \caption{Accuracy–computation trade-offs across difficulty buckets. LoopRPT enhances next-token reasoning accuracy with reduced loop cycles.}
    \label{fig:exit_thres}
\end{wrapfigure}
To better understand how LoopRPT improves the accuracy–computation trade-off, we analyze the latent reasoning dynamics of LoopLMs at the level of intermediate computation steps. Fig.~\ref{fig:exit_thres} presents the accuracy–computation Pareto curves under varying exit thresholds. Across three difficulty buckets, LoopRPT consistently achieves higher accuracy with fewer average computation steps compared to the baseline Ouro model. Notably, this Pareto dominance holds across model scales, indicating that the gains introduced by LoopRPT are not tied to a specific parameter regime.

% \begin{wrapfigure}{r}{0.45\columnwidth} 
%     \centering
%     \includegraphics[width=\linewidth]{figures/step_wise.pdf}
%     \caption{Per-step next-token reasoning accuracy across difficulty buckets.}
%     \label{fig:step_wise}
% \end{wrapfigure}

To uncover the mechanism behind this improvement, Fig.~\ref{fig:step_wise} reports per-step next-token prediction accuracy obtained from intermediate hidden states within a single forward pass. We observe that LoopRPT improves prediction accuracy at every latent step, with particularly pronounced gains at early steps. This effect is most evident for hard tokens, where LoopRPT significantly boosts step-1 and step-2 accuracy.

These results suggest that LoopRPT does not merely encourage earlier exiting, but fundamentally enhances the quality of intermediate representations. By shaping latent reasoning through reinforcement pre-training, LoopRPT compresses effective computation into fewer steps while preserving or improving predictive performance.

\section{Related Work}
\label{sec:related}

\paragraph{Reinforcement learning for (pre-)training LMs.}
Modern RL-based post-training for LMs is often instantiated as RL from preference feedback, where a reward model or implicit reward is optimized with policy-gradient updates such as TRPO/PPO \citep{pmlr-v37-schulman15,schulman2017proximalpolicyoptimizationalgorithms} and human preference supervision \citep{NIPS2017_d5e2c0ad,NEURIPS2020_1f89885d,NEURIPS2022_b1efde53}. Subsequent work reduces reliance on human labels by leveraging AI feedback and rule-based supervision, e.g., Constitutional AI / RLAIF \citep{bai2022constitutionalaiharmlessnessai}, and broadens the toolkit beyond reward-model RLHF via direct/implicit preference optimization \citep{NEURIPS2023_a85b405e,ethayarajh2024ktomodelalignmentprospect,hong2024orpomonolithicpreferenceoptimization} and its extensions \citep{xiao2025comprehensivesurveydirectpreference,jiang2025preferencesurvey,cen2025valueincentivizedpreferenceoptimizationunified}. In parallel, RLVR exploit programmatic or checker-based supervision to improve reasoning without a learned reward model \citep{wen2025reinforcementlearningverifiablerewards,Guo_2025,shao2024deepseekmathpushinglimitsmathematical}, but are typically applied as task-specific post-training due to limited availability of verifiable supervision at scale. RPT reframes next-token prediction as a verifiable RL objective, deriving RL signals directly from the pre-training corpus \citep{method:rpt, hatamizadeh2025rlpreinforcementpretrainingobjective}. Our work follows this direction and addresses a key bottleneck: token-level feedback is sparse and dominated by easy transitions. By concentrating dense updates on hard next-token instances \citep{wang20258020rulehighentropyminority}, we improve both token-level reasoning and downstream transfer while preserving stable training dynamics.

\paragraph{Looped language models and adaptive depth.}
Recurrence and adaptive computation have a long history in sequence models. Early approaches introduce dynamic halting and step allocation, such as ACT \citep{graves2016act} and Universal Transformers \citep{dehghani2019universaltransformers}, later extended with sparse/shared-depth variants \citep{tan2023sparseuniversaltransformer} and learned pondering mechanisms \citep{banino2021pondernetlearningponder}. Complementary to within-token recurrence, segment recurrence and explicit memory target long-context modeling, e.g., Transformer-XL \citep{dai-etal-2019-transformer}, Compressive Transformers \citep{rae2019compressivetransformerslongrangesequence} and Recurrent Memory Transformers \citep{NEURIPS2022_47e28862}. Another related line is dynamic-depth inference via early exiting or layer-wise halting \citep{xin2020deebertdynamicearlyexiting,NEURIPS2020_d4dd111a,xu-mcauley-2023-survey,bajpai2025surveyearlyexitdeep}, and recent work allocates compute non-uniformly across tokens and layers \citep{raposo2024mixtureofdepthsdynamicallyallocatingcompute,chen2025innerthinkingtransformerleveraging}. Closest to our setting are Looped/Depth-recurrent LMs that explicitly build iterative latent computation into the model, enabling adaptive early exit at test time \citep{zhu2025scaling,wu2025parallellooptransformerefficient,li2025skiplayerloopit,mcleish2025teachingpretrainedlanguagemodels}, as well as related latent-pondering pretraining schemes \citep{zeng2025ponderlm2pretrainingllmlatent}. Concurrently, RLTT~\citep{williams2026prioritize} also explores distributing rewards across latent reasoning trajectories in looped architectures. However, RLTT is explicitly designed for task-specific post-training (RLVR) relying on final-answer correctness. Our work sits at this intersection: we use LoopLMs as the substrate and contribute an RL pre-training recipe that strengthens token-level reasoning and calibrates early-exit behavior, addressing the instability and weak supervision that arise when scaling latent recurrence with sparse token-level rewards.

\section{Conclusion}
\label{sec:conclusion}
We introduced LoopRPT, a reinforcement pre-training framework for looped language models that addresses the inherent sparsity of token-level feedback. LoopRPT concentrates on learning signals on hard next-token instances and provides step-wise supervision over latent recurrence. Experiments demonstrate that LoopRPT improves next-token reasoning across difficulty levels under both maximum-loop execution and adaptive early exit, while simultaneously reducing the average inference steps. These gains translate to consistent improvements on a diverse set of end-task benchmarks, particularly in math and code, and are accompanied by better-calibrated early-exit behavior. Together, our results suggest that combining looped computation with hard-token–focused RL pre-training is an effective and scalable path toward stronger intermediate reasoning and more efficient inference. Future work includes extending the approach to larger scales and broader data mixtures, and further improving robustness of early-exit calibration under distribution shift.

\section*{Acknowledgements}
The research in this article is supported by the National Science Foundation of China (U22B2059, 62276083), Key Research and Development Program of Heilongjiang Province (2024ZX01A05) , Suzhou Municipal Key Technology R\&D Project / Suzhou Municipal Key Core Technology Project (SYG2025069).

%\section*{References}
\bibliographystyle{unsrtnat}
\bibliography{example_paper}

@misc{bench:mmlu,
      title={Measuring Massive Multitask Language Understanding}, 
      author={Dan Hendrycks and Collin Burns and Steven Basart and Andy Zou and Mantas Mazeika and Dawn Song and Jacob Steinhardt},
      year={2021},
      eprint={2009.03300},
      archivePrefix={arXiv},
      primaryClass={cs.CY},
      url={https://arxiv.org/abs/2009.03300}, 
}

@inproceedings{bench:mmlu_pro,
 author = {Wang, Yubo and Ma, Xueguang and Zhang, Ge and Ni, Yuansheng and Chandra, Abhranil and Guo, Shiguang and Ren, Weiming and Arulraj, Aaran and He, Xuan and Jiang, Ziyan and Li, Tianle and Ku, Max and Wang, Kai and Zhuang, Alex and Fan, Rongqi and Yue, Xiang and Chen, Wenhu},
 booktitle = {Advances in Neural Information Processing Systems},
 doi = {10.52202/079017-3018},
 editor = {A. Globerson and L. Mackey and D. Belgrave and A. Fan and U. Paquet and J. Tomczak and C. Zhang},
 pages = {95266--95290},
 publisher = {Curran Associates, Inc.},
 title = {MMLU-Pro: A More Robust and Challenging Multi-Task Language Understanding Benchmark},
 volume = {37},
 year = {2024}
}

@inproceedings{bench:bbh,
    title = "Challenging {BIG}-Bench Tasks and Whether Chain-of-Thought Can Solve Them",
    author = {Suzgun, Mirac  and
      Scales, Nathan  and
      Sch{\"a}rli, Nathanael  and
      Gehrmann, Sebastian  and
      Tay, Yi  and
      Chung, Hyung Won  and
      Chowdhery, Aakanksha  and
      Le, Quoc  and
      Chi, Ed  and
      Zhou, Denny  and
      Wei, Jason},
    editor = "Rogers, Anna  and
      Boyd-Graber, Jordan  and
      Okazaki, Naoaki",
    booktitle = "Findings of the Association for Computational Linguistics: ACL 2023",
    month = jul,
    year = "2023",
    address = "Toronto, Canada",
    publisher = "Association for Computational Linguistics",
    url = "https://aclanthology.org/2023.findings-acl.824/",
    doi = "10.18653/v1/2023.findings-acl.824",
    pages = "13003--13051"
}

@misc{bench:arc_challenge,
      title={Think you have Solved Question Answering? Try ARC, the AI2 Reasoning Challenge}, 
      author={Peter Clark and Isaac Cowhey and Oren Etzioni and Tushar Khot and Ashish Sabharwal and Carissa Schoenick and Oyvind Tafjord},
      year={2018},
      eprint={1803.05457},
      archivePrefix={arXiv},
      primaryClass={cs.AI},
      url={https://arxiv.org/abs/1803.05457}, 
}

@misc{bench:hellaswag,
      title={HellaSwag: Can a Machine Really Finish Your Sentence?}, 
      author={Rowan Zellers and Ari Holtzman and Yonatan Bisk and Ali Farhadi and Yejin Choi},
      year={2019},
      eprint={1905.07830},
      archivePrefix={arXiv},
      primaryClass={cs.CL},
      url={https://arxiv.org/abs/1905.07830}, 
}

@article{bench:winogrande,
author = {Sakaguchi, Keisuke and Bras, Ronan Le and Bhagavatula, Chandra and Choi, Yejin},
title = {WinoGrande: an adversarial winograd schema challenge at scale},
year = {2021},
issue_date = {September 2021},
publisher = {Association for Computing Machinery},
address = {New York, NY, USA},
volume = {64},
number = {9},
issn = {0001-0782},
url = {https://doi.org/10.1145/3474381},
doi = {10.1145/3474381},
journal = {Commun. ACM},
month = aug,
pages = {99–106},
numpages = {8}
}

@misc{bench:gsm8k,
      title={Training Verifiers to Solve Math Word Problems}, 
      author={Karl Cobbe and Vineet Kosaraju and Mohammad Bavarian and Mark Chen and Heewoo Jun and Lukasz Kaiser and Matthias Plappert and Jerry Tworek and Jacob Hilton and Reiichiro Nakano and Christopher Hesse and John Schulman},
      year={2021},
      eprint={2110.14168},
      archivePrefix={arXiv},
      primaryClass={cs.LG},
      url={https://arxiv.org/abs/2110.14168}, 
}

@misc{hendrycks2021measuringmathematicalproblemsolving,
      title={Measuring Mathematical Problem Solving With the MATH Dataset}, 
      author={Dan Hendrycks and Collin Burns and Saurav Kadavath and Akul Arora and Steven Basart and Eric Tang and Dawn Song and Jacob Steinhardt},
      year={2021},
      eprint={2103.03874},
      archivePrefix={arXiv},
      primaryClass={cs.LG},
      url={https://arxiv.org/abs/2103.03874}, 
}

@misc{bench:humaneval,
      title={Evaluating Large Language Models Trained on Code},  
      author={Mark Chen and Jerry Tworek and Heewoo Jun and Qiming Yuan and Henrique Ponde de Oliveira Pinto and Jared Kaplan and Harri Edwards and Yuri Burda and Nicholas Joseph and Greg Brockman and Alex Ray and Raul Puri and Gretchen Krueger and Michael Petrov and Heidy Khlaaf and Girish Sastry and Pamela Mishkin and Brooke Chan and Scott Gray and Nick Ryder and Mikhail Pavlov and Alethea Power and Lukasz Kaiser and Mohammad Bavarian and Clemens Winter and Philippe Tillet and Felipe Petroski Such and Dave Cummings and Matthias Plappert and Fotios Chantzis and Elizabeth Barnes and Ariel Herbert-Voss and William Hebgen Guss and Alex Nichol and Alex Paino and Nikolas Tezak and Jie Tang and Igor Babuschkin and Suchir Balaji and Shantanu Jain and William Saunders and Christopher Hesse and Andrew N. Carr and Jan Leike and Josh Achiam and Vedant Misra and Evan Morikawa and Alec Radford and Matthew Knight and Miles Brundage and Mira Murati and Katie Mayer and Peter Welinder and Bob McGrew and Dario Amodei and Sam McCandlish and Ilya Sutskever and Wojciech Zaremba},
      year={2021},
      eprint={2107.03374},
      archivePrefix={arXiv},
      primaryClass={cs.LG},
      url={https://arxiv.org/abs/2107.03374}, 
}

@misc{bench:mbpp,
      title={Program Synthesis with Large Language Models}, 
      author={Jacob Austin and Augustus Odena and Maxwell Nye and Maarten Bosma and Henryk Michalewski and David Dohan and Ellen Jiang and Carrie Cai and Michael Terry and Quoc Le and Charles Sutton},
      year={2021},
      eprint={2108.07732},
      archivePrefix={arXiv},
      primaryClass={cs.PL},
      url={https://arxiv.org/abs/2108.07732}, 
}

@misc{gao2024omnimathuniversalolympiadlevel,
      title={Omni-MATH: A Universal Olympiad Level Mathematic Benchmark For Large Language Models}, 
      author={Bofei Gao and Feifan Song and Zhe Yang and Zefan Cai and Yibo Miao and Qingxiu Dong and Lei Li and Chenghao Ma and Liang Chen and Runxin Xu and Zhengyang Tang and Benyou Wang and Daoguang Zan and Shanghaoran Quan and Ge Zhang and Lei Sha and Yichang Zhang and Xuancheng Ren and Tianyu Liu and Baobao Chang},
      year={2024},
      eprint={2410.07985},
      archivePrefix={arXiv},
      primaryClass={cs.CL},
      url={https://arxiv.org/abs/2410.07985}, 
}

@article{zhu2025scaling,
  title={Scaling latent reasoning via looped language models},
  author={Zhu, Rui-Jie and Wang, Zixuan and Hua, Kai and Zhang, Tianyu and Li, Ziniu and Que, Haoran and Wei, Boyi and Wen, Zixin and Yin, Fan and Xing, He and others},
  journal={arXiv preprint arXiv:2510.25741},
  year={2025}
}

@misc{wang20258020rulehighentropyminority,
      title={Beyond the 80/20 Rule: High-Entropy Minority Tokens Drive Effective Reinforcement Learning for LLM Reasoning}, 
      author={Shenzhi Wang and Le Yu and Chang Gao and Chujie Zheng and Shixuan Liu and Rui Lu and Kai Dang and Xionghui Chen and Jianxin Yang and Zhenru Zhang and Yuqiong Liu and An Yang and Andrew Zhao and Yang Yue and Shiji Song and Bowen Yu and Gao Huang and Junyang Lin},
      year={2025},
      eprint={2506.01939},
      archivePrefix={arXiv},
      primaryClass={cs.CL},
      url={https://arxiv.org/abs/2506.01939}, 
}

@article{model:llama3,
  title={The llama 3 herd of models},
  author={Grattafiori, Aaron and Dubey, Abhimanyu and Jauhri, Abhinav and Pandey, Abhinav and Kadian, Abhishek and Al-Dahle, Ahmad and Letman, Aiesha and Mathur, Akhil and Schelten, Alan and Vaughan, Alex and others},
  journal={arXiv preprint arXiv:2407.21783},
  year={2024}
}

@article{model:qwen3,
  title={Qwen3 technical report},
  author={Yang, An and Li, Anfeng and Yang, Baosong and Zhang, Beichen and Hui, Binyuan and Zheng, Bo and Yu, Bowen and Gao, Chang and Huang, Chengen and Lv, Chenxu and others},
  journal={arXiv preprint arXiv:2505.09388},
  year={2025}
}

@article{model:qwen2,
  title={Qwen2 technical report},
  author={Team, Qwen and others},
  journal={arXiv preprint arXiv:2407.10671},
  volume={2},
  number={3},
  year={2024}
}

@article{model:gemma3,
  title={Gemma 3 technical report},
  author={Team, Gemma and Kamath, Aishwarya and Ferret, Johan and Pathak, Shreya and Vieillard, Nino and Merhej, Ramona and Perrin, Sarah and Matejovicova, Tatiana and Ram{\'e}, Alexandre and Rivi{\`e}re, Morgane and others},
  journal={arXiv preprint arXiv:2503.19786},
  year={2025}
}

@misc{bench:eval-harness,
  author       = {Gao, Leo and Tow, Jonathan and Abbasi, Baber and Biderman, Stella and Black, Sid and DiPofi, Anthony and Foster, Charles and Golding, Laurence and Hsu, Jeffrey and Le Noac'h, Alain and Li, Haonan and McDonell, Kyle and Muennighoff, Niklas and Ociepa, Chris and Phang, Jason and Reynolds, Laria and Schoelkopf, Hailey and Skowron, Aviya and Sutawika, Lintang and Tang, Eric and Thite, Anish and Wang, Ben and Wang, Kevin and Zou, Andy},
  title        = {The Language Model Evaluation Harness},
  month        = 07,
  year         = 2024,
  publisher    = {Zenodo},
  version      = {v0.4.3},
  doi          = {10.5281/zenodo.12608602},
  url          = {https://zenodo.org/records/12608602}
}

@inproceedings{bench:evalplus,
  title = {Is Your Code Generated by Chat{GPT} Really Correct? Rigorous Evaluation of Large Language Models for Code Generation},
  author = {Liu, Jiawei and Xia, Chunqiu Steven and Wang, Yuyao and Zhang, Lingming},
  booktitle = {Thirty-seventh Conference on Neural Information Processing Systems},
  year = {2023},
  url = {https://openreview.net/forum?id=1qvx610Cu7},
}

@Misc{train:accelerate,
  title =        {Accelerate: Training and inference at scale made simple, efficient and adaptable.},
  author =       {Sylvain Gugger and Lysandre Debut and Thomas Wolf and Philipp Schmid and Zachary Mueller and Sourab Mangrulkar and Marc Sun and Benjamin Bossan},
  howpublished = {\url{https://github.com/huggingface/accelerate}},
  year =         {2022}
}

@inproceedings{train:transformers,
    title = "Transformers: State-of-the-Art Natural Language Processing",
    author = "Thomas Wolf and Lysandre Debut and Victor Sanh and Julien Chaumond and Clement Delangue and Anthony Moi and Pierric Cistac and Tim Rault and Rémi Louf and Morgan Funtowicz and Joe Davison and Sam Shleifer and Patrick von Platen and Clara Ma and Yacine Jernite and Julien Plu and Canwen Xu and Teven Le Scao and Sylvain Gugger and Mariama Drame and Quentin Lhoest and Alexander M. Rush",
    booktitle = "Proceedings of the 2020 Conference on Empirical Methods in Natural Language Processing: System Demonstrations",
    month = oct,
    year = "2020",
    address = "Online",
    publisher = "Association for Computational Linguistics",
    url = "https://www.aclweb.org/anthology/2020.emnlp-demos.6",
    pages = "38--45"
}

@article{method:rpt,
  author       = {Qingxiu Dong and
                  Li Dong and
                  Yao Tang and
                  Tianzhu Ye and
                  Yutao Sun and
                  Zhifang Sui and
                  Furu Wei},
  title        = {Reinforcement Pre-Training},
  journal      = {CoRR},
  volume       = {abs/2506.08007},
  year         = {2025},
  url          = {https://doi.org/10.48550/arXiv.2506.08007},
  doi          = {10.48550/ARXIV.2506.08007},
  eprinttype    = {arXiv},
  eprint       = {2506.08007},
  bibsource    = {dblp computer science bibliography, https://dblp.org}
}

@InProceedings{pmlr-v37-schulman15,
  title = 	 {Trust Region Policy Optimization},
  author = 	 {Schulman, John and Levine, Sergey and Abbeel, Pieter and Jordan, Michael and Moritz, Philipp},
  booktitle = 	 {Proceedings of the 32nd International Conference on Machine Learning},
  pages = 	 {1889--1897},
  year = 	 {2015},
  editor = 	 {Bach, Francis and Blei, David},
  volume = 	 {37},
  series = 	 {Proceedings of Machine Learning Research},
  address = 	 {Lille, France},
  month = 	 {07--09 Jul},
  publisher =    {PMLR},
  url = 	 {https://proceedings.mlr.press/v37/schulman15.html}
}

@misc{schulman2017proximalpolicyoptimizationalgorithms,
      title={Proximal Policy Optimization Algorithms}, 
      author={John Schulman and Filip Wolski and Prafulla Dhariwal and Alec Radford and Oleg Klimov},
      year={2017},
      eprint={1707.06347},
      archivePrefix={arXiv},
      primaryClass={cs.LG},
      url={https://arxiv.org/abs/1707.06347}, 
}

@inproceedings{NIPS2017_d5e2c0ad,
 author = {Christiano, Paul F and Leike, Jan and Brown, Tom and Martic, Miljan and Legg, Shane and Amodei, Dario},
 booktitle = {Advances in Neural Information Processing Systems},
 editor = {I. Guyon and U. Von Luxburg and S. Bengio and H. Wallach and R. Fergus and S. Vishwanathan and R. Garnett},
 pages = {},
 publisher = {Curran Associates, Inc.},
 title = {Deep Reinforcement Learning from Human Preferences},
 volume = {30},
 year = {2017}
}

@inproceedings{NEURIPS2020_1f89885d,
 author = {Stiennon, Nisan and Ouyang, Long and Wu, Jeffrey and Ziegler, Daniel and Lowe, Ryan and Voss, Chelsea and Radford, Alec and Amodei, Dario and Christiano, Paul F},
 booktitle = {Advances in Neural Information Processing Systems},
 editor = {H. Larochelle and M. Ranzato and R. Hadsell and M.F. Balcan and H. Lin},
 pages = {3008--3021},
 publisher = {Curran Associates, Inc.},
 title = {Learning to summarize with human feedback},
 volume = {33},
 year = {2020}
}

@inproceedings{NEURIPS2022_b1efde53,
 author = {Ouyang, Long and Wu, Jeffrey and Jiang, Xu and Almeida, Diogo and Wainwright, Carroll and Mishkin, Pamela and Zhang, Chong and Agarwal, Sandhini and Slama, Katarina and Ray, Alex and Schulman, John and Hilton, Jacob and Kelton, Fraser and Miller, Luke and Simens, Maddie and Askell, Amanda and Welinder, Peter and Christiano, Paul F and Leike, Jan and Lowe, Ryan},
 booktitle = {Advances in Neural Information Processing Systems},
 editor = {S. Koyejo and S. Mohamed and A. Agarwal and D. Belgrave and K. Cho and A. Oh},
 pages = {27730--27744},
 publisher = {Curran Associates, Inc.},
 title = {Training language models to follow instructions with human feedback},
 volume = {35},
 year = {2022}
}

@misc{bai2022constitutionalaiharmlessnessai,
      title={Constitutional AI: Harmlessness from AI Feedback}, 
      author={Yuntao Bai and Saurav Kadavath and Sandipan Kundu and Amanda Askell and Jackson Kernion and Andy Jones and Anna Chen and Anna Goldie and Azalia Mirhoseini and Cameron McKinnon and Carol Chen and Catherine Olsson and Christopher Olah and Danny Hernandez and Dawn Drain and Deep Ganguli and Dustin Li and Eli Tran-Johnson and Ethan Perez and Jamie Kerr and Jared Mueller and Jeffrey Ladish and Joshua Landau and Kamal Ndousse and Kamile Lukosuite and Liane Lovitt and Michael Sellitto and Nelson Elhage and Nicholas Schiefer and Noemi Mercado and Nova DasSarma and Robert Lasenby and Robin Larson and Sam Ringer and Scott Johnston and Shauna Kravec and Sheer El Showk and Stanislav Fort and Tamera Lanham and Timothy Telleen-Lawton and Tom Conerly and Tom Henighan and Tristan Hume and Samuel R. Bowman and Zac Hatfield-Dodds and Ben Mann and Dario Amodei and Nicholas Joseph and Sam McCandlish and Tom Brown and Jared Kaplan},
      year={2022},
      eprint={2212.08073},
      archivePrefix={arXiv},
      primaryClass={cs.CL},
      url={https://arxiv.org/abs/2212.08073}, 
}

@inproceedings{NEURIPS2023_a85b405e,
 author = {Rafailov, Rafael and Sharma, Archit and Mitchell, Eric and Manning, Christopher D and Ermon, Stefano and Finn, Chelsea},
 booktitle = {Advances in Neural Information Processing Systems},
 editor = {A. Oh and T. Naumann and A. Globerson and K. Saenko and M. Hardt and S. Levine},
 pages = {53728--53741},
 publisher = {Curran Associates, Inc.},
 title = {Direct Preference Optimization: Your Language Model is Secretly a Reward Model},
 volume = {36},
 year = {2023}
}

@misc{ethayarajh2024ktomodelalignmentprospect,
      title={KTO: Model Alignment as Prospect Theoretic Optimization}, 
      author={Kawin Ethayarajh and Winnie Xu and Niklas Muennighoff and Dan Jurafsky and Douwe Kiela},
      year={2024},
      eprint={2402.01306},
      archivePrefix={arXiv},
      primaryClass={cs.LG},
      url={https://arxiv.org/abs/2402.01306}, 
}

@misc{hong2024orpomonolithicpreferenceoptimization,
      title={ORPO: Monolithic Preference Optimization without Reference Model}, 
      author={Jiwoo Hong and Noah Lee and James Thorne},
      year={2024},
      eprint={2403.07691},
      archivePrefix={arXiv},
      primaryClass={cs.CL},
      url={https://arxiv.org/abs/2403.07691}, 
}

@misc{xiao2025comprehensivesurveydirectpreference,
      title={A Comprehensive Survey of Direct Preference Optimization: Datasets, Theories, Variants, and Applications}, 
      author={Wenyi Xiao and Zechuan Wang and Leilei Gan and Shuai Zhao and Zongrui Li and Ruirui Lei and Wanggui He and Luu Anh Tuan and Long Chen and Hao Jiang and Zhou Zhao and Fei Wu},
      year={2025},
      eprint={2410.15595},
      archivePrefix={arXiv},
      primaryClass={cs.AI},
      url={https://arxiv.org/abs/2410.15595}, 
}

@article{jiang2025preferencesurvey,
author = {Jiang, Ruili and Chen, Kehai and Bai, Xuefeng and He, Zhixuan and Li, Juntao and Yang, Muyun and Zhao, Tiejun and Nie, Liqiang and Zhang, Min},
title = {A Survey on Human Preference Learning for Aligning Large Language Models},
year = {2025},
issue_date = {April 2026},
publisher = {Association for Computing Machinery},
address = {New York, NY, USA},
volume = {58},
number = {6},
issn = {0360-0300},
url = {https://doi.org/10.1145/3773279},
doi = {10.1145/3773279},
journal = {ACM Comput. Surv.},
month = dec,
articleno = {152},
numpages = {39}
}

@misc{cen2025valueincentivizedpreferenceoptimizationunified,
      title={Value-Incentivized Preference Optimization: A Unified Approach to Online and Offline RLHF}, 
      author={Shicong Cen and Jincheng Mei and Katayoon Goshvadi and Hanjun Dai and Tong Yang and Sherry Yang and Dale Schuurmans and Yuejie Chi and Bo Dai},
      year={2025},
      eprint={2405.19320},
      archivePrefix={arXiv},
      primaryClass={cs.LG},
      url={https://arxiv.org/abs/2405.19320}, 
}

@misc{wen2025reinforcementlearningverifiablerewards,
      title={Reinforcement Learning with Verifiable Rewards Implicitly Incentivizes Correct Reasoning in Base LLMs}, 
      author={Xumeng Wen and Zihan Liu and Shun Zheng and Shengyu Ye and Zhirong Wu and Yang Wang and Zhijian Xu and Xiao Liang and Junjie Li and Ziming Miao and Jiang Bian and Mao Yang},
      year={2025},
      eprint={2506.14245},
      archivePrefix={arXiv},
      primaryClass={cs.AI},
      url={https://arxiv.org/abs/2506.14245}, 
}

@article{Guo_2025,
   title={DeepSeek-R1 incentivizes reasoning in LLMs through reinforcement learning},
   volume={645},
   ISSN={1476-4687},
   url={http://dx.doi.org/10.1038/s41586-025-09422-z},
   DOI={10.1038/s41586-025-09422-z},
   number={8081},
   journal={Nature},
   publisher={Springer Science and Business Media LLC},
   author={Guo, Daya and Yang, Dejian and Zhang, Haowei and Song, Junxiao and Wang, Peiyi and Zhu, Qihao and Xu, Runxin and Zhang, Ruoyu and Ma, Shirong and Bi, Xiao and Zhang, Xiaokang and Yu, Xingkai and Wu, Yu and Wu, Z. F. and Gou, Zhibin and Shao, Zhihong and Li, Zhuoshu and Gao, Ziyi and Liu, Aixin and Xue, Bing and Wang, Bingxuan and Wu, Bochao and Feng, Bei and Lu, Chengda and Zhao, Chenggang and Deng, Chengqi and Ruan, Chong and Dai, Damai and Chen, Deli and Ji, Dongjie and Li, Erhang and Lin, Fangyun and Dai, Fucong and Luo, Fuli and Hao, Guangbo and Chen, Guanting and Li, Guowei and Zhang, H. and Xu, Hanwei and Ding, Honghui and Gao, Huazuo and Qu, Hui and Li, Hui and Guo, Jianzhong and Li, Jiashi and Chen, Jingchang and Yuan, Jingyang and Tu, Jinhao and Qiu, Junjie and Li, Junlong and Cai, J. L. and Ni, Jiaqi and Liang, Jian and Chen, Jin and Dong, Kai and Hu, Kai and You, Kaichao and Gao, Kaige and Guan, Kang and Huang, Kexin and Yu, Kuai and Wang, Lean and Zhang, Lecong and Zhao, Liang and Wang, Litong and Zhang, Liyue and Xu, Lei and Xia, Leyi and Zhang, Mingchuan and Zhang, Minghua and Tang, Minghui and Zhou, Mingxu and Li, Meng and Wang, Miaojun and Li, Mingming and Tian, Ning and Huang, Panpan and Zhang, Peng and Wang, Qiancheng and Chen, Qinyu and Du, Qiushi and Ge, Ruiqi and Zhang, Ruisong and Pan, Ruizhe and Wang, Runji and Chen, R. J. and Jin, R. L. and Chen, Ruyi and Lu, Shanghao and Zhou, Shangyan and Chen, Shanhuang and Ye, Shengfeng and Wang, Shiyu and Yu, Shuiping and Zhou, Shunfeng and Pan, Shuting and Li, S. S. and Zhou, Shuang and Wu, Shaoqing and Yun, Tao and Pei, Tian and Sun, Tianyu and Wang, T. and Zeng, Wangding and Liu, Wen and Liang, Wenfeng and Gao, Wenjun and Yu, Wenqin and Zhang, Wentao and Xiao, W. L. and An, Wei and Liu, Xiaodong and Wang, Xiaohan and Chen, Xiaokang and Nie, Xiaotao and Cheng, Xin and Liu, Xin and Xie, Xin and Liu, Xingchao and Yang, Xinyu and Li, Xinyuan and Su, Xuecheng and Lin, Xuheng and Li, X. Q. and Jin, Xiangyue and Shen, Xiaojin and Chen, Xiaosha and Sun, Xiaowen and Wang, Xiaoxiang and Song, Xinnan and Zhou, Xinyi and Wang, Xianzu and Shan, Xinxia and Li, Y. K. and Wang, Y. Q. and Wei, Y. X. and Zhang, Yang and Xu, Yanhong and Li, Yao and Zhao, Yao and Sun, Yaofeng and Wang, Yaohui and Yu, Yi and Zhang, Yichao and Shi, Yifan and Xiong, Yiliang and He, Ying and Piao, Yishi and Wang, Yisong and Tan, Yixuan and Ma, Yiyang and Liu, Yiyuan and Guo, Yongqiang and Ou, Yuan and Wang, Yuduan and Gong, Yue and Zou, Yuheng and He, Yujia and Xiong, Yunfan and Luo, Yuxiang and You, Yuxiang and Liu, Yuxuan and Zhou, Yuyang and Zhu, Y. X. and Huang, Yanping and Li, Yaohui and Zheng, Yi and Zhu, Yuchen and Ma, Yunxian and Tang, Ying and Zha, Yukun and Yan, Yuting and Ren, Z. Z. and Ren, Zehui and Sha, Zhangli and Fu, Zhe and Xu, Zhean and Xie, Zhenda and Zhang, Zhengyan and Hao, Zhewen and Ma, Zhicheng and Yan, Zhigang and Wu, Zhiyu and Gu, Zihui and Zhu, Zijia and Liu, Zijun and Li, Zilin and Xie, Ziwei and Song, Ziyang and Pan, Zizheng and Huang, Zhen and Xu, Zhipeng and Zhang, Zhongyu and Zhang, Zhen},
   year={2025},
   month=sep, pages={633–638} }

@misc{shao2024deepseekmathpushinglimitsmathematical,
      title={DeepSeekMath: Pushing the Limits of Mathematical Reasoning in Open Language Models}, 
      author={Zhihong Shao and Peiyi Wang and Qihao Zhu and Runxin Xu and Junxiao Song and Xiao Bi and Haowei Zhang and Mingchuan Zhang and Y. K. Li and Y. Wu and Daya Guo},
      year={2024},
      eprint={2402.03300},
      archivePrefix={arXiv},
      primaryClass={cs.CL},
      url={https://arxiv.org/abs/2402.03300}, 
}

@misc{hatamizadeh2025rlpreinforcementpretrainingobjective,
      title={RLP: Reinforcement as a Pretraining Objective}, 
      author={Ali Hatamizadeh and Syeda Nahida Akter and Shrimai Prabhumoye and Jan Kautz and Mostofa Patwary and Mohammad Shoeybi and Bryan Catanzaro and Yejin Choi},
      year={2025},
      eprint={2510.01265},
      archivePrefix={arXiv},
      primaryClass={cs.LG},
      url={https://arxiv.org/abs/2510.01265}, 
}

@article{graves2016act,
  title={Adaptive computation time for recurrent neural networks},
  author={Graves, Alex},
  journal={arXiv preprint arXiv:1603.08983},
  year={2016}
}

@misc{dehghani2019universaltransformers,
      title={Universal Transformers}, 
      author={Mostafa Dehghani and Stephan Gouws and Oriol Vinyals and Jakob Uszkoreit and Łukasz Kaiser},
      year={2019},
      eprint={1807.03819},
      archivePrefix={arXiv},
      primaryClass={cs.CL},
      url={https://arxiv.org/abs/1807.03819}, 
}

@misc{tan2023sparseuniversaltransformer,
      title={Sparse Universal Transformer}, 
      author={Shawn Tan and Yikang Shen and Zhenfang Chen and Aaron Courville and Chuang Gan},
      year={2023},
      eprint={2310.07096},
      archivePrefix={arXiv},
      primaryClass={cs.CL},
      url={https://arxiv.org/abs/2310.07096}, 
}

@misc{banino2021pondernetlearningponder,
      title={PonderNet: Learning to Ponder}, 
      author={Andrea Banino and Jan Balaguer and Charles Blundell},
      year={2021},
      eprint={2107.05407},
      archivePrefix={arXiv},
      primaryClass={cs.LG},
      url={https://arxiv.org/abs/2107.05407}, 
}

@inproceedings{dai-etal-2019-transformer,
    title = "Transformer-{XL}: Attentive Language Models beyond a Fixed-Length Context",
    author = "Dai, Zihang  and
      Yang, Zhilin  and
      Yang, Yiming  and
      Carbonell, Jaime  and
      Le, Quoc  and
      Salakhutdinov, Ruslan",
    editor = "Korhonen, Anna  and
      Traum, David  and
      M{\`a}rquez, Llu{\'i}s",
    booktitle = "Proceedings of the 57th Annual Meeting of the Association for Computational Linguistics",
    month = jul,
    year = "2019",
    address = "Florence, Italy",
    publisher = "Association for Computational Linguistics",
    url = "https://aclanthology.org/P19-1285/",
    doi = "10.18653/v1/P19-1285",
    pages = "2978--2988"
}

@misc{rae2019compressivetransformerslongrangesequence,
      title={Compressive Transformers for Long-Range Sequence Modelling}, 
      author={Jack W. Rae and Anna Potapenko and Siddhant M. Jayakumar and Timothy P. Lillicrap},
      year={2019},
      eprint={1911.05507},
      archivePrefix={arXiv},
      primaryClass={cs.LG},
      url={https://arxiv.org/abs/1911.05507}, 
}

@inproceedings{NEURIPS2022_47e28862,
 author = {Bulatov, Aydar and Kuratov, Yury and Burtsev, Mikhail},
 booktitle = {Advances in Neural Information Processing Systems},
 editor = {S. Koyejo and S. Mohamed and A. Agarwal and D. Belgrave and K. Cho and A. Oh},
 pages = {11079--11091},
 publisher = {Curran Associates, Inc.},
 title = {Recurrent Memory Transformer},
 volume = {35},
 year = {2022}
}

@misc{xin2020deebertdynamicearlyexiting,
      title={DeeBERT: Dynamic Early Exiting for Accelerating BERT Inference}, 
      author={Ji Xin and Raphael Tang and Jaejun Lee and Yaoliang Yu and Jimmy Lin},
      year={2020},
      eprint={2004.12993},
      archivePrefix={arXiv},
      primaryClass={cs.CL},
      url={https://arxiv.org/abs/2004.12993}, 
}

@inproceedings{NEURIPS2020_d4dd111a,
 author = {Zhou, Wangchunshu and Xu, Canwen and Ge, Tao and McAuley, Julian and Xu, Ke and Wei, Furu},
 booktitle = {Advances in Neural Information Processing Systems},
 editor = {H. Larochelle and M. Ranzato and R. Hadsell and M.F. Balcan and H. Lin},
 pages = {18330--18341},
 publisher = {Curran Associates, Inc.},
 title = {BERT Loses Patience: Fast and Robust Inference with Early Exit},
 volume = {33},
 year = {2020}
}

@inproceedings{xu-mcauley-2023-survey,
    title = "A Survey on Dynamic Neural Networks for Natural Language Processing",
    author = "Xu, Canwen  and
      McAuley, Julian",
    editor = "Vlachos, Andreas  and
      Augenstein, Isabelle",
    booktitle = "Findings of the Association for Computational Linguistics: EACL 2023",
    month = may,
    year = "2023",
    address = "Dubrovnik, Croatia",
    publisher = "Association for Computational Linguistics",
    url = "https://aclanthology.org/2023.findings-eacl.180/",
    doi = "10.18653/v1/2023.findings-eacl.180",
    pages = "2370--2381"
}

@misc{bajpai2025surveyearlyexitdeep,
      title={A Survey of Early Exit Deep Neural Networks in NLP}, 
      author={Divya Jyoti Bajpai and Manjesh Kumar Hanawal},
      year={2025},
      eprint={2501.07670},
      archivePrefix={arXiv},
      primaryClass={cs.LG},
      url={https://arxiv.org/abs/2501.07670}, 
}

@misc{raposo2024mixtureofdepthsdynamicallyallocatingcompute,
      title={Mixture-of-Depths: Dynamically allocating compute in transformer-based language models}, 
      author={David Raposo and Sam Ritter and Blake Richards and Timothy Lillicrap and Peter Conway Humphreys and Adam Santoro},
      year={2024},
      eprint={2404.02258},
      archivePrefix={arXiv},
      primaryClass={cs.LG},
      url={https://arxiv.org/abs/2404.02258}, 
}

@misc{chen2025innerthinkingtransformerleveraging,
      title={Inner Thinking Transformer: Leveraging Dynamic Depth Scaling to Foster Adaptive Internal Thinking}, 
      author={Yilong Chen and Junyuan Shang and Zhenyu Zhang and Yanxi Xie and Jiawei Sheng and Tingwen Liu and Shuohuan Wang and Yu Sun and Hua Wu and Haifeng Wang},
      year={2025},
      eprint={2502.13842},
      archivePrefix={arXiv},
      primaryClass={cs.CL},
      url={https://arxiv.org/abs/2502.13842}, 
}

@misc{wu2025parallellooptransformerefficient,
      title={Parallel Loop Transformer for Efficient Test-Time Computation Scaling}, 
      author={Bohong Wu and Mengzhao Chen and Xiang Luo and Shen Yan and Qifan Yu and Fan Xia and Tianqi Zhang and Hongrui Zhan and Zheng Zhong and Xun Zhou and Siyuan Qiao and Xingyan Bin},
      year={2025},
      eprint={2510.24824},
      archivePrefix={arXiv},
      primaryClass={cs.CL},
      url={https://arxiv.org/abs/2510.24824}, 
}

@misc{li2025skiplayerloopit,
      title={Skip a Layer or Loop it? Test-Time Depth Adaptation of Pretrained LLMs}, 
      author={Ziyue Li and Yang Li and Tianyi Zhou},
      year={2025},
      eprint={2507.07996},
      archivePrefix={arXiv},
      primaryClass={cs.LG},
      url={https://arxiv.org/abs/2507.07996}, 
}

@misc{mcleish2025teachingpretrainedlanguagemodels,
      title={Teaching Pretrained Language Models to Think Deeper with Retrofitted Recurrence}, 
      author={Sean McLeish and Ang Li and John Kirchenbauer and Dayal Singh Kalra and Brian R. Bartoldson and Bhavya Kailkhura and Avi Schwarzschild and Jonas Geiping and Tom Goldstein and Micah Goldblum},
      year={2025},
      eprint={2511.07384},
      archivePrefix={arXiv},
      primaryClass={cs.CL},
      url={https://arxiv.org/abs/2511.07384}, 
}

@misc{zeng2025ponderlm2pretrainingllmlatent,
      title={PonderLM-2: Pretraining LLM with Latent Thoughts in Continuous Space}, 
      author={Boyi Zeng and He Li and Shixiang Song and Yixuan Wang and Ziwei He and Xinbing Wang and Zhouhan Lin},
      year={2025},
      eprint={2509.23184},
      archivePrefix={arXiv},
      primaryClass={cs.CL},
      url={https://arxiv.org/abs/2509.23184}, 
}

@article{wei2022chain,
  title={Chain-of-thought prompting elicits reasoning in large language models},
  author={Wei, Jason and Wang, Xuezhi and Schuurmans, Dale and Bosma, Maarten and Xia, Fei and Chi, Ed and Le, Quoc V and Zhou, Denny and others},
  journal={Advances in neural information processing systems},
  volume={35},
  pages={24824--24837},
  year={2022}
}

@article{tarvainen2017mean,
  title={Mean teachers are better role models: Weight-averaged consistency targets improve semi-supervised deep learning results},
  author={Tarvainen, Antti and Valpola, Harri},
  journal={Advances in neural information processing systems},
  volume={30},
  year={2017}
}

@article{williams2026prioritize,
  title={Prioritize the Process, Not Just the Outcome: Rewarding Latent Thought Trajectories Improves Reasoning in Looped Language Models},
  author={Williams, Jonathan and Tureci, Esin},
  journal={arXiv preprint arXiv:2602.10520},
  year={2026}
}

%%%%%%%%%%%%%%%%%%%%%%%%%%%%%%%%%%%%%%%%%%%%%%%%%%%%%%%%%%%%

\appendix

\clearpage
\onecolumn 
\appendix

\section{Limitations}
\label{app:limit}
While LoopRPT is effective for reinforcement pre-training of looped language models, this study has a few limitations that pave the way for future research: (1) Model Scale and Data Diversity: Due to computational constraints, our current instantiations are evaluated at the 1.4B and 2.6B parameter scales, with training primarily focused on the mathematical reasoning domain (\textsc{Omni-Math}). While our method shows strong transferability to coding and general logic, exploring the scaling laws of LoopRPT on much larger models (e.g., 7B+ parameters) and across broader, open-domain pre-training mixtures remains an important next step. (2) Fixed Maximum Loop Depth During Training: The current LoopRPT framework computes step-wise rewards and policy advantages over a predefined maximum computational horizon (e.g., $K=4$ loops during training). Although the model learns an adaptive early-exit policy for inference, the fixed unrolling depth during training may cap the upper bound of reasoning complexity it can explore. Formulating a reinforcement objective that supports infinite-horizon or fully dynamic training depth is an exciting avenue for future work.

\section{Additional Theory}
\label{app:additional_theory}

\paragraph{Setup and notation.}
For a fixed next-token position $t$ in a sequence, LoopRPT defines:
(i) a teacher (EMA) model $\bar{\theta}$ and a student model $\theta$;
(ii) a teacher exit distribution $\pi_{\bar{\theta}}(k)$ and student exit distribution $\pi_{\theta}(k)$ over latent steps $k\in\{1,\dots,K\}$ (Eq.~(\ref{eq:exit_prob}--\ref{eq:s_recursion}));
(iii) a teacher reference step
$t_{\mathrm{ref}}=\min\{k:\sum_{j\le k}\pi_{\bar{\theta}}(j)\ge \tau\}$;
(iv) a per-token dynamic baseline $b_{\mathrm{ref}}=\ell_{\bar{\theta}}^{(t_{\mathrm{ref}})}$ (Eq.~(\ref{eq:ref_logprob}));
(v) a step-wise reward
$R(k)=\Delta_{\mathrm{acc}}(k)-C(k)$ (Eq.~(\ref{eq:step_reward})) where
$\Delta_{\mathrm{acc}}(k)=\ell_{\theta}^{(k)}-b_{\mathrm{ref}}$ (Eq.~(\ref{eq:acc_delta})) and
$C(k)=\lambda_t(k-t_{\mathrm{ref}})$ (Eq.~(\ref{eq:time_penalty})),
with $\lambda_t$ derived from teacher uncertainty (Eq.~(\ref{eq:lambda_t})).
We use $\Pi_{\theta}(k)\coloneqq \sum_{j\le k}\pi_{\theta}(j)$ to denote the CDF.

% =========================================================
\subsection{Why High-Entropy Token Selection Focuses Training Signal}
\label{app:theory_entropy_select}

\paragraph{Teacher entropy and selection.}
LoopRPT computes teacher entropy (Eq.~(\ref{eq:ref_entropy})):
\begin{equation}
\label{eq:teacher_entropy_app}
H_t \coloneqq -\sum_{v\in V}\bar{p}_{\bar{\theta}}(v\mid x_{<t})\log \bar{p}_{\bar{\theta}}(v\mid x_{<t}),
\end{equation}
and selects the top-$\rho$ fraction of positions within each example as hard tokens (Sec.~\ref{sec:method:selection}).

\paragraph{A principled proxy: expected gradient energy.}
We justify entropy selection by relating teacher uncertainty to the \emph{expected squared norm} of the
cross-entropy gradient, a standard proxy for how informative a sample is for learning.

\begin{assumption}[Teacher distribution as a proxy for label uncertainty]
\label{ass:teacher_as_proxy_app}
Conditioned on context $x_{<t}$, the ground-truth next token $y$ is drawn from the teacher distribution
$\bar{p}(v)\coloneqq \bar{p}_{\bar{\theta}}(v\mid x_{<t})$.
\end{assumption}

\paragraph{Cross-entropy gradient.}
Let the student predictive distribution be $q(v)\coloneqq p_{\theta}(v\mid x_{<t})$.
For a single sampled label $y$, the negative log-likelihood is $\mathcal{L}(q;y)=-\log q(y)$.
Differentiating w.r.t.\ the student logits $z$ (where $q=\mathrm{Softmax}(z)$) yields the well-known form:
\begin{equation}
\label{eq:ce_grad_logits_app}
\nabla_{z}\mathcal{L}(q;y)= q - e_y,
\end{equation}
where $e_y$ is the one-hot vector.

\begin{lemma}[Expected gradient energy depends on collision probability]
\label{lem:grad_energy_collision_app}
Under Assumption~\ref{ass:teacher_as_proxy_app}, the expected squared $\ell_2$-norm of the logit-gradient is
\begin{equation}
\label{eq:grad_energy_general_app}
\mathbb{E}_{y\sim \bar{p}}\big[\|\nabla_{z}\mathcal{L}(q;y)\|_2^2\big]
=
\|q\|_2^2 + 1 - 2\langle q,\bar{p}\rangle.
\end{equation}
In particular, when the student matches the teacher locally ($q=\bar{p}$),
\begin{equation}
\label{eq:grad_energy_teacher_app}
\mathbb{E}_{y\sim \bar{p}}\big[\|\nabla_{z}\mathcal{L}(\bar{p};y)\|_2^2\big]
=
1-\|\bar{p}\|_2^2.
\end{equation}
\end{lemma}

\begin{proof}
Using \eqref{eq:ce_grad_logits_app},
\[
\|\nabla_z\mathcal{L}(q;y)\|_2^2
=\|q-e_y\|_2^2
=\|q\|_2^2 + \|e_y\|_2^2 - 2\langle q,e_y\rangle
=\|q\|_2^2 + 1 - 2q(y).
\]
Taking expectation over $y\sim\bar{p}$ gives
$\mathbb{E}[q(y)]=\sum_v \bar{p}(v)q(v)=\langle q,\bar{p}\rangle$,
which yields \eqref{eq:grad_energy_general_app}. Setting $q=\bar{p}$ gives \eqref{eq:grad_energy_teacher_app}.
\end{proof}

\paragraph{Relating collision probability to entropy.}
Define the collision probability $c(\bar{p})\coloneqq \|\bar{p}\|_2^2=\sum_v \bar{p}(v)^2$.
The (order-2) R\'enyi entropy is $H_2(\bar{p})\coloneqq -\log c(\bar{p})$, and it is standard that
Shannon entropy upper-bounds R\'enyi-2 entropy:
\begin{equation}
\label{eq:shannon_ge_renyi2_app}
H_t \ge H_2(\bar{p}) = -\log \|\bar{p}\|_2^2.
\end{equation}
Equivalently,
\begin{equation}
\label{eq:collision_upper_app}
\|\bar{p}\|_2^2 \le e^{-H_t}.
\end{equation}

\begin{proposition}[High entropy implies large expected gradient energy]
\label{prop:high_entropy_more_signal_app}
Under Assumption~\ref{ass:teacher_as_proxy_app} and local student-teacher agreement ($q=\bar{p}$),
the expected gradient energy satisfies
\begin{equation}
\label{eq:entropy_signal_lower_app}
\mathbb{E}_{y\sim \bar{p}}\big[\|\nabla_{z}\mathcal{L}(\bar{p};y)\|_2^2\big]
=
1-\|\bar{p}\|_2^2
\;\ge\;
1-e^{-H_t}.
\end{equation}
Thus, larger teacher entropy $H_t$ yields a (monotone) larger lower bound on gradient energy.
\end{proposition}

\begin{proof}
Combine \eqref{eq:grad_energy_teacher_app} with \eqref{eq:collision_upper_app}:
$1-\|\bar{p}\|_2^2 \ge 1-e^{-H_t}$.
Monotonicity follows since $1-e^{-H_t}$ increases in $H_t$.
\end{proof}

\paragraph{Connection to LoopRPT.}
LoopRPT selects top-$\rho$ high-entropy tokens using \eqref{eq:teacher_entropy_app}. Proposition~\ref{prop:high_entropy_more_signal_app}
shows that, under a standard uncertainty-as-label-proxy assumption, higher teacher entropy provably implies larger expected
cross-entropy gradient energy (and hence more informative updates) near local student-teacher agreement. This provides a
principled explanation for why focusing losses (Sec.~\ref{sec:method:reward}--\ref{sec:method:opt}) on high-entropy positions improves sample-efficiency and
amplifies useful training signals.

% =========================================================
\subsection{Noisy Latent Rollouts Optimize a Smoothed On-Policy Objective}
\label{app:theory_noise}

\paragraph{Noisy latent rollouts.}
LoopRPT injects Gaussian noise into the recurrent latent states (Eq.~(\ref{eq:Gaussian_noise})):
\begin{equation}
\label{eq:noise_injection_app}
h^{(k)} \leftarrow h^{(k)}+\epsilon^{(k)},\qquad \epsilon^{(k)}\sim \mathcal{N}(0,\sigma^2 I).
\end{equation}
Let $\epsilon \coloneqq \{\epsilon^{(k)}\}_{k=1}^{K}$ denote the full noise trajectory, and write the resulting
(noisy) exit distribution as $\pi_{\theta,\epsilon}(k)$.
Under \emph{reward regain}, the per-step student log-probability is recomputed on the same noisy trajectory,
denoted $\ell_{\theta,\epsilon}^{(k)}$, and the step-wise reward is
\begin{equation}
\label{eq:reward_regain_app}
R_{\epsilon}(k)
\coloneqq
\underbrace{\big(\ell_{\theta,\epsilon}^{(k)}-b_{\mathrm{ref}}\big)}_{\Delta_{\mathrm{acc},\epsilon}(k)}
-
\underbrace{\lambda_t(k-t_{\mathrm{ref}})}_{C(k)}.
\end{equation}
This mirrors Eq.~(\ref{eq:acc_delta}--\ref{eq:step_reward}) with $\ell_{\theta}^{(k)}$ replaced by $\ell_{\theta,\epsilon}^{(k)}$.

\begin{definition}[Smoothed on-policy objective induced by latent noise]
\label{def:smoothed_objective_app}
Define
\begin{equation}
\label{eq:smoothed_objective_app}
J_\sigma(\theta)
\coloneqq
\mathbb{E}_{\epsilon\sim\mathcal{N}(0,\sigma^2 I)}
\left[
\sum_{k=1}^{K}\pi_{\theta,\epsilon}(k)\,R_{\epsilon}(k)
\right]
=
\mathbb{E}_{\epsilon}\;\mathbb{E}_{k\sim\pi_{\theta,\epsilon}}\big[R_{\epsilon}(k)\big].
\end{equation}
\end{definition}

\paragraph{Key technical point.}
In the policy-gradient term of LoopRPT (Eq.~(\ref{eq:pg_loss})), $R_{\epsilon}(k)$ is used as a scalar reward signal.
Accordingly, we analyze the score-function gradient, i.e., we do not backpropagate through the internal computation of $R_{\epsilon}(k)$.

\begin{proposition}[Exact score-function gradient of $J_\sigma$]
\label{prop:noise_pg_exact_app}
Under the above convention (treating $R_{\epsilon}(k)$ as a reward signal), the gradient of $J_\sigma$ is
\begin{equation}
\label{eq:noise_pg_exact_app}
\nabla_\theta J_\sigma(\theta)
=
\mathbb{E}_{\epsilon}\;\mathbb{E}_{k\sim\pi_{\theta,\epsilon}}
\Big[
R_{\epsilon}(k)\,\nabla_\theta\log \pi_{\theta,\epsilon}(k)
\Big].
\end{equation}
Consequently, Monte Carlo estimation using $G$ i.i.d.\ noisy rollouts yields an unbiased estimator of
the right-hand side of \eqref{eq:noise_pg_exact_app}.
\end{proposition}

\begin{proof}
By Definition~\ref{def:smoothed_objective_app},
\[
J_\sigma(\theta)=\mathbb{E}_{\epsilon}\sum_{k=1}^{K}\pi_{\theta,\epsilon}(k)R_{\epsilon}(k).
\]
Treat $R_{\epsilon}(k)$ as constant w.r.t.\ $\theta$ in this gradient (score-function convention):
\[
\nabla_\theta J_\sigma(\theta)
=
\mathbb{E}_{\epsilon}\sum_{k=1}^{K} R_{\epsilon}(k)\,\nabla_\theta \pi_{\theta,\epsilon}(k).
\]
Using $\nabla_\theta \pi = \pi\nabla_\theta\log\pi$ gives
\[
\nabla_\theta J_\sigma(\theta)
=
\mathbb{E}_{\epsilon}\sum_{k=1}^{K}\pi_{\theta,\epsilon}(k)R_{\epsilon}(k)\,\nabla_\theta\log \pi_{\theta,\epsilon}(k)
=
\mathbb{E}_{\epsilon}\;\mathbb{E}_{k\sim\pi_{\theta,\epsilon}}
\left[R_{\epsilon}(k)\nabla_\theta\log \pi_{\theta,\epsilon}(k)\right],
\]
which is \eqref{eq:noise_pg_exact_app}. Unbiasedness of the Monte Carlo estimate follows from linearity of expectation.
\end{proof}

\begin{proposition}[Gaussian smoothing yields robustness bounds]
\label{prop:gaussian_smoothing_bounds_app}
Assume for a fixed $\theta$ that the per-noise rollout objective
$F(\epsilon)\coloneqq \mathbb{E}_{k\sim\pi_{\theta,\epsilon}}[R_{\epsilon}(k)]$ is $L$-Lipschitz in $\epsilon$:
$|F(\epsilon)-F(\epsilon')|\le L\|\epsilon-\epsilon'\|$.
Then
\begin{equation}
\label{eq:smoothing_deviation_app}
\big|J_\sigma(\theta)-F(0)\big|
=
\big|\mathbb{E}_{\epsilon}[F(\epsilon)]-F(0)\big|
\le
L\,\mathbb{E}\|\epsilon\|.
\end{equation}
Moreover, for $\epsilon\sim\mathcal N(0,\sigma^2 I_d)$, we have the explicit bound
$\mathbb{E}\|\epsilon\|\le \sigma\sqrt{d}$, hence
$|J_\sigma(\theta)-F(0)|\le L\sigma\sqrt{d}$.
\end{proposition}

\begin{proof}
By Jensen and Lipschitzness,
$|\mathbb{E}F(\epsilon)-F(0)|\le \mathbb{E}|F(\epsilon)-F(0)|\le L\mathbb{E}\|\epsilon\|$,
which is \eqref{eq:smoothing_deviation_app}. For a Gaussian vector in $\mathbb{R}^d$,
$\mathbb{E}\|\epsilon\|\le \sqrt{\mathbb{E}\|\epsilon\|^2}=\sqrt{\mathbb{E}\sum_{i=1}^d \epsilon_i^2}=\sigma\sqrt{d}$.
\end{proof}

\paragraph{Connection to LoopRPT.}
LoopRPT's noisy latent rollouts (Eq.~(\ref{eq:pg_loss})) induce the smoothed on-policy objective $J_\sigma$ in
\eqref{eq:smoothed_objective_app}. Proposition~\ref{prop:noise_pg_exact_app} formalizes that, under the standard
score-function convention used in the policy-gradient term, the rollout gradient estimates the exact
gradient of this smoothed objective. Proposition~\ref{prop:gaussian_smoothing_bounds_app} further shows that latent
Gaussian noise replaces a potentially brittle deterministic rollout objective $F(0)$ by its Gaussian average, which
is robust to local perturbations with a deviation controlled by $\sigma$; this provides a theoretical justification
for improved stability when learning exit behaviors from latent-space variability.

% =========================================================
\subsection{EMA Teacher as a Provably Slowly Moving Reference}
\label{app:theory_ema}

\paragraph{EMA update.}
LoopRPT maintains an EMA teacher updated after each student step:
\begin{equation}
\label{eq:ema_update_app}
\bar{\theta}_{n} = \phi \bar{\theta}_{n-1} + (1-\phi)\theta_n,\qquad \phi\in[0,1).
\end{equation}

\begin{lemma}[Closed form and lag decomposition]
\label{lem:ema_closed_form_app}
For any $n\ge 1$,
\begin{equation}
\label{eq:ema_closed_form_app}
\bar{\theta}_{n} = \phi^n\bar{\theta}_0 + (1-\phi)\sum_{i=1}^{n}\phi^{n-i}\theta_i.
\end{equation}
Moreover, defining parameter increments $\Delta_i \coloneqq \theta_i-\theta_{i-1}$, we have the exact identity
\begin{equation}
\label{eq:ema_lag_decomp_app}
\theta_n-\bar{\theta}_n
=
\phi^n(\theta_0-\bar{\theta}_0) + \sum_{i=1}^{n}\phi^{n-i}\Delta_i.
\end{equation}
\end{lemma}

\begin{proof}
Unrolling \eqref{eq:ema_update_app} yields \eqref{eq:ema_closed_form_app} by induction.
For \eqref{eq:ema_lag_decomp_app}, subtract \eqref{eq:ema_closed_form_app} from $\theta_n$:
\[
\theta_n-\bar{\theta}_n
=
\theta_n-\phi^n\bar{\theta}_0-(1-\phi)\sum_{i=1}^{n}\phi^{n-i}\theta_i.
\]
Rewrite $\theta_n$ as $\phi^n\theta_0 + \sum_{i=1}^{n}\phi^{n-i}(\theta_i-\phi\theta_{i-1})$ and
collect terms to obtain the stated decomposition in $\Delta_i$ (details follow from telescoping).
\end{proof}

\paragraph{Interpretation.}
Eq.~\eqref{eq:ema_lag_decomp_app} shows the teacher \emph{lags} behind the student by a geometrically
weighted sum of recent student updates $\{\Delta_i\}$; hence $\bar{\theta}$ is a low-pass filtered version of $\theta$.

\begin{assumption}[Lipschitz mapping from parameters to exit CDF]
\label{ass:lipschitz_cdf_app}
For any $k\in\{1,\dots,K\}$ and any context $x_{<t}$, there exists $L_\Pi>0$ such that
\begin{equation}
\label{eq:cdf_lipschitz_app}
\big|\Pi_{\theta}(k)-\Pi_{\theta'}(k)\big|
\le L_\Pi \|\theta-\theta'\|.
\end{equation}
\end{assumption}

\begin{proposition}[Teacher reference step is stable under small drift]
\label{prop:tref_stability_app}
Fix $(x_{<t})$ and suppose for some $k^\star$ there is a margin $m>0$ such that
\begin{equation}
\label{eq:tref_margin_app}
\Pi_{\bar{\theta}_{n-1}}(k^\star-1) \le \tau-m
\quad\text{and}\quad
\Pi_{\bar{\theta}_{n-1}}(k^\star) \ge \tau+m.
\end{equation}
If $\|\bar{\theta}_n-\bar{\theta}_{n-1}\| \le \frac{m}{L_\Pi}$, then the teacher reference step computed
from $\bar{\theta}_n$ remains unchanged: $t_{\mathrm{ref}}(\bar{\theta}_n)=k^\star$.
\end{proposition}

\begin{proof}
By Assumption~\ref{ass:lipschitz_cdf_app}, for any $k$,
$|\Pi_{\bar{\theta}_n}(k)-\Pi_{\bar{\theta}_{n-1}}(k)| \le L_\Pi \|\bar{\theta}_n-\bar{\theta}_{n-1}\|\le m$.
Hence
\[
\Pi_{\bar{\theta}_n}(k^\star-1) \le \Pi_{\bar{\theta}_{n-1}}(k^\star-1)+m \le (\tau-m)+m=\tau,
\]
and
\[
\Pi_{\bar{\theta}_n}(k^\star) \ge \Pi_{\bar{\theta}_{n-1}}(k^\star)-m \ge (\tau+m)-m=\tau.
\]
Therefore the minimum $k$ such that $\Pi_{\bar{\theta}_n}(k)\ge\tau$ is still $k^\star$.
\end{proof}

\begin{proposition}[EMA yields $O(1-\phi)$ consecutive-teacher drift]
\label{prop:ema_small_drift_app}
For the EMA update \eqref{eq:ema_update_app},
\begin{equation}
\label{eq:ema_consecutive_drift_app}
\|\bar{\theta}_n-\bar{\theta}_{n-1}\|
=(1-\phi)\|\theta_n-\bar{\theta}_{n-1}\|.
\end{equation}
In particular, when $\theta_n$ and $\bar{\theta}_{n-1}$ are close (typical in late training),
the teacher drift is suppressed by the factor $(1-\phi)$.
\end{proposition}

\begin{proof}
Subtract $\bar{\theta}_{n-1}$ from \eqref{eq:ema_update_app}:
$\bar{\theta}_n-\bar{\theta}_{n-1}=(1-\phi)(\theta_n-\bar{\theta}_{n-1})$. Taking norms yields \eqref{eq:ema_consecutive_drift_app}.
\end{proof}

\paragraph{Connection to LoopRPT.}
LoopRPT computes $t_{\mathrm{ref}}$ (Sec.~3.3) and $b_{\mathrm{ref}}=\ell_{\bar{\theta}}^{(t_{\mathrm{ref}})}$ (Eq.~(\ref{eq:ref_logprob}))
from the teacher. Proposition~\ref{prop:ema_small_drift_app} shows that EMA makes the teacher a slowly moving reference,
and Proposition~\ref{prop:tref_stability_app} further implies that whenever the teacher CDF crosses the exit threshold with a nontrivial margin,
the discrete reference step $t_{\mathrm{ref}}$ is invariant to small parameter drift. Together, these results formalize why EMA reduces
``target chasing'' when jointly optimizing the exit policy and intermediate representations.

\section{Online Hard-token Selection Details}
\label{app:hardtokens}

We perform entropy-based filtering on-the-fly and apply reinforcement-style losses only to the top-$\rho$ high-entropy tokens, following the high-entropy minority token selection strategy in \citet{wang20258020rulehighentropyminority}. Concretely, we compute token entropies using the teacher's \emph{final} latent step, which reflects the model's most refined uncertainty estimate for next-token prediction.

Let $\mathbf{h}^{(K)}_{\bar\theta}$ be the teacher hidden state at the last loop step. We compute logits and shift them for next-token alignment:
\begin{equation}
\mathbf{z}^{(K)}_{\bar\theta} = \mathrm{LMHead}(\mathbf{h}^{(K)}_{\bar\theta}) \in \mathbb{R}^{B\times S\times |\mathcal{V}|},
\qquad
\tilde{\mathbf{z}} = \mathbf{z}^{(K)}_{\bar\theta}(:,\,1:\!-1,\,:) \in \mathbb{R}^{B\times (S-1)\times |\mathcal{V}|}.
\end{equation}
The per-token entropy is then
\begin{equation}
H_{i,t} = -\sum_{v\in\mathcal{V}} p_{i,t}(v)\log p_{i,t}(v),
\quad
p_{i,t}=\mathrm{Softmax}(\tilde{\mathbf{z}}_{i,t,:}).
\end{equation}
We restrict selection to valid training positions using the shifted loss mask $\mathbf{m}\in\{0,1\}^{B\times(S-1)}$ (response tokens). For each example $i$, we compute a \emph{row-wise} quantile threshold over valid positions,
\begin{equation}
q_i = \mathrm{Quantile}\Bigl(\{H_{i,t}:\mathbf{m}_{i,t}=1\},\, 1-\rho\Bigr),
\end{equation}
and define the entropy mask
\begin{equation}
\mathbf{m}_{\textsc{rpt},i,t} = \mathbb{I}\bigl[H_{i,t} > q_i\bigr].
\end{equation}
Finally, we update the effective training mask by elementwise multiplication,
\begin{equation}
\mathbf{m} \leftarrow \mathbf{m}\odot \mathbf{m}_{\textsc{rpt}},
\end{equation}
so all subsequent losses are computed only on the selected high-entropy positions.

\section{K3-style KL surrogate}
\label{app:KL}
We regularize the student against the EMA teacher using the same stable surrogate.
For a given step $k$, let $\ell^{(k)}_{\theta}$ and $\ell^{(k)}_{\bar\theta}$ be the student/teacher log-probabilities of the gold next token.
Define the log-ratio
\begin{equation}
\label{eq:delta_logprob}
\Delta^{(k)} = \ell^{(k)}_{\bar\theta} - \ell^{(k)}_{\theta}.
\end{equation}
Optionally, we clamp $\Delta^{(k)}$ to a bounded range.
We then compute
\begin{equation}
\label{eq:k3}
u^{(k)}=\exp(\Delta^{(k)}),
\qquad
\mathrm{K3}^{(k)} = u^{(k)} - \Delta^{(k)} - 1,
\end{equation}
and clamp $\mathrm{K3}^{(k)}$ to stabilize optimization. Finally, we weight the per-step KL surrogate on the noisy rollout forward pass using the rollout masks $\mathbf{m}^G$, and average over masked tokens:
\begin{equation}
\label{eq:kl_loss}
\mathcal{L}_{KL} =
\frac{\sum \bigl(\sum_k \mathrm{K3}^{(k)}\bigr)\odot \mathbf{m}^G}{\sum \mathbf{m}^G+\epsilon}.
\end{equation}

\section{Additional Experimental Details}
\label{app:setup_details}
\subsection{Details of Datasets}
\paragraph{MMLU~\citep{bench:mmlu}}
MMLU\footnote{\href{https://huggingface.co/datasets/cais/mmlu}{https://huggingface.co/datasets/cais/mmlu}} is a multi-domain multiple-choice benchmark covering 57 subjects, designed to test broad factual knowledge and general reasoning. 

\paragraph{MMLU-Pro~\citep{bench:mmlu_pro}}
MMLU-Pro\footnote{\href{https://huggingface.co/datasets/TIGER-Lab/MMLU-Pro}{https://huggingface.co/datasets/TIGER-Lab/MMLU-Pro}} is an updated variant of MMLU with more challenging questions and stricter evaluation protocols to reduce ambiguity and contamination effects. 

\paragraph{BBH~\citep{bench:bbh}}
BIG-Bench Hard (BBH)\footnote{\href{https://github.com/suzgunmirac/BIG-Bench-Hard}{https://github.com/suzgunmirac/BIG-Bench-Hard}} is a curated subset of BIG-Bench tasks emphasizing compositional and multi-step reasoning (e.g., logical deduction, symbolic manipulation).

\paragraph{ARC-Challenge (ARC-C)~\citep{bench:arc_challenge}}
ARC-C\footnote{\href{https://huggingface.co/datasets/allenai/ai2_arc}{https://huggingface.co/datasets/allenai/ai2\_arc}} is a grade-school science multiple-choice benchmark focusing on difficult questions that require reasoning beyond surface pattern matching. 

\paragraph{HellaSwag~\citep{bench:hellaswag}}
HellaSwag\footnote{\href{https://huggingface.co/datasets/Rowan/hellaswag}{https://huggingface.co/datasets/Rowan/hellaswag}} evaluates commonsense narrative completion by selecting the most plausible continuation among candidates. 

\paragraph{Winogrande~\citep{bench:winogrande}}
Winogrande\footnote{\href{https://huggingface.co/datasets/allenai/winogrande}{https://huggingface.co/datasets/allenai/winogrande}} is a large-scale pronoun resolution benchmark that tests commonsense reasoning and bias-sensitive coreference decisions. 

\paragraph{GSM8K~\citep{bench:gsm8k}}
GSM8K\footnote{\href{https://huggingface.co/datasets/openai/gsm8k}{https://huggingface.co/datasets/openai/gsm8k}} is a grade-school math word-problem benchmark that requires multi-step arithmetic reasoning. 

\paragraph{MATH~\citep{hendrycks2021measuringmathematicalproblemsolving}}
MATH~\footnote{\href{https://huggingface.co/datasets/hendrycks/math}{https://huggingface.co/datasets/hendrycks/math}} is a competition-mathematics dataset with problems labeled by difficulty Level~1--5. We evaluate on the full 5{,}000-problem test split and on \textbf{MATH500}\footnote{\href{https://huggingface.co/datasets/HuggingFaceH4/MATH-500}{https://huggingface.co/datasets/HuggingFaceH4/MATH-500}}, a 500-problem subset drawn from the MATH competition dataset. Both use the official problem-level labels for the difficulty breakdown in Fig.~\ref{fig:math500_levels}.

\paragraph{AIME}
The American Invitational Mathematics Examination\footnote{\href{https://huggingface.co/datasets/AI-MO/aimo-validation-aime}{https://huggingface.co/datasets/AI-MO/aimo-validation-aime}} provides 30 short-answer olympiad problems per year, which we use to probe behavior furthest from the training distribution. We report AIME~2024 and AIME~2025.

\paragraph{OpenThoughts-3}
\textbf{OpenThoughts-3}\footnote{\href{https://huggingface.co/datasets/open-thoughts/OpenThoughts3-1.2M}{https://huggingface.co/datasets/open-thoughts/OpenThoughts3-1.2M}} is a multi-domain reasoning corpus of roughly 100k examples spanning mathematics, code, and science at a $5{:}3{:}2$ ratio. We use it as the broader training corpus in Appendix~\ref{app:scaling}. 

\paragraph{HumanEval~\citep{bench:humaneval}}
HumanEval\footnote{\href{https://huggingface.co/datasets/openai/openai_humaneval}{https://huggingface.co/datasets/openai/openai\_humaneval}} is a code-generation benchmark consisting of programming problems with unit-test based evaluation. 

\paragraph{HumanEval+~\citep{bench:evalplus}}
HumanEval+\footnote{\href{https://huggingface.co/datasets/evalplus/humanevalplus}{https://huggingface.co/datasets/evalplus/humanevalplus}} extends HumanEval with additional and harder tests to reduce overfitting to the original unit tests and better measure functional correctness. 

\paragraph{MBPP~\citep{bench:mbpp}}
MBPP (Mostly Basic Programming Problems)\footnote{\href{https://huggingface.co/datasets/google-research-datasets/mbpp}{https://huggingface.co/datasets/google-research-datasets/mbpp}} is a Python programming benchmark covering short, diverse coding tasks with reference tests. 

\paragraph{MBPP+~\citep{bench:evalplus}}
MBPP+\footnote{\href{https://huggingface.co/datasets/evalplus/mbppplus}{https://huggingface.co/datasets/evalplus/mbppplus}} strengthens MBPP by adding extended test suites, improving robustness of functional evaluation. 

\paragraph{Omni-Math~\citep{gao2024omnimathuniversalolympiadlevel}}
\textsc{Omni-Math}~\footnote{\href{https://huggingface.co/datasets/KbsdJames/Omni-MATH}{https://huggingface.co/datasets/KbsdJames/Omni-MATH}} is a benchmark for evaluating mathematical reasoning in large language models, composed of problems drawn from international mathematics competitions. The dataset emphasizes multi-step reasoning and symbolic manipulation.

\begin{table}[t]
\centering
\begin{minipage}[t]{0.43\textwidth}
    \centering
    \captionof{table}{Hyperparameters used for LoopRPT training.}
    \label{tab:hyperparameters}
    \fitwidth{%
    \begin{tabular}{lc}
        \toprule
        \textbf{Params} & \textbf{Values} \\
        \midrule
        Gradient clip norm & $1.0$ \\
        Teacher update momentum $\phi$ & $0.995$\\
        Batch size & $8$ \\
        Epoch number & $3$\\
        Rollout number & $8$ \\
        Learning rate & $5\times10^{-6}$ \\
        AdamW  & $(0.9, 0.999)$ \\
        Weight decay & $0.01$ \\
        Warmup ratio & $0.03$ \\
        Max sequence length & $4096$ \\
        Top-Entropy ratio $\rho$ & $0.2$ \\
        Time penalty base coefficient $\lambda_{\text{base}}$ & $0.02$ \\
        Time penalty scale coefficient $\lambda_{\text{scale}}$ & $0.5$ \\
        Policy-gradient loss coefficient $\alpha$ & $1.0$ \\
        Backbone loss coefficient $\beta$ & $1.0$ \\
        KL loss coefficient $\delta$  & $0.0001$ \\
        Entropy loss coefficient $\gamma$ & $0.01$ \\
        \bottomrule
    \end{tabular}}
\end{minipage}
\hfill
\begin{minipage}[t]{0.55\textwidth}
    \centering
    \captionof{table}{Evaluation settings and frameworks of benchmarks. Greedy decoding uses no chat template and the maximum-loop or adaptive-early-exit setting.}
    \label{tab:eval_setting}
    \footnotesize
    \fitwidth{%
    \begin{tabular}{@{}lll@{}}
        \toprule
        \textbf{Benchmark} & \textbf{Settings} & \textbf{Framework} \\ \midrule
        \rowcolor[HTML]{F5F5F5} \multicolumn{3}{l}{\textit{\textbf{General}}} \\
        MMLU       & logprobs, 5-shot         & lm-eval-harness \\
        MMLU-Pro   & strict match, 5-shot CoT & lm-eval-harness \\
        BBH        & strict match, 3-shot CoT & lm-eval-harness \\
        MMLU (no-CoT) & logprobs, 0-shot      & lm-eval-harness \\
        BBH (no-CoT)  & strict match, 0-shot   & lm-eval-harness \\
        ARC-C      & logprobs, 25-shot        & lm-eval-harness \\
        HellaSwag  & logprobs, 10-shot        & lm-eval-harness \\
        Winogrande & logprobs, 5-shot         & lm-eval-harness \\ \midrule
        \rowcolor[HTML]{F5F5F5} \multicolumn{3}{l}{\textit{\textbf{Math}}} \\
        GSM8k          & strict match, 3-shot CoT & lm-eval-harness \\
        GSM8k (no-CoT) & strict match, 0-shot     & lm-eval-harness \\
        MATH500        & math\_verify, greedy     & lm-eval-harness \\
        MATH (full)    & math\_verify, greedy     & lm-eval-harness \\
        AIME 2024/2025 & exact match, greedy      & lm-eval-harness \\ \midrule
        \rowcolor[HTML]{F5F5F5} \multicolumn{3}{l}{\textit{\textbf{Code}}} \\
        HumanEval  & pass@1                  & evalplus        \\
        HumanEval+ & pass@1                  & evalplus        \\
        MBPP       & pass@1                  & evalplus        \\
        MBPP+      & pass@1                  & evalplus        \\ \bottomrule
    \end{tabular}}
\end{minipage}
\end{table}

\subsection{Training Hyperparameters and Infrastructure}
We implement our training pipeline using Hugging Face~\citep{train:transformers} and Accelerate~\citep{train:accelerate}, and speed up training with Distributed Data Parallel, mixed precision, and gradient checkpointing. We train for 3 epochs with a maximum sequence length of 4096 and use the AdamW optimizer. We apply a cosine decay learning rate schedule with a warmup phase to stabilize training. For each training example, we independently sample 8 times from an isotropic Gaussian distribution, $\epsilon \sim \mathcal{N}(0,\sigma^2 I)$ with $\sigma=0.1$ (i.e., variance $0.01$), yielding 8 distinct \textsc{Ouro} latent trajectories.
All hyperparameters are summarized in Table~\ref{tab:hyperparameters}. All experiments are conducted on 8 NVIDIA A100 GPUs with 80GB memory. Each LoopRPT run takes approximately 2 hours for Ouro-1.4B and around 4 hours for Ouro-2.6B.

\subsection{Detailed Evaluation Settings}

We evaluated the LoopRPT-trained Ouro model across general knowledge, reasoning, mathematics, science, coding, and multilingual capabilities, focusing on its improvements over the original model. We also compared it with open-source base models of similar scale to Ouro, including the Qwen2.5~\citep{model:qwen2}, Qwen3~\citep{model:qwen3}, Gemma3~\citep{model:gemma3}, LLaMA3.1~\citep{model:llama3}, and LLaMA3.2~\citep{model:llama3} series. Following Ouro~\citep{zhu2025scaling}, all evaluations were conducted using lm-eval-harness~\citep{bench:eval-harness} and evalplus~\citep{bench:evalplus}. Detailed evaluation settings and metrics are provided in Table~\ref{tab:eval_setting}.

\section{Additional Discussion}
\subsection{Compute-Matched Controls}
\label{app:attribution}

Our main-text controls (Table~\ref{tab:attribution}) change only the training objective while holding the checkpoint, data, step budget, and schedule fixed. We use two objectives as controls. \emph{Continued-NTP} applies the standard next-token-prediction loss to the same \textsc{Omni-Math} sequences, which estimates the contribution of additional training on the corpus. \emph{GRPO/RLVR} applies group-relative policy optimization~\citep{shao2024deepseekmathpushinglimitsmathematical} on output-token rewards, which estimates the contribution of an off-the-shelf RL recipe. Table~\ref{tab:attribution_14b} and Table~\ref{tab:attribution_26b} report the full matrix, including the per-difficulty NTR splits, at both scales.

Two patterns emerge. First, Continued-NTP stays flat relative to the base checkpoint on the general benchmarks and on NTR, so a further pass over the corpus is not the source of the improvement. Second, GRPO fails to converge on the looped backbone. The multi-step benchmarks collapse (GSM8K, BBH, ARC-C, and MBPP+), the NTR splits stay flat, and the knowledge-heavy multiple-choice tasks (MMLU and MMLU-Pro) hold. This matches the argument in Sec.~\ref{sec:method} that a sparse output-token reward provides no credit assignment over latent iterations. The remaining LoopRPT-over-Continued-NTP gap therefore measures the contribution of the step-wise latent objective.

\begin{table}[t]
\centering
\begin{minipage}[t]{0.48\textwidth}
    \centering
    \captionof{table}{Compute-matched controls for Ouro-1.4B. NTP is continued next-token prediction (SFT). Best per row in \textbf{bold}.}
    \label{tab:attribution_14b}
    \footnotesize
    \fitwidth{%
    \begin{tabular}{@{}l cccc@{}}
        \toprule
        \textbf{Benchmark} & \textbf{Base} & \textbf{+NTP} & \textbf{+GRPO} & \textbf{+LoopRPT} \\
        \midrule
        NTR-Easy   & 74.40 & 74.59 & 73.62 & \textbf{75.00} \\
        NTR-Med    & 57.20 & 56.63 & 56.44 & \textbf{57.72} \\
        NTR-Hard   & 33.91 & 34.29 & 33.00 & \textbf{34.82} \\
        MMLU       & 67.35 & \textbf{67.82} & 67.60 & 67.62 \\
        MMLU-Pro   & 48.64 & 48.89 & 48.57 & \textbf{49.21} \\
        ARC-C      & 60.95 & 60.58 & 49.23 & \textbf{61.10} \\
        HellaSwag  & 74.30 & 74.51 & 74.36 & \textbf{74.93} \\
        WinoGrande & 71.67 & 71.19 & 66.93 & \textbf{71.98} \\
        BBH        & 70.97 & 71.10 & 28.84 & \textbf{71.37} \\
        GSM8K      & 78.17 & 77.63 & 4.17  & \textbf{78.92} \\
        MATH500    & 20.60 & 29.40 & 12.60 & \textbf{37.00} \\
        HumanEval  & 69.47 & 67.07 & 46.07 & \textbf{69.51} \\
        HumanEval+ & 67.01 & 64.63 & 42.86 & \textbf{67.07} \\
        MBPP       & 73.02 & 73.02 & 31.36 & \textbf{75.40} \\
        MBPP+      & 60.85 & 62.17 & 20.58 & \textbf{63.76} \\
        \bottomrule
    \end{tabular}}
\end{minipage}
\hfill
\begin{minipage}[t]{0.48\textwidth}
    \centering
    \captionof{table}{Compute-matched controls for Ouro-2.6B. NTP is continued next-token prediction (SFT). Best per row in \textbf{bold}.}
    \label{tab:attribution_26b}
    \footnotesize
    \fitwidth{%
    \begin{tabular}{@{}l cccc@{}}
        \toprule
        \textbf{Benchmark} & \textbf{Base} & \textbf{+NTP} & \textbf{+GRPO} & \textbf{+LoopRPT} \\
        \midrule
        NTR-Easy   & 74.51 & 75.80 & 74.20 & \textbf{76.07} \\
        NTR-Med    & 57.35 & 58.83 & 56.58 & \textbf{60.21} \\
        NTR-Hard   & 34.35 & 35.21 & 33.00 & \textbf{37.24} \\
        MMLU       & 73.85 & \textbf{74.00} & 73.60 & 73.91 \\
        MMLU-Pro   & 54.07 & 53.96 & 53.84 & \textbf{54.19} \\
        ARC-C      & 66.12 & 66.75 & 49.89 & \textbf{66.89} \\
        HellaSwag  & 79.28 & 79.74 & 78.07 & \textbf{80.03} \\
        WinoGrande & 75.85 & \textbf{76.58} & 67.19 & 76.47 \\
        BBH        & 77.98 & 78.19 & 29.36 & \textbf{78.24} \\
        GSM8K      & 81.76 & 82.41 & 6.63  & \textbf{85.36} \\
        MATH500    & 43.00 & 43.20 & 18.00 & \textbf{43.50} \\
        HumanEval  & 75.02 & 75.61 & 49.88 & \textbf{76.83} \\
        HumanEval+ & 70.13 & 68.29 & 46.17 & \textbf{71.95} \\
        MBPP       & 77.10 & 76.72 & 61.10 & \textbf{77.24} \\
        MBPP+      & 64.81 & 63.76 & 57.58 & \textbf{65.08} \\
        \bottomrule
    \end{tabular}}
\end{minipage}
\end{table}

\subsection{No-CoT Zero-Shot Evaluation}
\label{app:nocot}
\begin{table}[h]
\centering
\caption{No-CoT zero-shot accuracy (few-shot $=0$; GSM8K answers directly and BBH uses the no-CoT template). Subscripts give the change over Ouro.}
\label{tab:nocot}
\small
\setlength{\tabcolsep}{6pt}
\begin{tabular}{@{}llcc@{}}
\toprule
\textbf{Scale} & \textbf{Benchmark} & \textbf{Ouro} & \textbf{LoopRPT} \\
\midrule
\multirow{3}{*}{1.4B}
 & MMLU  & 68.32 & 68.32 \\
 & GSM8K & 29.87 & 43.52{\color{OliveGreen!90!black}$_{\scriptscriptstyle+13.65}$} \\
 & BBH   & 38.46 & 39.26{\color{OliveGreen!90!black}$_{\scriptscriptstyle+0.80}$} \\
\midrule
\multirow{3}{*}{2.6B}
 & MMLU  & 73.76 & 73.06{\color{Maroon!90!black}$_{\scriptscriptstyle-0.70}$} \\
 & GSM8K & 41.32 & 67.32{\color{OliveGreen!90!black}$_{\scriptscriptstyle+26.00}$} \\
 & BBH   & 54.17 & 56.84{\color{OliveGreen!90!black}$_{\scriptscriptstyle+2.67}$} \\
\bottomrule
\end{tabular}
\end{table}

GSM8K and BBH use few-shot chain-of-thought in the main evaluation. Table~\ref{tab:nocot} repeats them with the scaffold removed (few-shot $=0$, GSM8K answers directly, BBH uses the no-CoT ``Q:/A:'' template), alongside MMLU, which never uses chain-of-thought. Removing the scaffold lowers accuracy sharply on GSM8K and BBH, so the few-shot CoT protocol overstates a looped model's internal reasoning. The LoopRPT gain nonetheless survives: no-CoT GSM8K improves by 13.7 points at 1.4B and 26.0 at 2.6B, while MMLU barely moves.

\subsection{Broader-Corpus Scaling}
\label{app:scaling}
\begin{table}[h]
\centering
\caption{Broader-corpus scaling on OpenThoughts-3 (100k, multi-domain) versus the narrow-domain \textsc{Omni-Math} run, Ouro-1.4B, same LoopRPT recipe. The downstream protocol follows Appendix~\ref{app:setup_details}.}
\label{tab:scaling}
\scriptsize
\renewcommand{\arraystretch}{0.94}
\setlength{\tabcolsep}{5pt}
\begin{tabular}{@{}l cccccc@{}}
\toprule
\textbf{Model (training corpus)} & \textbf{MMLU} & \textbf{ARC-C} & \textbf{HellaSwag} & \textbf{WinoGrande} & \textbf{GSM8K} & \textbf{MBPP+} \\
\midrule
Ouro-1.4B base              & 67.35 & 60.95 & 74.30 & 71.67 & 78.17 & 60.85 \\
$+$LoopRPT (\textsc{Omni-Math})  & 67.62 & 61.10 & 74.93 & 71.98 & 78.92 & 63.76 \\
$+$LoopRPT (OpenThoughts)   & \textbf{67.83} & \textbf{61.19} & \textbf{75.66} & \textbf{72.24} & \textbf{79.03} & \textbf{63.88} \\
\bottomrule
\end{tabular}
\end{table}

Our training corpus is a single narrow domain, so we also run the same LoopRPT recipe on \textbf{OpenThoughts-3} (100k examples, math:code:science $\approx$ 5:3:2), roughly $23\times$ the \textsc{Omni-Math} corpus, on Ouro-1.4B. Table~\ref{tab:scaling} reports the downstream result. All six benchmarks match or exceed the narrow-domain run, so the mechanism does not depend on the narrow math corpus. The calibrated early-exit behavior is also preserved: the Hard-bucket adaptive exit step is 3.12, close to the original 3.07 and below the base checkpoint.

\subsection{Forced-depth Evaluation}

\begin{figure}[bt]
    \centering
    \begin{minipage}{0.49\linewidth}
        \centering
        \includegraphics[width=\linewidth]{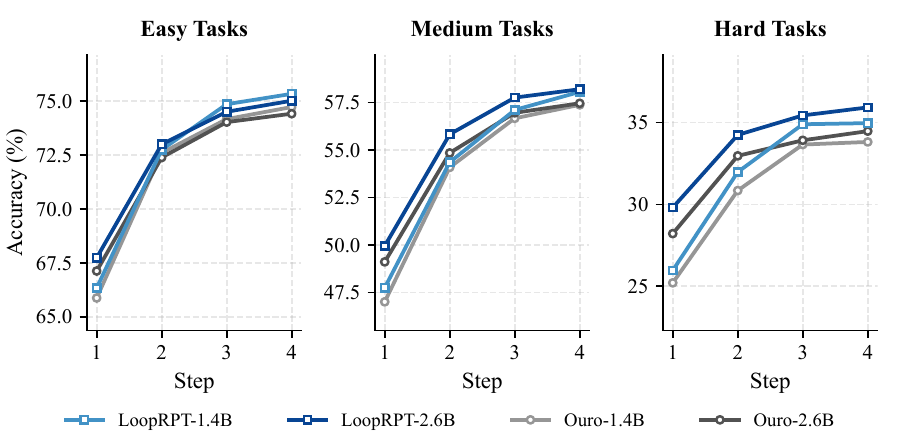}
        \captionof{figure}{Per-step next-token reasoning accuracy across difficulty buckets.}
        \label{fig:step_wise}
    \end{minipage}
    \hfill
    \begin{minipage}{0.49\linewidth}
        \centering
        \includegraphics[width=\linewidth]{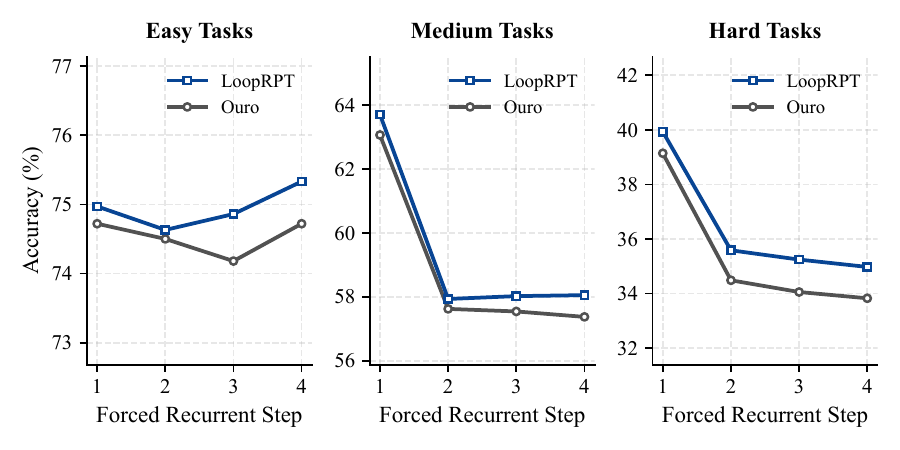}
        \caption{Forced-depth evaluation across difficulty buckets.}
        \label{fig:force_depth}
    \end{minipage}
\end{figure}

We set the exit threshold to $1.0$ and vary the latent steps to evaluate performance under compulsory computation. As shown in Fig.~\ref{fig:force_depth}, forced-depth accuracy is not guaranteed to improve monotonically as depth increases, and may even degrade for Medium and Hard buckets. This behavior is expected for looped models: later latent steps are not explicitly optimized to monotonically refine token-level cross-entropy, and additional iterations can introduce distributional shifts or over-updating of latent representations when the model is forced to continue reasoning. Consequently, forced-depth evaluation measures the behavior under compulsory computation, rather than the intrinsic quality of intermediate representations.

LoopRPT consistently outperforms the baseline across all buckets and depths, indicating that reinforcement pre-training improves robustness even under forced computation. 

\subsection{Analysis of the EMA Mechanism}
\begin{figure}
    \centering
    \includegraphics[width=0.6\linewidth]{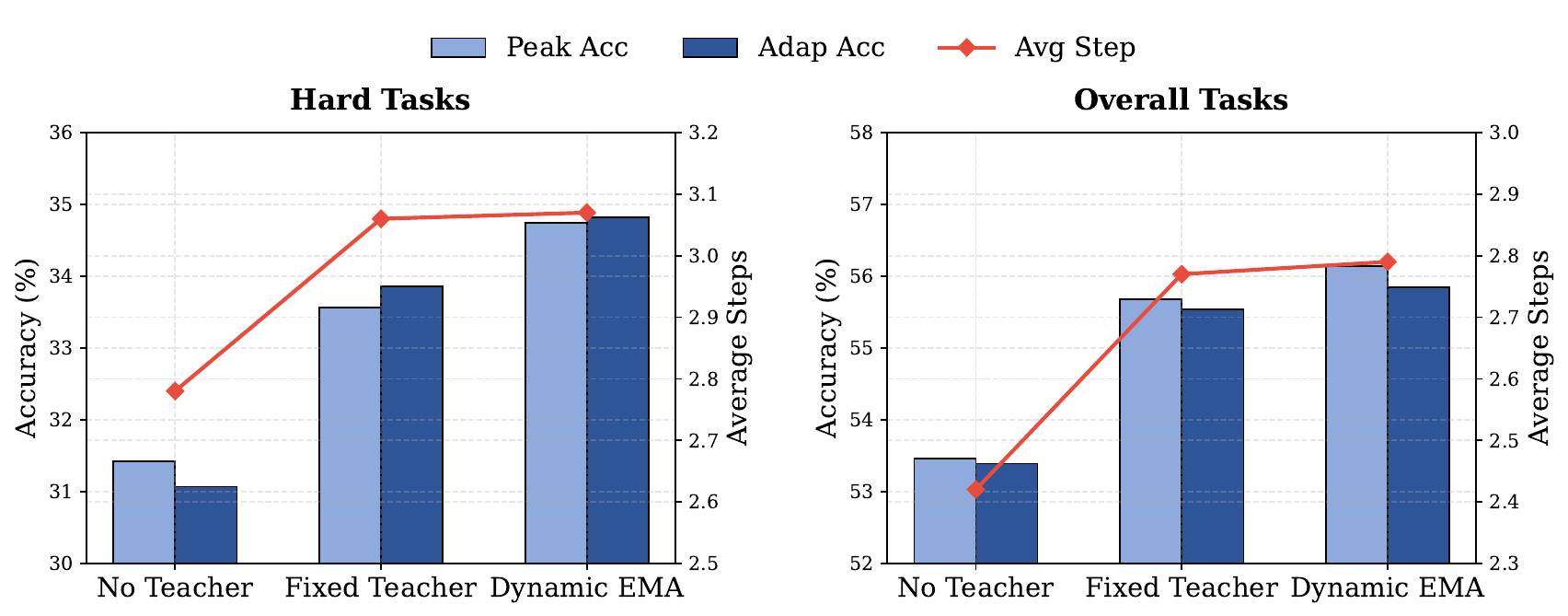}
    \caption{Performance comparison of different EMA teacher mechanisms on next-token reasoning tasks.
}
    \label{fig:ema}
\end{figure}

% \begin{wrapfigure}{r}{0.55\columnwidth}  
%     \centering
%     \includegraphics[width=\linewidth]{figures/ema.pdf}
%     \caption{Performance comparison of different EMA teacher mechanisms on next-token reasoning tasks.
% }
%     \label{fig:ema}
% \end{wrapfigure}
To address the mechanism by which the EMA teacher stabilizes training and avoids potential self-distillation bottlenecks, we further ablate the teacher configuration. Figure~\ref{fig:ema} presents a comparison among three setups: standard LoopRPT (Dynamic EMA), Fixed Teacher ($\phi=1.0$), and No Teacher (Remove the teacher baseline in $R(k)$ and eliminate the KL penalty).

When the reference teacher is removed, performance drops sharply across all metrics. Without a stable reference, the step-wise advantage $\widehat{A}(k)$ becomes erratic. The model attempts to maximize rewards by exploiting premature exits rather than refining latent reasoning, resulting in drastically reduced average steps coupled with a collapse in accuracy. The EMA teacher provides a consistent moving baseline that anchors the reward and prevents the exit policy from collapsing into low-entropy, shallow computations.

Using a Fixed Teacher prevents the extreme degradation seen in the No Teacher setup. However, it falls short of the Dynamic EMA configuration. A static teacher inherently bounds the exploration space, restricting the learning signal to patterns strictly within the initial checkpoint's capabilities. This imposes a ``self-distillation upper limit." The Dynamic EMA, acting as a low-pass filter of the student's recent updates, continuously evolves. It pushes the reward baseline marginally higher as the student improves, incentivizing the model to discover reasoning paths that surpass the initial fixed reference while maintaining stable step-wise constraints.

\subsection{Wall-clock Latency and Physical Efficiency}
\begin{wrapfigure}{r}{0.55\columnwidth}  
    \centering
    \includegraphics[width=\linewidth]{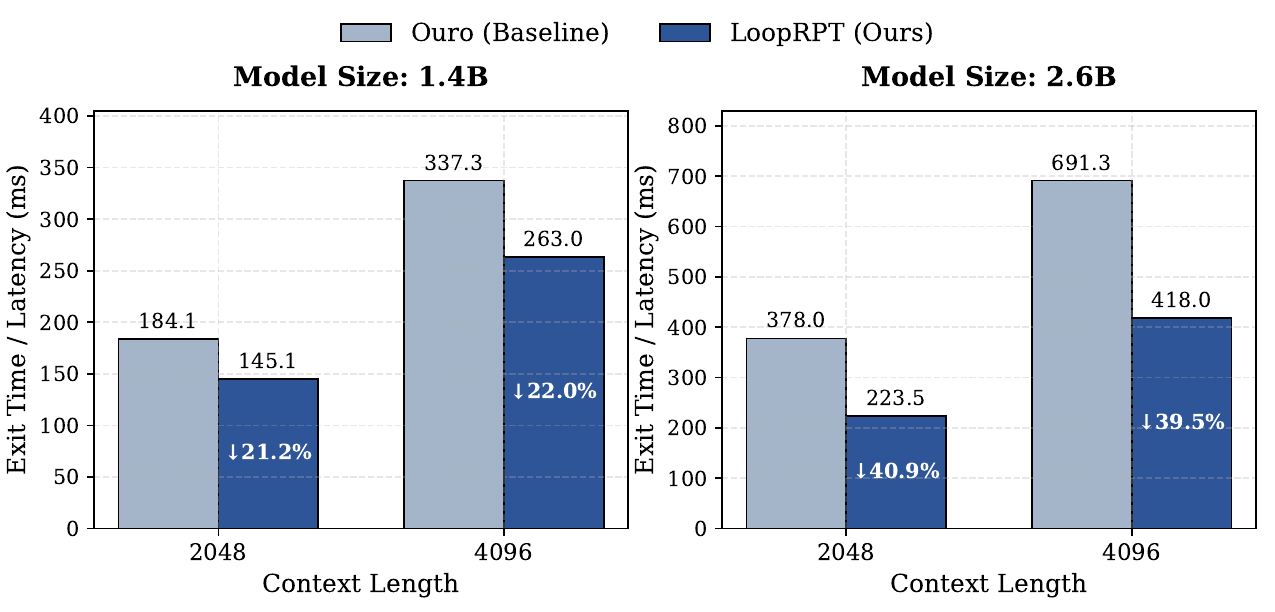}
    \caption{Comparison of actual physical inference time for token exiting.
}
    \label{fig:eff}
\end{wrapfigure}
Figure~\ref{fig:eff} illustrates the average exit time (in milliseconds) required per token generation across varying context lengths on a single A100 GPU. 

Across both model scales and context lengths, LoopRPT consistently reduces physical latency. On Ouro-2.6B at a 4096-token context, LoopRPT cuts real-time latency by 39.5\%.
These results confirm that the early-exit calibration learned via our dense reinforcement pre-training is highly effective. The computational savings from terminating the latent recurrence early (fewer forward passes through the shared transformer block) vastly outweigh any systemic memory bandwidth overheads associated with the exit routing mechanism.

The relative latency reduction also holds as the model size scales from 1.4B to 2.6B, extending the efficiency gain to larger parameter regimes where latent computation costs are more pronounced.

\subsection{Training Cost}
\label{app:cost}
LoopRPT adds three terms on top of Continued-NTP: an EMA teacher forward pass, a grouped noisy student forward pass, and per-step readouts of the shared backbone. The grouped noisy rollout ($G{=}8$) dominates the overhead. The EMA teacher forward stores no activations, and the step-weighted readout reuses hidden states the backbone already produces. Measured per step on a single A100, LoopRPT costs $1.10\times$ GRPO at $G{=}8$, while a full Continued-NTP run takes 1.10~h (1.4B) and 2.14~h (2.6B) on 8$\times$A100. The extra training cost is offset at inference, where the calibrated early exit lowers token latency by 22\% (1.4B) and 39.5\% (2.6B) at a 4096-token context (Fig.~\ref{fig:eff}).

\section{Training Dynamics Visualization}
\label{app:train_curves}

\begin{figure*}[p]
    \centering
    \begin{subfigure}[b]{0.46\textwidth}  
        \centering
        \includegraphics[width=\textwidth]{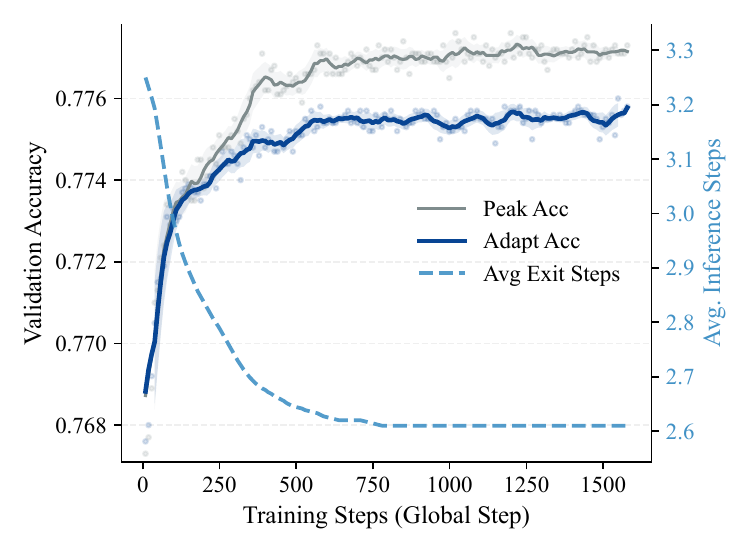}
        \caption{\textbf{Validation accuracy--compute trade-off.} We report validation next-token accuracy under maximum-loop inference (Peak Acc) and adaptive early exit (Adapt Acc), together with the average exit steps (right y-axis). Faint dots indicate raw evaluation points and shaded bands reflect locally smoothed variability.}
        \label{fig:acc_curve_14b}
    \end{subfigure}
    \hfill
    \begin{subfigure}[b]{0.46\textwidth}  
        \centering
        \includegraphics[width=\textwidth]{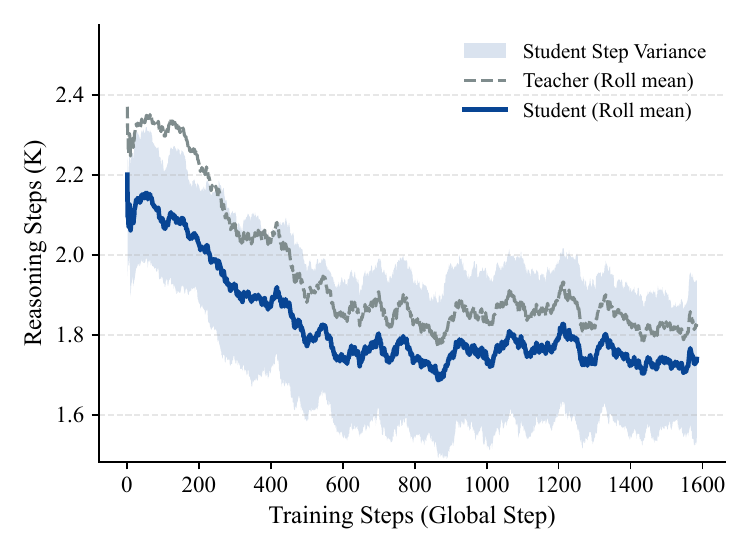}
        \caption{\textbf{Evolution of early-exit behavior during training.} We plot the rolling-mean reasoning steps of the student (solid) and the EMA teacher reference (dashed); the shaded region shows the student’s within-window variability. The student progressively aligns with the reference while converging to fewer steps, indicating improved early-exit calibration.}
        \label{fig:step_curve_14b}
    \end{subfigure}
    
    %\vspace{3pt}  % 添加一些垂直间距
    
    \caption{Training dynamics of LoopRPT on Ouro-1.4B.}
    \label{fig:train_curves_1p4b}
\end{figure*}

\begin{figure*}[t]
    \centering
    \begin{subfigure}[b]{0.46\textwidth}  
        \centering
        \includegraphics[width=\textwidth]{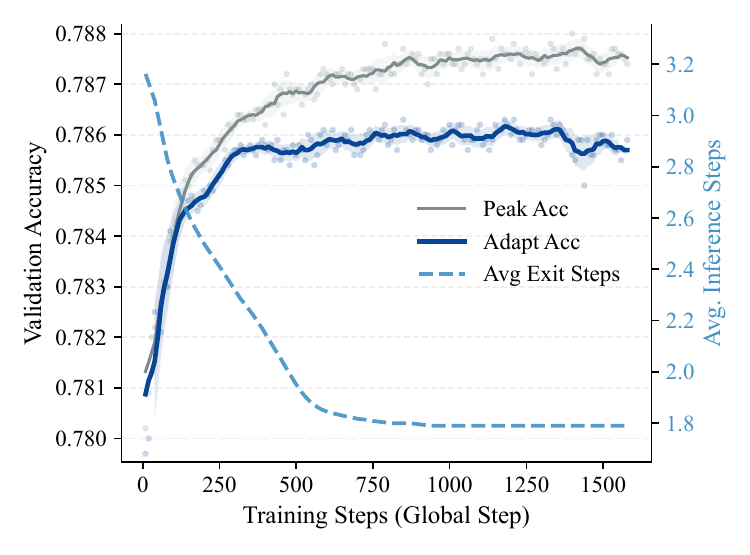}
        \caption{\textbf{Validation accuracy--compute trade-off.} We report validation next-token accuracy under maximum-loop inference (Peak Acc) and adaptive early exit (Adapt Acc), together with the average exit steps (right y-axis). Faint dots indicate raw evaluation points and shaded bands reflect locally smoothed variability.}
        \label{fig:acc_curve_26b}
    \end{subfigure}
    \hfill
    \begin{subfigure}[b]{0.46\textwidth}  
        \centering
        \includegraphics[width=\textwidth]{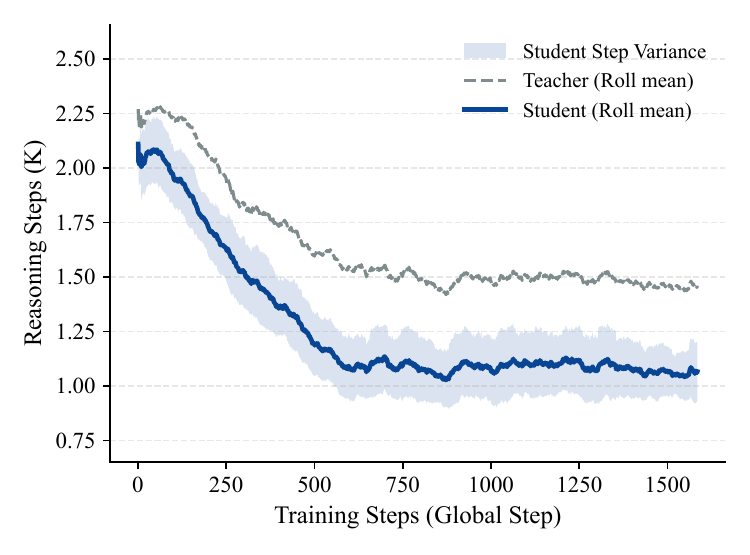}
        \caption{\textbf{Evolution of early-exit behavior during training.} We plot the rolling-mean reasoning steps of the student (solid) and the EMA teacher reference (dashed); the shaded region shows the student’s within-window variability. The student progressively aligns with the reference while converging to fewer steps, indicating improved early-exit calibration.}
        \label{fig:step_curve_26b}
    \end{subfigure}
    
    %\vspace{3pt}  % 添加一些垂直间距
    
    \caption{Training dynamics of LoopRPT on Ouro-2.6B.}
    \label{fig:train_curves_2p6b}
\end{figure*}

Figure~\ref{fig:train_curves_1p4b} and~\ref{fig:train_curves_2p6b} visualizes the training dynamics of LoopRPT on Ouro. Subfigure (a) tracks validation next-token accuracy under maximum-loop inference (Peak) and adaptive early exit (Adap.), together with the average exit steps. Subfigure (b) shows the evolution of the student’s exit behavior relative to the EMA teacher reference: we report rolling-mean reasoning steps and the student’s within-window variability. Overall, the curves indicate that LoopRPT improves validation performance while progressively reducing the required reasoning steps, consistent with better early-exit calibration.

\section{Case Studies}
\label{app:case_analysis}

To further investigate the underlying reasoning mechanisms of LoopRPT compared to the base Ouro model, we conduct a detailed qualitative analysis across three distinct domains: general reasoning, mathematical problem-solving, and code generation. As illustrated in Fig.~\ref{fig:full_case_study}, Fig.~\ref{fig:gsm8k_case_study}, and Fig.~\ref{fig:coding_case_study}, these cases highlight the error-correction and multi-step dependency tracking capabilities of our model.

\paragraph{General reasoning \& factual consistency.} 
In complex scenarios requiring interdisciplinary knowledge (e.g., thermodynamics and biology), the base model frequently exhibits \textit{factual hallucinations} or \textit{logical fragility}. For instance, in Case 2 (Biology), the base model incorrectly categorizes birds as ectotherms despite mentioning their endothermic nature later in the trace. In contrast, LoopRPT maintains high factual consistency throughout the reasoning chain. In Case 3 (Discrete Logic), LoopRPT successfully identifies implicit logical equivalences between different Boolean expressions, whereas the base model fails to recognize the semantic overlap, leading to incomplete conclusions.

\paragraph{Mathematical logic \& constraint satisfaction.} 
The mathematical cases on GSM8K reveal that the base model is prone to \textit{variable mapping errors} and \textit{set-neglect}. In Case 1 (Scheduling), the base model confuses the count of inpatients with appointments, whereas LoopRPT meticulously tracks separate constraints. In Case 2 (Finance) and Case 3 (Fractional Calculation), the base model fails to maintain the global state of sets (e.g., neglecting the combined total of two agents). LoopRPT demonstrates a superior ability to iterate through multi-step arithmetic constraints, likely benefiting from the implicit loops that reinforce state tracking.

\paragraph{Algorithmic invariants in coding.} 
On coding benchmarks (MBPP and HumanEval), the divergence is primarily seen in \textit{boundary condition handling} and \textit{algorithmic interpretation}. The base model often produces ``off-by-one" errors (Case 2) or fails to identify cyclic invariants in array operations (Case 3). LoopRPT consistently generates code that adheres to task-specific constraints and rigorous logic. This suggests that the latent reasoning loops in LoopRPT act as a form of implicit verification, allowing the model to simulate execution states more effectively than standard autoregressive generation.

% 请确保在导言区添加：\usepackage[most]{tcolorbox} 和 \usepackage{pifont}
\begin{figure*}[h]
\centering
\begin{tcolorbox}[
    colback=white, 
    colframe=gray!60, 
    arc=3pt, 
    title=Case Analysis on General Tasks (MMLU-Pro),
    fonttitle=\bfseries\sffamily,
    colbacktitle=gray!10, 
    coltitle=black,
    left=5pt, right=5pt, top=5pt, bottom=5pt
]

% --- CASE 1: Engineering ---
\footnotesize
\textbf{Case 1: Engineering (Thermodynamics)} \\
\textbf{Question:} Three kilograms of air are at an initial state of 100kPa, 300°K. The air is then compressed polytropically with $n = 2.56$, to a final pressure of 500kPa. Assuming constant specificheats, calculate the change in entropy using the three ideal gas equations. \\
\textbf{Options:} (A) 2.50 kJ/°K (B) 1.25 kJ/°K (C) 2.00 kJ/°K (D) 1.75 kJ/°K (E) 2.25 kJ/°K \textbf{(F) 1.57 kJ/°K} (G) 0.75 kJ/°K (H) 1.10 kJ/°K (I) 1.00 kJ/°K (J) 1.40 kJ/°K \\
\textbf{Golden Answer:} \textbf{F}

\vspace{2pt}
\begin{tabularx}{\textwidth}{@{}X|X@{}}
\toprule
\multicolumn{1}{c|}{\textbf{Base Model (Ouro-1.4B)}} & \multicolumn{1}{c}{\textbf{LoopRPT (Ours)}} \\ \midrule
\textit{Full Reasoning:} & \textit{Full Reasoning:} \\
The change in entropy for a polytropic process is given by $\Delta S = nC_p \ln(T_2/T_1) - nR \ln(P_2/P_1)$. For air, $C_p = 1.005$ kJ/kg·K and $R = 0.287$ kJ/kg·K. Using the polytropic relation $P_1V_1^n = P_2V_2^n$, we find $V_2/V_1 = (P_1/P_2)^{1/n} \approx 0.554$. The temperature ratio $T_2/T_1 = (P_2/P_1)^{(n-1)/n} = (500/100)^{(1.56/2.56)} \approx 1.25$. Substituting into the entropy formula, $\Delta S \approx 3 \times 1.005 \times \ln(1.25) - 3 \times 0.287 \times \ln(5) \approx$ \textbf{\textcolor{Maroon}{0.75 kJ/°K}}. & The change in entropy is given by $\Delta S = m C_p \ln \left( \frac{T_2}{T_1} \right) - m R \ln \left( \frac{P_2}{P_1} \right)$. For an ideal gas, $C_p = \frac{nR}{n-1}$ and $R$ is the gas constant. Using the polytropic relation $P_1 V_1^n = P_2 V_2^n$, we find $T_2 = T_1 \left( \frac{P_2}{P_1} \right)^{\frac{n-1}{n}}$. \textbf{\textcolor{OliveGreen}{Substituting the given values, $T_2 = 300 \cdot \left(\frac{500}{100}\right)^{\frac{2.56-1}{2.56}} = 300 \cdot 5^{0.6094} \approx 799.6 \text{ K}$}}. Then, $\Delta S = m \left[ c_p \ln\left(\frac{T_2}{T_1}\right) - R \ln\left(\frac{P_2}{P_1}\right) \right]$. \textbf{\textcolor{OliveGreen}{Calculating each term, $\Delta S = 3 \cdot \left[ 1.005 \ln\left(\frac{799.6}{300}\right) - 0.287 \ln(5) \right] \approx 1.57 \text{ kJ/K}$}}.  \\
\addlinespace[1pt]
\textbf{Conclusion:} \textcolor{Maroon}{The answer is (G) \ding{55}} & \textbf{Conclusion:} \textcolor{OliveGreen}{The answer is (F) \ding{51}} \\ \bottomrule
\end{tabularx}

\vspace{8pt}

% --- CASE 2: Biology ---
\textbf{Case 2: Biology (Thermoregulation)} \\
\textbf{Question:} Which of the following would maintain a constant body temperature in spite of changes in the environmental temperature: frog, robin, fish, dog, or lizard? \\
\textbf{Options:} (A) Robin, Lizard (B) Frog, Fish (C) Lizard, Frog (D) Fish, Dog (E) Fish, Lizard (F) Dog, Frog (G) Lizard, Fish (H) Frog, Robin (I) Dog, Lizard \textbf{(J) Robin, Dog} \\
\textbf{Golden Answer:} \textbf{J}

\vspace{2pt}
\begin{tabularx}{\textwidth}{@{}X|X@{}}
\toprule
\multicolumn{1}{c|}{\textbf{Base Model (Ouro-1.4B)}} & \multicolumn{1}{c}{\textbf{LoopRPT (Ours)}} \\ \midrule
\textit{Full Reasoning:} & \textit{Full Reasoning:} \\
The correct answer is (A) Robin, Lizard. \textbf{\textcolor{Maroon}{Birds and reptiles are ectotherms}}, meaning they rely on external sources of heat to maintain their body temperature. Frogs and fish are also ectotherms, but they are not included in the correct answer. Dogs are endotherms, meaning they can regulate their body temperature internally. & The correct answer is (J). \textbf{\textcolor{OliveGreen}{Dogs and robins are endothermic}}, meaning they can maintain a constant body temperature in spite of changes in the environmental temperature. \textbf{\textcolor{OliveGreen}{Frogs, fish, and lizards are ectothermic}}, meaning they cannot maintain a constant body temperature in spite of changes in the environmental temperature. \\
\addlinespace[1pt]
\textbf{Conclusion:} \textcolor{Maroon}{The answer is (A) \ding{55}} & \textbf{Conclusion:} \textcolor{OliveGreen}{The answer is (J) \ding{51}} \\ \bottomrule
\end{tabularx}

\vspace{8pt}

% --- CASE 3: Computer Science ---
\textbf{Case 3: Computer Science (Discrete Logic)} \\
\textbf{Question:} Which of the following expressions evaluates to true if the person is old enough to drive ($\ge 16$) but not old enough to vote ($< 18$)? \\
I. (age $\ge$ 16) AND (age $\le$ 18) \quad II. (age $\ge$ 16) AND (NOT(age $\ge$ 18)) \quad III. (age $<$ 18) AND (NOT(age $<$ 16)) \\
\textbf{Options:} (A) II only (B) I and II only (C) I and III only \textbf{(D) II and III only} \\
\textbf{Golden Answer:} \textbf{D}

\vspace{2pt}
\begin{tabularx}{\textwidth}{@{}X|X@{}}
\toprule
\multicolumn{1}{c|}{\textbf{Base Model (Ouro-1.4B)}} & \multicolumn{1}{c}{\textbf{LoopRPT (Ours)}} \\ \midrule
\textit{Full Reasoning:} & \textit{Full Reasoning:} \\
The correct answer is II only. The expression (age $\ge$ 16) AND (NOT(age $\ge$ 18)) evaluates to true if the person is at least 16 years old but not at least 18 years old... \textbf{\textcolor{Maroon}{The other expressions do not correctly represent the conditions}}. & The correct answer is (D). The first expression evaluates to true if the person is 16 or 17 years old, but it also evaluates to true if the person is 18 or older... The second expression evaluates to true if the person is 16 or 17 years old... \textbf{\textcolor{OliveGreen}{The third expression evaluates to true if the person is 16 or 17 years old, which is also what we want}}. \\
\addlinespace[1pt]
\textbf{Conclusion:} \textcolor{Maroon}{The answer is (A) \ding{55}} & \textbf{Conclusion:} \textcolor{OliveGreen}{The answer is (D) \ding{51}} \\ \bottomrule
\end{tabularx}

\end{tcolorbox}
\caption{\textbf{Comparative case study on general tasks (MMLU-Pro).} LoopRPT successfully identifies implicit logical equivalences (Case 3) and maintains factual consistency in biological classification (Case 2), whereas the base model exhibits knowledge hallucinations or incomplete reasoning.}
\label{fig:full_case_study}
\end{figure*}
% 请确保在导言区已添加：
% \usepackage[most]{tcolorbox}
% \usepackage{pifont}
% \usepackage{tabularx}
% \usepackage{booktabs}
% \usepackage{xcolor}

\begin{figure*}[h]
\centering
\begin{tcolorbox}[
    colback=white, 
    colframe=gray!60, 
    arc=3pt, 
    title=Case Analysis on Mathematical Tasks (GSM8K),
    fonttitle=\bfseries\sffamily,
    colbacktitle=gray!10, 
    coltitle=black,
    left=5pt, right=5pt, top=5pt, bottom=5pt
]

% --- CASE 1: Doctor Jones ---
\footnotesize
\textbf{Case 1: Scheduling and Time Management}  \\
\textbf{Question:} Doctor Jones is spending nine hours at the clinic. He has to do rounds to check on inpatients (20 min/inpatient) and has ten appointments (30 min each). How many hours will he have left if he has 9 inpatients?  \\
\textbf{Golden Answer:} \textbf{1} 

\vspace{2pt}
\begin{tabularx}{\textwidth}{@{}X|X@{}}
\toprule
\multicolumn{1}{c|}{\textbf{Base Model (Ouro-1.4B)}} & \multicolumn{1}{c}{\textbf{LoopRPT (Ours)}} \\ \midrule
\textit{Full Reasoning:} & \textit{Full Reasoning:} \\
Doctor Jones is spending 9 hours at the clinic. He has to do rounds to check on inpatients, which takes 20 minutes per inpatient. He has 10 appointments, which take 30 minutes each. \textbf{\textcolor{Maroon}{So in total, he spends 20 * 10 + 30 * 10 = 200 + 300 = 500 minutes}}.  Converting minutes to hours, 500 / 60 = 8.33 hours. So he has 9 - 8.33 = 0.67 hours left.  & Doctor Jones is spending 9 hours at the clinic. \textbf{\textcolor{OliveGreen}{He has 9 inpatients at the clinic. So he spends 20 * 9 = 180 minutes doing rounds}}.  He has 10 appointments, which take 30 minutes each. So he spends 30 * 10 = 300 minutes on appointments. \textbf{\textcolor{OliveGreen}{In total, he spends 180 + 300 = 480 minutes}}.  9 hours is 540 minutes. So he has 540 - 480 = 60 minutes left (1 hour).  \\
\addlinespace[1pt]
\textbf{Conclusion:} \textcolor{Maroon}{The answer is 0.67 \ding{55}}  & \textbf{Conclusion:} \textcolor{OliveGreen}{The answer is 1 \ding{51}}  \\ \bottomrule
\end{tabularx}

\vspace{8pt}

% --- CASE 2: Maggie & Riza ---
\textbf{Case 2: Multi-Agent Finance}  \\
\textbf{Question:} Maggie spent a quarter of her money, while Riza spent one-third of her money. They each had \$60. How much money do the two of them have left?  \\
\textbf{Golden Answer:} \textbf{85} 

\vspace{2pt}
\begin{tabularx}{\textwidth}{@{}X|X@{}}
\toprule
\multicolumn{1}{c|}{\textbf{Base Model (Ouro-1.4B)}} & \multicolumn{1}{c}{\textbf{LoopRPT (Ours)}} \\ \midrule
\textit{Full Reasoning:} & \textit{Full Reasoning:} \\
Maggie spent 1/4 * 60 = 15. Riza spent 1/3 * 60 = 20. So in total they spent 15 + 20 = 35. \textbf{\textcolor{Maroon}{So they had 60 - 35 = 25}}.  & Maggie spent 1/4 * 60 = 15. Riza spent 1/3 * 60 = 20. Together, they spent 15 + 20 = 35. \textbf{\textcolor{OliveGreen}{They had 60 + 60 = 120. So they have 120 - 35 = 85 left}}.  \\
\addlinespace[1pt]
\textbf{Conclusion:} \textcolor{Maroon}{The answer is 25 \ding{55}}  & \textbf{Conclusion:} \textcolor{OliveGreen}{The answer is 85 \ding{51}}  \\ \bottomrule
\end{tabularx}

\vspace{8pt}

% --- CASE 3: John's Learning Time ---
\textbf{Case 3: Fractional Multi-step Calculation}  \\
\textbf{Question:} John has 8 classes (40 min each) for 5 days. He spends 1/16 of his weekly minutes each on Saturday and Sunday as extra time. How many hours a week does he spend learning?  \\
\textbf{Golden Answer:} \textbf{30} 

\vspace{2pt}
\begin{tabularx}{\textwidth}{@{}X|X@{}}
\toprule
\multicolumn{1}{c|}{\textbf{Base Model (Ouro-1.4B)}} & \multicolumn{1}{c}{\textbf{LoopRPT (Ours)}} \\ \midrule
\textit{Full Reasoning:} & \textit{Full Reasoning:} \\
Total class time is 8 * 40 * 5 = 1600 minutes.  He spends 1/16 of his weekly minutes on Saturday and Sunday. \textbf{\textcolor{Maroon}{So in total he spends 1600 * 1/16 = 100 minutes}}.  So in total he spends 1600 + 100 = 1700 minutes (28.33 hours).  & Total class time is 8 * 40 * 5 = 1600 minutes.  \textbf{\textcolor{OliveGreen}{He then spends 1600 / 16 = 100 minutes each on Saturday and Sunday. So he spends 1600 + 100 + 100 = 1800 minutes in total}}.  1800 minutes is 30 hours.  \\
\addlinespace[1pt]
\textbf{Conclusion:} \textcolor{Maroon}{The answer is 28.33 \ding{55}}  & \textbf{Conclusion:} \textcolor{OliveGreen}{The answer is 30 \ding{51}}  \\ \bottomrule
\end{tabularx}

\end{tcolorbox}
\caption{\textbf{Comparison of GSM8K mathematical reasoning}. Base model (Ouro-1.4B) frequently makes ``set-neglect" errors (Case 2, 3) or hallucinations in variable mapping (Case 1), whereas LoopRPT correctly iterates through multi-step logic constraints.}
\label{fig:gsm8k_case_study}
\end{figure*}
% 请确保导言区包含：\usepackage[most]{tcolorbox} \usepackage{pifont} \usepackage{tabularx} \usepackage{booktabs}

\begin{figure*}[h]
\centering
\begin{tcolorbox}[
    colback=white, 
    colframe=gray!60, 
    arc=3pt, 
    title=Case Analysis on Coding Tasks (MBPP \& HumanEval),
    fonttitle=\bfseries\sffamily,
    colbacktitle=gray!10, 
    coltitle=black,
    left=5pt, right=5pt, top=5pt, bottom=5pt
]

\footnotesize
\setlength{\tabcolsep}{4pt}
\begin{tabularx}{\textwidth}{@{}X|X@{}}

% --- CASE 1 ---
\multicolumn{2}{l}{\textbf{Case 1: MBPP (Mathematical Sequence)}}  \\
\multicolumn{2}{l}{\textbf{Question:} Write a function to find the nth octagonal number. (Assertion: \texttt{is\_octagonal(5) == 65})}  \\
\multicolumn{2}{l}{\textbf{Golden Answer:} \texttt{return 3 * n * n - 2 * n}}  \\ \addlinespace[2pt]
\toprule
\multicolumn{1}{c|}{\textbf{Base Model (Ouro-1.4B)}} & \multicolumn{1}{c}{\textbf{LoopRPT (Ours)}} \\ \midrule

\begin{tabular}[t]{@{}l@{}}
\textit{Generated Code:} \\
\texttt{def is\_octagonal(n):} \\
\texttt{\ \ \# Returns True if n is octagonal...} \\
\texttt{\ \ return n == (n * (3 * n - 1)) // 2} 
\end{tabular} & 
\begin{tabular}[t]{@{}l@{}}
\textit{Generated Code:} \\
\texttt{def is\_octagonal(n):} \\
\texttt{\ \ \# Returns the n-th octagonal number.} \\
\texttt{\ \ return 3 * n * n - 2 * n} 
\end{tabular} \\ \addlinespace[4pt]

\textcolor{Maroon}{\textbf{Analysis:} Confuses ``nth number" request with ``check property" and hallucinates pentagonal formula.}  & 
\textcolor{OliveGreen}{\textbf{Analysis:} Correctly identifies the task and provides the exact octagonal sequence formula.}  \\ 
\textbf{Conclusion:} \textcolor{Maroon}{Functional Failure \ding{55}}  & \textbf{Conclusion:} \textcolor{OliveGreen}{Correct Generation \ding{51}}  \\ \midrule

% --- CASE 2 ---
\multicolumn{2}{l}{\textbf{Case 2: HumanEval (String Manipulation)}}  \\
\multicolumn{2}{l}{\textbf{Question:} Find the shortest palindrome that begins with a supplied string by reverse-prefixing.}  \\
\multicolumn{2}{l}{\textbf{Golden Answer:} Iteratively find longest palindromic suffix and append reverse prefix.}  \\ \addlinespace[2pt]

\toprule
\multicolumn{1}{c|}{\textbf{Base Model (Ouro-1.4B)}} & \multicolumn{1}{c}{\textbf{LoopRPT (Ours)}} \\ \midrule

\begin{tabular}[t]{@{}l@{}}
\textit{Key Implementation:} \\
\texttt{for i in range(len(string), 0, -1):} \\
\texttt{\ \ if is\_palindrome(string[-i:]):} \\
\texttt{\ \ \ \ break} \\
\texttt{return string + string[:i-1][::-1]} 
\end{tabular} & 
\begin{tabular}[t]{@{}l@{}}
\textit{Key Implementation:} \\
\texttt{for i in range(n):} \\
\texttt{\ \ if is\_palindrome(string[i:]):} \\
\texttt{\ \ \ \ return string + string[:i][::-1]} 
\end{tabular} \\ \addlinespace[4pt]

\textcolor{Maroon}{\textbf{Analysis:} Off-by-one error in indexing (\texttt{i-1}) leads to incorrect palindrome completion.}  & 
\textcolor{OliveGreen}{\textbf{Analysis:} Accurately iterates through split points to find the optimal palindromic suffix.}  \\
\textbf{Conclusion:} \textcolor{Maroon}{Logic Error \ding{55}}  & \textbf{Conclusion:} \textcolor{OliveGreen}{Correct Generation \ding{51}}  \\ \midrule

% --- CASE 3 ---
\multicolumn{2}{l}{\textbf{Case 3: HumanEval (Cyclic Array Order)}}  \\
\multicolumn{2}{l}{\textbf{Question:} Determine if an array can be sorted by performing any number of right shifts.}  \\ \addlinespace[2pt]
\toprule
\multicolumn{1}{c|}{\textbf{Base Model (Ouro-1.4B)}} & \multicolumn{1}{c}{\textbf{LoopRPT (Ours)}} \\ \midrule

\begin{tabular}[t]{@{}l@{}}
\textit{Key Implementation:} \\
\texttt{for i in range(n):} \\
\texttt{\ \ if arr[i] > arr[(i + 1) \% n]:} \\
\texttt{\ \ \ \ return False} \\
\texttt{return True} 
\end{tabular} & 
\begin{tabular}[t]{@{}l@{}}
\textit{Key Implementation:} \\
\texttt{sorted\_arr = sorted(arr)} \\
\texttt{for i in range(len(arr)):} \\
\texttt{\ \ if arr==sorted\_arr[i:]+sorted\_arr[:i]:} \\
\texttt{\ \ \ \ return True} 
\end{tabular} \\ \addlinespace[4pt]

\textcolor{Maroon}{\textbf{Analysis:} Incorrectly implements a strict sorted check, failing for valid cyclic shifts.}  & 
\textcolor{OliveGreen}{\textbf{Analysis:} Correctly uses a cyclic comparison loop to verify all possible shift states.}  \\
\textbf{Conclusion:} \textcolor{Maroon}{Logic Error \ding{55}}  & \textbf{Conclusion:} \textcolor{OliveGreen}{Correct Generation \ding{51}}  \\ \bottomrule

\end{tabularx}
\end{tcolorbox}
\caption{\textbf{Code generation comparison.} Base model (Ouro-1.4B) fails on boundary conditions and algorithmic interpretation, while LoopRPT demonstrates superior adherence to task constraints.}
\label{fig:coding_case_study}
\end{figure*}

\section{Full Training Algorithm}
\label{app:algorithm}

Alg.~\ref{alg:phase1_reward_table_default}--\ref{alg:phase2_joint_default} provides a description of our training iteration. The algorithm uses two phases: (i) a no-gradient phase that constructs a dense step-wise reward table for each selected token, and (ii) a gradient phase that jointly optimizes the exit policy via noisy latent rollouts and the backbone via step-weighted next-token learning.

\begin{table*}[phtb]
\centering
\caption{Comparison of 2.6B LoopLM with 3-12B parameter baselines. Results for vanilla LLMs are taken from \cite{zhu2025scaling}, with Ouro results obtained via our reimplementation. Best in \textbf{bold}, second best \underline{underlined}. Subscripts indicate LoopRPT's improvement over Ouro baseline.}
\label{tab:main_26b}
\footnotesize % 保持字体比例协调，避免使用 resizebox
\renewcommand{\arraystretch}{0.94}
\setlength{\tabcolsep}{2.5pt} % 针对11列数据进行极限间距压缩
\begin{tabular}{l ccccccccc >{\columncolor[HTML]{ECF4FF}}c}
\toprule
\textbf{Benchmark} & 
\begin{tabular}[c]{@{}c@{}}Gemma3\\3B\end{tabular} & 
\begin{tabular}[c]{@{}c@{}}Llama3.2\\3B\end{tabular} & 
\begin{tabular}[c]{@{}c@{}}Qwen2.5\\4B\end{tabular} & 
\begin{tabular}[c]{@{}c@{}}Qwen3\\4B\end{tabular} & 
\begin{tabular}[c]{@{}c@{}}Qwen2.5\\7B\end{tabular} & 
\begin{tabular}[c]{@{}c@{}}Llama3.1\\8B\end{tabular} & 
\begin{tabular}[c]{@{}c@{}}Qwen3\\8B\end{tabular} & 
\begin{tabular}[c]{@{}c@{}}Gemma3\\12B\end{tabular} & 
\begin{tabular}[c]{@{}c@{}}Ouro\\2.6B\end{tabular} & 
\begin{tabular}[c]{@{}c@{}}\textbf{LoopRPT}\\\textbf{2.6B}\end{tabular} \\ \midrule
\rowcolor[HTML]{F5F5F5} \multicolumn{11}{c}{\textit{\textbf{General Tasks}}} \\
MMLU       & 65.62 & 59.69 & 73.19 & 58.37 & \underline{74.20} & 73.02 & \textbf{76.63} & 72.14 & 73.85 & 73.91{\color{OliveGreen!90!black}$_{\scriptscriptstyle+0.06}$} \\
MMLU-Pro   & 37.87 & 33.34 & 51.40 & 34.61 & 43.55 & 43.24 & \underline{53.72} & 49.21 & 54.07 & \textbf{54.19}{\color{OliveGreen!90!black}$_{\scriptscriptstyle+0.12}$} \\
BBH        & 55.37 & 39.45 & 71.14 & 66.32 & 53.72 & 71.56 & 77.65 & \textbf{78.41} & 77.98 & \underline{78.24}{\color{OliveGreen!90!black}$_{\scriptscriptstyle+0.26}$} \\
ARC-C      & 55.46 & 52.47 & 63.65 & 60.75 & 63.65 & 60.75 & 66.10 & \textbf{72.44} & 66.12 & \underline{66.89}{\color{OliveGreen!90!black}$_{\scriptscriptstyle+0.77}$} \\
HellaSwag  & 74.54 & 73.09 & 75.66 & 75.58 & 79.98 & \underline{81.97} & 79.60 & \textbf{83.68} & 79.28 & 80.03{\color{OliveGreen!90!black}$_{\scriptscriptstyle+0.75}$} \\
Winogrande & 70.17 & 69.14 & 71.19 & 71.27 & 76.48 & \underline{77.11} & 76.80 & \textbf{77.74} & 75.85 & 76.47{\color{OliveGreen!90!black}$_{\scriptscriptstyle+0.62}$} \\ \midrule
\rowcolor[HTML]{F5F5F5} \multicolumn{11}{c}{\textit{\textbf{Math \& Coding Tasks}}} \\
GSM8K      & 74.60 & 67.20 & 72.86 & 68.69 & 81.50 & 78.17 & \underline{83.09} & 77.18 & 81.76 & \textbf{85.36}{\color{OliveGreen!90!black}$_{\scriptscriptstyle+3.60}$} \\
HumanEval  & 68.90 & 29.90 & 77.70 & 34.80 & \underline{79.30} & 38.40 & \textbf{84.80} & 46.30 & 75.02 & 76.83{\color{OliveGreen!90!black}$_{\scriptscriptstyle+1.81}$} \\
HumanEval+ & 62.20 & 26.20 & 70.70 & 29.30 & 70.60 & 31.10 & \textbf{75.30} & 37.20 & 70.13 & \underline{71.95}{\color{OliveGreen!90!black}$_{\scriptscriptstyle+1.82}$} \\
MBPP       & 63.00 & 50.30 & \underline{78.80} & 60.60 & 73.80 & 62.40 & \textbf{79.00} & 73.50 & 77.10 & 77.24{\color{OliveGreen!90!black}$_{\scriptscriptstyle+0.14}$} \\
MBPP+      & 54.20 & 39.70 & 65.90 & 51.10 & 63.50 & 51.60 & \textbf{67.90} & \underline{66.10} & 64.81 & 65.08{\color{OliveGreen!90!black}$_{\scriptscriptstyle+0.27}$} \\ \bottomrule
\end{tabular}
\end{table*}

\begin{algorithm}[p]
\caption{Phase I: Dense Step-wise Reward Table Construction with RPT Token Filtering}
\label{alg:phase1_reward_table_default}
\begin{algorithmic}[1]
\Require Batch $\{x_{0:S}\}$ with valid-token mask $m\in\{0,1\}^{B\times(S-1)}$ (aligned to next-token positions);
student $\theta$; EMA teacher $\bar{\theta}$;
max steps $K$; exit threshold $\tau$; top-ratio $\rho$;
$\lambda_{\mathrm{base}},\lambda_{\mathrm{scale}}$; constant $\varepsilon$.
\Ensure RPT mask $m_{\mathrm{RPT}}$; reward table $\{R_t(k)\}_{k=1}^K$; step advantage $\{A_{b,t}(k)\}_{k=1}^K$.

\State \textbf{Teacher forward (for hard-token selection and reference quantities)}
\State Run teacher LoopLM for $K$ steps to obtain last-step logits $\bar{z}_t^{(K)}$ and gate logits $\{\bar{a}_t^{(k)}\}_{k=1}^K$ at each next-token position $t$.
\State Compute last-step entropy $H_t^{\mathrm{last}} = -\sum_{v\in V}\bar{p}_t(v)\log \bar{p}_t(v)$ where $\bar{p}_t=\mathrm{Softmax}(\bar{z}_t^{(K)})$.

\State \textbf{RPT hard-token filtering (top-$\rho$ entropy on valid positions)}
\For{each example $i$ in the batch}
    \State $q_i \leftarrow \mathrm{Quantile}(\{H_{i,t}^{\mathrm{last}}: m_{i,t}=1\}, 1-\rho)$
    \State $m_{\mathrm{RPT},i,t}\leftarrow \mathbb{I}[H_{i,t}^{\mathrm{last}} > q_i]$
\EndFor
\State $m \leftarrow m \odot m_{\mathrm{RPT}}$

\State \textbf{Reference step $t_{\mathrm{ref}}$ and dynamic baseline $b_{\mathrm{ref}}$}
\State Compute teacher exit distribution $\pi_{\bar{\theta}}(k)$ via Eq.~(\ref{eq:exit_prob}--\ref{eq:s_recursion}) from $\{\bar{a}_t^{(k)}\}$ and its CDF $\Pi_{\bar{\theta}}(k)=\sum_{j\le k}\pi_{\bar{\theta}}(j)$.
\State $t_{\mathrm{ref}}\leftarrow \min\{k:\Pi_{\bar{\theta}}(k)\ge\tau\}$ \quad (use $K$ if never exceeded).
\State Compute teacher per-step gold log-prob $\ell_{\bar{\theta}}^{(k)}(t)=\log p_{\bar{\theta}}(x_t\mid x_{<t};\bar{h}_t^{(k)})$.
\State Set $b_{\mathrm{ref}}(t)\leftarrow \ell_{\bar{\theta}}^{(t_{\mathrm{ref}})}(t)$ \Comment{Eq.~\ref{eq:ref_logprob}}

\State \textbf{Difficulty-aware time penalty weight $\lambda_t$}
\State Compute teacher entropy at the reference step $H_t^{\mathrm{ref}}$ (from teacher distribution at $t_{\mathrm{ref}}$).
\State $d_t \leftarrow \mathrm{Clamp}\!\left(H_t^{\mathrm{ref}}/\log|V|,\,0,\,1\right)$,\quad
$\lambda_t\leftarrow \lambda_{\mathrm{base}}\left(1+\lambda_{\mathrm{scale}}(1-d_t)\right)$ \Comment{Eq.~\ref{eq:lambda_t}}

\State \textbf{Student deterministic pass for step-wise rewards}
\State Run student without noise to obtain $\ell_{\theta}^{(k)}(t)=\log p_{\theta}(x_t\mid x_{<t};h_t^{(k)})$ for $k=1,\dots,K$.
\For{$k=1,\dots,K$}
    \State $\Delta_{\mathrm{acc},t}(k)\leftarrow \ell_{\theta}^{(k)}(t)-b_{\mathrm{ref}}(t)$ \Comment{Eq.~\ref{eq:acc_delta}}
    \State $C_t(k)\leftarrow \lambda_t\,(k-t_{\mathrm{ref}})$ \Comment{Eq.~\ref{eq:time_penalty}}
    \State $R_t(k)\leftarrow \Delta_{\mathrm{acc},t}(k)-C_t(k)$ \Comment{Eq.~\ref{eq:step_reward}}
\EndFor

\State \textbf{Step advantage for representation learning}
\State $\mu_{R,t}\leftarrow \frac{1}{K}\sum_{j=1}^K R_t(j)$,\quad
$\sigma_{R,t}\leftarrow \sqrt{\frac{1}{K}\sum_{j=1}^K (R_t(j)-\mu_{R,t})^2}$
\For{$k=1,\dots,K$}
    \State $A_{b,t}(k)\leftarrow \frac{R_t(k)-\mu_{R,t}}{\sigma_{R,t}+\varepsilon}$ \Comment{Eq.~\ref{eq:advantage_nonoise}}
\EndFor

\State \Return $(m_{\mathrm{RPT}}, \{R_t(k)\}_{k=1}^K, \{A_{b,t}(k)\}_{k=1}^K)$ on positions where $m=1$.
\end{algorithmic}
\end{algorithm}

\begin{algorithm}[p]
\caption{Phase II: Noisy Rollout Policy Gradient + Step-Weighted Next-Token Training}
\label{alg:phase2_joint_default}
\begin{algorithmic}[1]
\Require Batch $\{x_{0:S}\}$ with effective mask $m$ (after Alg.~\ref{alg:phase1_reward_table_default});
student $\theta$, EMA teacher $\bar{\theta}$;
max steps $K$, threshold $\tau$;
reference quantities from Phase I: $t_{\mathrm{ref}}$, $b_{\mathrm{ref}}(t)$, $\lambda_t$, and $A_{b,t}(k)$;
rollouts $G$, noise scale $\sigma$;
loss weights $\alpha,\beta,\gamma,\delta$; constant $\varepsilon$.
\Ensure Updated $\theta$ and EMA teacher $\bar{\theta}$.

\State \textbf{Deterministic student pass (for $\pi_\theta$ and step-weighted NTP)}
\State Run student without noise to obtain gate logits $\{a_t^{(k)}\}_{k=1}^K$ and per-step gold log-probs $\ell_{\theta}^{(k)}(t)$.
\State Compute exit distribution $\pi_\theta(k)$ via Eq.~(\ref{eq:exit_prob}--\ref{eq:s_recursion}).
\For{$k=1,\dots,K$}
    \State $w_t(k)\leftarrow \text{softmax}\Big[\pi_\theta(k)\Big(1+\mathrm{ReLU}(A_{b,t}(k))\Big)\Big]$ \Comment{Eq.~\ref{eq:bb_loss_weight}}
\EndFor
\State $L_{\mathrm{rep}}\leftarrow -\sum_{k=1}^K w_t(k)\,\ell_{\theta}^{(k)}(t)$ \Comment{Eq.~\ref{eq:backbone_loss}}
\Statex \hfill (Applied only on positions where $m=1$; sums/means are over those positions.)

\State \textbf{Entropy regularization on exit distribution}
\State $L_{\mathrm{ent}}\leftarrow -\mathbb{E}\Big[\sum_{k=1}^K \pi_\theta(k)\log \pi_\theta(k)\Big]$.

\State \textbf{Noisy latent rollouts with reward regain (on-policy reward recomputation)}
\For{$g=1,\dots,G$}
    \State Run a noisy recurrence (Eq.~\ref{eq:Gaussian_noise}) to obtain noisy states $\{h_{t,g}^{(k)}\}_{k=1}^K$ and rollout gates $\{a_{t,g}^{(k)}\}_{k=1}^K$.
    \State Compute rollout exit distribution $\pi_{\theta}^{(g)}(k)$ from $\{a_{t,g}^{(k)}\}$.
    \State Recompute per-step gold log-probs on rollout states:
    $\ell_{\theta,g}^{(k)}(t)=\log p_\theta(x_t\mid x_{<t};h_{t,g}^{(k)})$.
    \For{$k=1,\dots,K$}
        \State $\Delta_{\mathrm{acc},t}^{(g)}(k)\leftarrow \ell_{\theta,g}^{(k)}(t)-b_{\mathrm{ref}}(t)$
        \State $C_t(k)\leftarrow \lambda_t\,(k-t_{\mathrm{ref}})$
        \State $R_t^{(g)}(k)\leftarrow \Delta_{\mathrm{acc},t}^{(g)}(k)-C_t(k)$
    \EndFor
    \State Sample exit step $t^{(g)} \sim \pi_{\theta}^{(g)}(\cdot)$ and set rollout reward $r^{(g)}\leftarrow R_t^{(g)}(t^{(g)})$.
\EndFor

\State \textbf{Group-normalized rollout advantage and policy gradient}
\State $A^{(g)} \leftarrow \dfrac{r^{(g)}-\mathrm{mean}_g[r^{(g)}]}{\mathrm{std}_g[r^{(g)}]+\varepsilon}$ \Comment{Eq.~\ref{eq:group_advantage}}
\State $L_{\mathrm{PG}} \leftarrow -\mathbb{E}_g\!\left[A^{(g)}\,\log \pi_{\theta}^{(g)}\!\left(t^{(g)}\right)\right]$ \Comment{Eq.~\ref{eq:pg_loss}}

\State \textbf{Rollout KL regularization (K3 surrogate, rollout aggregation)}
\State Compute teacher per-step gold log-probs $\ell_{\bar{\theta}}^{(k)}(t)$ on the same rollout-evaluated tokens/steps.
\For{$k=1,\dots,K$}
    \State $\Delta^{(k)} \leftarrow \ell_{\bar{\theta}}^{(k)}(t)-\ell_{\theta,g}^{(k)}(t)$ \Comment{Eq.~\ref{eq:delta_logprob}}
    \State $u^{(k)} \leftarrow \exp(\Delta^{(k)})$, \quad $K3(k)\leftarrow u^{(k)}-\Delta^{(k)}-1$ \Comment{Eq.~\ref{eq:k3}}
\EndFor
\State Let $m_G$ be the rollout mask indicating the positions/steps evaluated in the rollout pass.
\State $L_{\mathrm{KL}} \leftarrow \dfrac{\sum_{k=1}^K K3(k)\odot m_G}{\sum m_G+\varepsilon}$ \Comment{Eq.~\ref{eq:kl_loss}}

\State \textbf{Total objective and updates}
\State $L \leftarrow \alpha L_{\mathrm{PG}} + \beta L_{\mathrm{rep}} + \gamma L_{\mathrm{ent}} + \delta L_{\mathrm{KL}}$ \Comment{Eq.~\ref{eq:total_loss}}
\State Update $\theta$ by one gradient step on $L$; update EMA teacher $\bar{\theta}\leftarrow \mathrm{EMA}(\bar{\theta},\theta)$.
\end{algorithmic}
\end{algorithm}

%%%%%%%%%%%%%%%%%%%%%%%%%%%%%%%%%%%%%%%%%%%%%%%%%%%%%%%%%%%%

\clearpage
\section*{NeurIPS Paper Checklist}

\begin{enumerate}

\item {\bf Claims}
    \item[] Question: Do the main claims made in the abstract and introduction accurately reflect the paper's contributions and scope?
    \item[] Answer: \answerYes{} % Replace by \answerYes{}, \answerNo{}, or \answerNA{}.
    \item[] Justification:  The claims made in the abstract and introduction are empirically verified in Section~\ref{sec:result}.

\item {\bf Limitations}
    \item[] Question: Does the paper discuss the limitations of the work performed by the authors?
    \item[] Answer: \answerYes{} % Replace by \answerYes{}, \answerNo{}, or \answerNA{}.
    \item[] Justification: We mention our limitations in Appendix~\ref{app:limit}.

\item {\bf Theory assumptions and proofs}
    \item[] Question: For each theoretical result, does the paper provide the full set of assumptions and a complete (and correct) proof?
    \item[] Answer: \answerYes{} % Replace by \answerYes{}, \answerNo{}, or \answerNA{}.
    \item[] Justification: We provide the necessary justifications for the design of high-entropy token selection, noisy latent rollouts and EMA teacher in Appendix~\ref{app:additional_theory}.

    \item {\bf Experimental result reproducibility}
    \item[] Question: Does the paper fully disclose all the information needed to reproduce the main experimental results of the paper to the extent that it affects the main claims and/or conclusions of the paper (regardless of whether the code and data are provided or not)?
    \item[] Answer: \answerYes{} % Replace by \answerYes{}, \answerNo{}, or \answerNA{}.
    \item[] Justification: We provide necessary implementation details in Appendix~\ref{app:hardtokens}-~\ref{app:setup_details}. We do not release the source code.

\item {\bf Open access to data and code}
    \item[] Question: Does the paper provide open access to the data and code, with sufficient instructions to faithfully reproduce the main experimental results, as described in supplemental material?
    \item[] Answer: \answerNo{} % Replace by \answerYes{}, \answerNo{}, or \answerNA{}.
    \item[] Justification: We do not provide the source code. 

\item {\bf Experimental setting/details}
    \item[] Question: Does the paper specify all the training and test details (e.g., data splits, hyperparameters, how they were chosen, type of optimizer) necessary to understand the results?
    \item[] Answer: \answerYes{} % Replace by \answerYes{}, \answerNo{}, or \answerNA{}.
    \item[] Justification: We provide training and testing details in Section~\ref{sec:setup} and Appendix~\ref{app:setup_details}.

\item {\bf Experiment statistical significance}
    \item[] Question: Does the paper report error bars suitably and correctly defined or other appropriate information about the statistical significance of the experiments?
    \item[] Answer: \answerNo{} % Replace by \answerYes{}, \answerNo{}, or \answerNA{}.
    \item[] Justification: All experimental results are averaged over three random seeds; we do not report error bars or significance tests.

\item {\bf Experiments compute resources}
    \item[] Question: For each experiment, does the paper provide sufficient information on the computer resources (type of compute workers, memory, time of execution) needed to reproduce the experiments?
    \item[] Answer: \answerYes{} % Replace by \answerYes{}, \answerNo{}, or \answerNA{}.
    \item[] Justification: We provide information about the computer resources used in Section~\ref{sec:setup} and Appendix~\ref{app:setup_details}.
    
\item {\bf Code of ethics}
    \item[] Question: Does the research conducted in the paper conform, in every respect, with the NeurIPS Code of Ethics \url{https://neurips.cc/public/EthicsGuidelines}?
    \item[] Answer: \answerYes{} % Replace by \answerYes{}, \answerNo{}, or \answerNA{}.
    \item[] Justification:  We have read and understood the code of ethics and have made every effort to adhere to it.

\item {\bf Broader impacts}
    \item[] Question: Does the paper discuss both potential positive societal impacts and negative societal impacts of the work performed?
    \item[] Answer: \answerNA{} % Replace by \answerYes{}, \answerNo{}, or \answerNA{}.
    \item[] Justification: Our work contributes to a better understanding of the reinforcement learning for looped LMs. It does not impact society at large.
    
\item {\bf Safeguards}
    \item[] Question: Does the paper describe safeguards that have been put in place for responsible release of data or models that have a high risk for misuse (e.g., pre-trained language models, image generators, or scraped datasets)?
    \item[] Answer: \answerNA{} % Replace by \answerYes{}, \answerNo{}, or \answerNA{}.
    \item[] Justification: We do not release any new models or datasets. Our work uses standard,
publicly available resources and does not pose additional risk of misuse beyond their existing public availability.

\item {\bf Licenses for existing assets}
    \item[] Question: Are the creators or original owners of assets (e.g., code, data, models), used in the paper, properly credited and are the license and terms of use explicitly mentioned and properly respected?
    \item[] Answer: \answerYes{} % Replace by \answerYes{}, \answerNo{}, or \answerNA{}.
    \item[] Justification: We use publicly available datasets (\textsc{Omni-Math}) and pretrained
models (Ouro), all of which are properly cited in the paper. We adhere to the licenses and
terms of use as specified by their original creators.

\item {\bf New assets}
    \item[] Question: Are new assets introduced in the paper well documented and is the documentation provided alongside the assets?
    \item[] Answer: \answerYes{} % Replace by \answerYes{}, \answerNo{}, or \answerNA{}.
    \item[] Justification: We do not release new assets; an anonymous repository was made available only for the reviewing process.

\item {\bf Crowdsourcing and research with human subjects}
    \item[] Question: For crowdsourcing experiments and research with human subjects, does the paper include the full text of instructions given to participants and screenshots, if applicable, as well as details about compensation (if any)? 
    \item[] Answer: \answerNA{} % Replace by \answerYes{}, \answerNo{}, or \answerNA{}.
    \item[] Justification: We didn't involve crowdsourcing nor research with human subjects.

\item {\bf Institutional review board (IRB) approvals or equivalent for research with human subjects}
    \item[] Question: Does the paper describe potential risks incurred by study participants, whether such risks were disclosed to the subjects, and whether Institutional Review Board (IRB) approvals (or an equivalent approval/review based on the requirements of your country or institution) were obtained?
    \item[] Answer: \answerNA{} % Replace by \answerYes{}, \answerNo{}, or \answerNA{}.
    \item[] Justification: We didn't involve crowdsourcing nor research with human subjects.

\item {\bf Declaration of LLM usage}
    \item[] Question: Does the paper describe the usage of LLMs if it is an important, original, or non-standard component of the core methods in this research? Note that if the LLM is used only for writing, editing, or formatting purposes and does \emph{not} impact the core methodology, scientific rigor, or originality of the research, declaration is not required.
    %this research? 
    \item[] Answer: \answerNA{} % Replace by \answerYes{}, \answerNo{}, or \answerNA{}.
    \item[] Justification: We used an LLM solely for editing assistance.

\end{enumerate}

\end{document}